\documentclass{article}

\PassOptionsToPackage{sort, numbers, compress}{natbib}
\usepackage[main,final]{format/neurips_2026}
\usepackage[utf8]{inputenc} % allow utf-8 input
\usepackage[T1]{fontenc}    % use 8-bit T1 fonts
\usepackage{hyperref}       % hyperlinks
\usepackage{url}            % simple URL typesetting
\usepackage{booktabs}       % professional-quality tables
\usepackage{amsfonts}       % blackboard math symbols
\usepackage{nicefrac}       % compact symbols for 1/2, etc.
\usepackage{microtype}      % microtypography
\usepackage{xcolor}         % colors

\usepackage{amsmath,amssymb,amsthm}

\usepackage{format/submission-commands}
\usepackage{format/refcommands}

\newif\ifICLR
\ICLRfalse

\newcommand{\githubRepo}{\url{https://github.com/7angel4/prism-games/tree/csg-learning}}

\title{Robust PAC Learning of Concurrent Stochastic Games}

\author{%
  Angel Y.~He \\
  Department of Engineering Science\\
  University of Oxford\\
  \texttt{angelyhe@robots.ox.ac.uk} \\
  \And
  David Parker \\
  Department of Computer Science\\
  University of Oxford\\
  \texttt{david.parker@cs.ox.ac.uk} \\
}

\begin{document}

\maketitle

\begin{abstract}
We introduce the first Probably Approximately Correct (PAC) learning framework for general-sum concurrent stochastic games (CSGs) with transition uncertainty, while addressing the challenge of Nash equilibrium (NE) existence. Our algorithm maintains data-driven $L^1$ confidence sets over transition kernels and solves a robust CSG to compute a social-welfare optimal $\varepsilon$-NE, using a robust MDP-based exploration mechanism to drive joint state--action coverage. 
Crucially, we introduce a Nash margin characterisation that enables principled reasoning about equilibrium existence: the framework either returns an $\varepsilon$-approximate NE whose social-welfare value is $\varepsilon$-close to optimal, or provides a sound certificate that no exact NE exists.
Under a minimum reachability condition $p_{\mathrm{reach}}>0$ over relevant state--action pairs, the algorithm terminates after a polynomial number of trajectory samples, with sample complexity $\widetilde{\mathcal{O}}\!\left( {R_{\max}^2 H^4 |S|^2 |A| / (p_{\mathrm{reach}} \varepsilon^2)} \right)$. Empirical results on benchmark CSGs demonstrate near-optimal performance, correct handling of equilibrium (non-)existence, and sample complexity consistent with theory.
% ; implementation available at: \githubRepo.
\end{abstract}

\section{Introduction} 
\label{sec:intro}

Learning to act optimally in multi-agent stochastic environments is fundamentally harder when the transition dynamics are unknown: even small estimation errors can destabilise equilibria or make a nominally near-optimal strategy far from equilibrium under the true model. Concurrent stochastic games (CSGs) \citep{shapley1953SGs}, where players choose actions simultaneously without observing one another, are the canonical model for such settings --- yet most existing approaches to equilibrium computation assume that transition probabilities are precisely known. This paper asks: \emph{how many environment interactions suffice to compute an approximately optimal equilibrium with high probability, and what happens when no exact equilibrium exists?}

Uncertainty in stochastic decision-making has been studied from two complementary angles.
On the robust-planning side, robust MDPs (RMDPs)~\citep{iyengar2005robust,nilim2005robust,wiesemann2013robust} including interval MDPs~\citep{givan2000bounded}, and their stochastic-game extensions~\citep{castro2025polytopal,berthon2025robustmulti} optimise worst-case performance over a prescribed uncertainty set. Recent work has extended this line to robust CSGs (RCSGs)~\citep{HP26} and distributionally robust Markov games (DRMGs)~\citep{shi2024sample,li2022minimax,zhang2023modelbased,roch2025drmg}.
On the learning side, \emph{Probably Approximately Correct} (PAC)~\citep{valiant1984PAC} algorithms provide sample-complexity guarantees for MDPs~\citep{kearns2002nearE3,brafman2002RMAX,jaksch2010near,azar2017minimax,strehl2009RL} and turn-based stochastic games~\citep{bai2020provable,liu2021sharp,zehfroosh2023pac}. 
However, these~two~lines of work have remained largely separate in the concurrent general-sum setting: robust methods typically assume availability of the uncertainty set or a generative model~\citep{iyengar2005robust,nilim2005robust,shi2024sample,li2022minimax,zhang2023modelbased}, while PAC methods do not address robust equilibrium computation or equilibrium existence under unknown dynamics. Work at their intersection on \emph{robust PAC learning} remains largely restricted to the single-agent setting~\citep{yang2022towards,zhou2021finite,wang2023online,panaganti2022robustRL,panaganti2022sample}.

Combining robust equilibrium computation with PAC learning in concurrent games poses two challenges absent from prior work: 
% \angel{improve wording: it's sound in non-existence detection but not existence detection}
\begin{enumerate*}[\textbf{\arabic*)}]
    \item \textbf{Equilibrium existence under uncertainty.}
% \ifICLR
    While $\varepsilon$-NE exist for any $\varepsilon>0$~\citep{de2007concurrent}, exact stationary (i.e., memoryless and time-independent) NE may fail to exist in CSGs~\citep{bouyer2014mixed}.
    The issue is not merely pathological: in multi-agent systems, an NE represents a stable operating point in which no agent has unilateral incentive to deviate. When no such equilibrium exists, any deployed strategy may be inherently unstable, potentially inducing unpredictable or undesirable system behaviour, as observed in domains such as traffic coordination, network routing, and competitive markets~\citep{holt2004nash,jin2012game}.
    Moreover, even when an exact NE exists in the true game, an uncertainty set around the true transition kernel may not have a \emph{robust} equilibrium~\citep{HP26}. A PAC framework for CSGs should therefore support both approximate equilibrium computation and principled handling of non-existence.
    %
% \ifICLR
%     \item \textbf{Global coverage from stochastic trajectories.}
%     Joint state--action coverage is hard to obtain online because no individual player controls the joint action and the transition kernel is uncertain. Under centralised exploration, we use an auxiliary \emph{exploration RMDP} to select joint profiles that robustly reach under-visited state--action pairs, yielding~a per-episode lower bound on visiting under-explored pairs in the true game. Since the under-visited set and exploration policy adapt to the observed history, these episode-level events are dependent; we therefore aggregate the conditional lower bound into a high-probability global coverage and sample-complexity guarantee via a martingale argument based on Freedman's inequality~\citep{freedman1975tail}.
% \else
    \item \textbf{Global coverage from stochastic trajectories.}
    Obtaining sufficient joint state--action coverage from online trajectories is challenging in our setting because no individual player controls the joint action and the transition kernel is uncertain. Under centralised exploration, we address this by coordinating joint actions and constructing an auxiliary \emph{exploration RMDP} on which we compute joint profiles to robustly reach under-visited state--action pairs. This yields a per-episode lower bound on visiting insufficiently explored pairs in the true game. Because the set of under-visited pairs and the exploration policy adapt to the observed history, these episode-level events are dependent; we therefore use a martingale concentration argument based on Freedman's inequality~\citep{freedman1975tail} to aggregate the conditional lower bound into a high-probability global coverage and sample-complexity guarantee.
    % which yields substantially tighter control than Azuma-type bounds~\citep{azuma1967weighted,hoeffding1963probability} used in prior PAC-MDP work, e.g., \citep{kearns2002nearE3,jaksch2010near,strehl2009RL}
% \fi
\end{enumerate*}

\paragraph{Contributions.}
We introduce \emph{PAC-CSG}, the first PAC-learning framework for two-player general-sum CSGs under transition uncertainty with \emph{sound} detection of stationary equilibrium non-existence. When an exact NE exists, the algorithm returns an $\varepsilon$-NE (i.e., stable up to $\varepsilon$ under unilateral deviations) whose social-welfare value is $\varepsilon$-close to optimal; otherwise, it may return a sound certificate of non-existence.
We consider finite- and infinite-horizon probabilistic and reward reachability objectives, which are standard in verification and planning~\citep{kwiatkowska2021automatic,kwiatkowska2022correlated}, unlike the discounted reward objectives common in reinforcement learning (RL)~\citep{azar2017minimax,strehl2009RL}. 
To realise this framework, we design an algorithm that maintains $L^1$ confidence sets over transition kernels and combines robust equilibrium computation with an auxiliary RMDP-based exploration mechanism that formulates joint state--action coverage as a robust reachability problem. 
Under a graph reachability condition, the algorithm terminates after $\bigOtilde{\Rmax^2 H^4 |S|^2 |A| / (\reach{p} \varepsilon^2)}$ trajectory samples and outputs either an $\varepsilon$-optimal $\varepsilon$-NE or a sound non-existence certificate, where $\reach{p}$ is the minimum reachability probability over all relevant state--action pairs.
% $:= \min_{(s,a)} \max_\sigma \inf_P P^{\sigma,P}[\reachopt\ (s,a)] > 0$. 
Finally, we evaluate the approach on six CSG benchmarks, demonstrating near-optimality, correct handling of equilibrium (non-)existence, and sample complexity scaling consistent with the theoretical bounds.
% We discuss potential societal impacts of this work, including both positive applications and possible risks, in \Cref{app:broader-impacts}.

% ---------

% The remainder of the paper is organised as follows. Section~2 introduces the model and preliminaries, Section~3 presents the PAC-CSG framework, and Sections~4–5 provide theoretical guarantees and experiments.

% ---------------------------------------------------------------
\paragraph{Related work.} %\label{sec:related-work}
% ---------------------------------------------------------------

\ifICLR
% Condensed for the 9-page main text; the full discussion is in \Cref{app:related-work}.
\emph{Robust planning} in stochastic games optimises worst-case performance over a prescribed uncertainty set, spanning robust and interval MDPs~\citep{iyengar2005robust,nilim2005robust,wiesemann2013robust,givan2000bounded} and, more recently, multi-agent extensions~\citep{castro2025polytopal,berthon2025robustmulti,HP26}; these assume the uncertainty set is given rather than learned from data.
\emph{PAC learning} provides sample-complexity guarantees for~MDPs \citep{kearns2002nearE3,brafman2002RMAX,strehl2009RL} and turn-based stochastic games~\citep{bai2020provable,liu2021sharp,zehfroosh2023pac}, but not for concurrent games, where no single agent controls state--action visitation, nor does it incorporate robustness or equilibrium existence.
\emph{Robust RL} learns under model uncertainty from data~\citep{yang2022towards,panaganti2022robustRL,panaganti2022sample,zhou2021finite,wang2023online}, with multi-agent work concentrated on zero-sum distributionally robust Markov games~\citep{shi2024sample,li2022minimax,zhang2023modelbased,roch2025drmg}.
Closest to our setting, \citet{farhat2026sample} study online robust multi-agent RL, but optimise regret rather than providing PAC guarantees and do not address equilibrium existence in general-sum concurrent games.
See \Cref{app:related-work} for an extended discussion.
\else

\emph{Robust planning} in stochastic games addresses decision-making under a prescribed model uncertainty set.
In the single-agent setting, this includes robust MDPs~\citep{iyengar2005robust,nilim2005robust,wiesemann2013robust} and interval MDPs~\citep{givan2000bounded}, which optimise worst-case performance over admissible transition kernels. 
Multi-agent extensions are more recent: robust turn-based stochastic games~\citep{castro2025polytopal} and qualitative objectives~\citep{berthon2025robustmulti} have been studied, and \citet{HP26}'s robust CSG framework introduces robust equilibrium notions for general-sum concurrent games. 
These approaches assume that the uncertainty set is specified a priori and focus on computing robust policies or equilibria, rather than learning them from data.

% In parallel, distributionally robust Markov games (DRMGs) have been developed for zero-sum settings~\citep{shi2024sample,li2022minimax,zhang2023modelbased,roch2025drmg,liu2021sharp}, including regularisation equivalences~\citep{zhang2023modelbased}, average-reward formulations with robust Nash equilibria~\citep{roch2025drmg}, and sample-efficient or online learning algorithms~\citep{farhat2026sample} under simulator or interaction access. 
% These works provide important planning and robustness foundations, but they still assume that the uncertainty model is already available rather than learned from online trajectories, or they do not address robust equilibrium existence in the general-sum concurrent case.

\emph{PAC learning} methods
are well developed for MDPs via E$^3$~\citep{kearns2002nearE3}, R-MAX~\citep{brafman2002RMAX}, and related PAC-MDP analyses~\citep{strehl2009RL}. 
Closely related exploration guarantees include regret-minimisation methods such as UCRL2~\citep{jaksch2010near} and UCBVI~\citep{azar2017minimax}.
% UCRL2~\citep{jaksch2010near}, UCBVI~\citep{azar2017minimax}, 
% with related sample complexity guarantees in model-free settings~\citep{strehl2009RL}.
% \angel{add a few more model-free algorithms??}
These results have been extended to turn-based stochastic games in both zero-sum self-play~\citep{bai2020provable,liu2021sharp} and general-sum settings~\citep{zehfroosh2023pac}. 
However, they do not address the concurrent setting, where no single agent controls state--action visitation, nor do they incorporate robustness to transition uncertainty or equilibrium existence considerations.

\emph{Robust RL} considers learning under model uncertainty from data. 
In the single-agent setting, both model-based and model-free approaches provide sample complexity guarantees~\citep{yang2022towards,panaganti2022robustRL,panaganti2022sample,zhou2021finite,wang2023online}. 
In multi-agent settings, recent work studies distributionally robust Markov games (DRMGs), primarily in zero-sum settings~\citep{shi2024sample,li2022minimax,zhang2023modelbased,roch2025drmg}, including sample-efficient algorithms under generative-model assumptions that largely remove the need for exploration.
Other work considers alternative solution concepts, e.g., robust correlated equilibria~\citep{ma2023decentralized}. 
Closest to our setting is \citet{farhat2026sample}, who study online robust multi-agent RL under similar uncertainty sets. However, they optimise regret rather than providing PAC guarantees, and do not address equilibrium existence in general-sum concurrent games. 
In contrast, we provide a high-confidence output guarantee: either an $\varepsilon$-NE when one exists, or a sound certificate of non-existence. 
This distinction is fundamental in our concurrent general-sum setting: PAC learning targets the quality of a single output equilibrium, whereas regret minimisation controls performance over the entire learning trajectory. 
% Moreover, without unilateral control of exploration, regret-based methods can be statistically hard, with exponential lower bounds even in adversarial settings~\citep{liu2022adversarial}, making PAC guarantees a more suitable objective in this setting.
\fi

\section{Preliminaries}
\label{sec:prelim}
We denote by $\distr(X)$ the set of discrete probability distributions over a finite set $X$, and let $\ind{A}$ be the indicator function that equals $1$ if $A$ holds and $0$ otherwise.

% \subsection{Models}\label{sec:prelim-CSG}

%\citep{shapley1953SGs}

\begin{definition}[Concurrent stochastic game]
\label{def:CSG}
A (reward-augmented) \emph{concurrent stochastic game} (CSG) is a tuple
$\csg = (N, S, \bar{s}, A, \Gamma, P, r)$
where:
$N = \{1,\dots,n\}$ is a finite set of players;
$S$ is a finite set of states with initial state $\bar{s} \in S$;
$A = \bigtimes_{i \in N} (A_i \cup \{\idle\})$ is the set of joint actions, where each $A_i$ is a finite action set and $\idle$ is a distinguished idle action;
$\Gamma : S \to 2^{\cup_i A_i}$ assigns available actions;
$P : S \times A \rightharpoonup \distr(S)$ is a transition kernel; and
$r = (r_1,\dots,r_n)$ with rewards $r_i : S \times A \to \reals$.
\end{definition}

At each state $s$, each player ${i \in N}$ simultaneously selects an action $a_i\in A_i(s)$, where $A_i(s) = \Gamma(s) \intersect A_i$ if ${\Gamma(s) \intersect A_i \neq \emptyset}$ and $A_i(s) =\idleset$ otherwise. The joint action $a = (a_1,\dots,a_n)\in A(s) := \bigtimes_{i\in N}{A_i(s)}$ induces a transition to state $s'$ according to $P_{sa}:= P(s,a)$. Henceforth let $\slots := \{(s,a): s\in S, a\in A(s)\}$ denote the set of all admissible state--action pairs (or ``slots'').

\ifICLR
A \emph{strategy} for player $i$ is a function $\sigma_i$ mapping finite histories to distributions over actions; it is \emph{stationary} (or \emph{memoryless}) if $\sigma_i(\pi)=\sigma_i(s)$ for all histories $\pi$ ending in $s$, the class we focus on throughout. A \emph{strategy profile} (or \emph{profile}) is a tuple $\sigma=(\sigma_1, \ldots, \sigma_n) \in \Sigma := \bigtimes_{i\in N}{\Sigma_i}$, and for a unilateral deviation $\sigma_i'\in\Sigma_i$ we write $\sigma_{-i}[\sigma_i']$ for the profile replacing player $i$'s strategy with $\sigma_i'$.

Each player $i$ has an objective $X_i$ mapping infinite paths to $\reals$. For a profile $\sigma$, kernel $P$, and state $s$, the \emph{social welfare} is $u(\sigma,P\mid s) := \sum_{i\in N}{u_i(\sigma,P\mid s)}$, where $u_i(\sigma,P\mid s) := \ev^{\sigma,P}_s[X_i]$; we write $u(\sigma,P) := u(\sigma,P\mid \bar{s})$ for the value at the initial state.
\else
A \emph{strategy} for player $i$ is a function $\sigma_i$ %: FPaths_{\csg}\to\distr(A_i)$
mapping finite histories to distributions over actions. 
A \emph{stationary} (or \emph{memoryless}) strategy depends only on the current state, i.e., $\sigma_i(\pi)=\sigma_i(s)$ for all histories $\pi$ ending in state $s$. We focus on this class of strategies throughout this work.
A \emph{strategy profile} (or just \emph{profile}) is a tuple of strategies, one for each player, denoted $\sigma=(\sigma_1, \ldots, \sigma_n) \in \Sigma := \bigtimes_{i\in N}{\Sigma_i}$. 
For a profile $\sigma$ and a unilateral deviation $\sigma_i'\in\Sigma_i$, we write $\sigma_{-i}[\sigma_i']$ for the profile obtained by replacing player $i$'s strategy $\sigma_i$ with $\sigma_i'$.

Each player $i$ is associated with an objective $X_i$ % : \ipathg \to \reals$.
mapping infinite paths to $\reals$.
For a profile $\sigma$, kernel $P$, and state $s$, we define players' total expected utility, i.e., the \emph{social welfare}, as $u(\sigma,P\mid s) := \sum_{i\in N}{u_i(\sigma,P\mid s)}$, where $u_i(\sigma,P\mid s) := \ev^{\sigma,P}_s[X_i]$ is player $i$'s expected utility. 
% For brevity, we will omit the $s$ here when $s$ is taken to be the initial state $\bar{s}$.
For the value in the initial state $\bar{s}$, we write $u(\sigma,P) := u(\sigma,P\mid \bar{s})$.
\fi

% An \emph{objective} (or utility function) of player $i$ is a random variable $X_i: \ipathg \rightarrow \reals$. 
% % In the \emph{general-sum} (or \emph{nonzero-sum}) case, 
% We write $X = (X_1,\dots,X_n)$ for the tuple of all player objectives, and denote the \emph{expected utility} of player $i$ from state $s$ under profile $\sigma$ in $\csg$ as $u_i(\sigma \mid s,X) :=  \valg^i(s \mid \sigma, X) := \stratgs{\ev}[X_i]$. 
% Since players may cooperate or compete in a general-sum game, we use the concept of a \emph{Nash equilibrium} (NE): a profile in which no player can improve their utility by unilaterally deviating. A \emph{social-welfare optimal NE} (SWNE) \citep{kwiatkowska2021automatic} refers to an NE that also maximises the players' total utility. 

\ifICLR
\emph{Robust CSGs} (RCSGs)~\citep{HP26} add transition uncertainty: an RCSG is a tuple $\csg = (N,S,\sbar,A,\Gamma,\Punc,r)$ where $\Punc: S \times A \rightharpoonup 2^{\distr(S)}$ is an uncertain transition kernel and all else is as in \Cref{def:CSG}; fixing $P \in \Punc$ induces a CSG $\csg_P$. We focus on $(s,a)$-rectangular uncertainty, where $\Punc=\prod_{(s,a)\in\slots}{\Punc(s,a)}$, so uncertainty at each slot resolves independently.
\else
\emph{Robust CSGs} (RCSGs)~\citep{HP26} add transition uncertainty to CSGs.
Formally, an RCSG is a tuple $\csg = (N,S,\sbar,A,\Gamma,\Punc,r)$, where $\Punc: S \times A \rightharpoonup 2^{\distr(S)}$ is an uncertain transition kernel, and all other components are as in \Cref{def:CSG}. 
Fixing a $P \in \Punc$ in an RCSG induces a CSG, which we denote $\csg_P$. % = (N,S,\sbar,A,\Gamma,P,r)$. 
% Utilities are now evaluated with respect to a chosen kernel: $u_i(\sigma,P \mid s) := \gs{\ev}^{\sigma,P}[X]$.
We focus on $(s,a)$-rectangular uncertainty, where the global uncertainty set factorises as $\Punc=\prod_{(s,a)\in\slots}{\Punc(s,a)}$, so transition uncertainty at each slot can be resolved independently.
\fi
For a profile $\sigma$, we denote its \emph{robust} and \emph{optimistic} values over $\Punc$, respectively, as 
\[
\rob{u}_{\Punc}(\sigma) := \inf_{P \in \Punc} u(\sigma,P),
\qquad
\opt{u}_{\Punc}(\sigma) := \sup_{P \in \Punc} u(\sigma,P).
\]

\paragraph{Objectives.}
\label{sec:objectives}
Following \citet{HP26,kwiatkowska2021automatic}, we focus on four objectives (two finite-horizon, two infinite-horizon). Let $S_T \subseteq S$ be a set of target states, and $\ntslots := \{(s,a)\in\slots: s\not\in S_T\}$ the set of non-target slots. For finite-horizon objectives fix $H \in \nats$:
\begin{enumerate*}[(1)]
    \item \emph{Bounded probabilistic reachability}: $X(\pi) = \ind{\exists{j \leq H}. \ {\pi(j)\in S_T}}$;
    \item \emph{Bounded cumulative reward}:
    $X(\pi) = \sum_{i=0}^{H-1}{r(\pi(i),\pi[i])}$;
    \item \emph{Probabilistic reachability}: 
    $X(\pi) = \ind{\exists{j \in \nats}. \ {\pi(j)\in S_T}}$; and 
    \item \emph{Reachability reward (or stochastic shortest paths)}:
    $X(\pi) = \sum_{i=0}^{\tau_T-1}{r(\pi(i),\pi[i])}$ if $\exists{j \in \nats}. \ {\pi(j)\in S_T}$ and $X(\pi) =\infty$ otherwise, where $\tau_T = \min{\{ j\in \nats \mid \pi(j)\in S_T\}}$.
\end{enumerate*}

For the reachability objectives, transitions after entering $S_T$ do not affect the objective, so we only need to resolve uncertainty over $\ntslots$. Thus, henceforth we consider $\Punc=\prod_{(s,a)\in\ntslots}{\Punc(s,a)}$. For the bounded cumulative reward objective, $S_T=\emptyset$ so $\ntslots=\slots$.

\paragraph{Solution concepts.}
A \emph{Nash equilibrium} (NE) in a CSG is a profile $\sigma$ such that, for all players $i$, states $s$ and deviations $\sigma_i'$, $u_i(\sigma,P\mid s) \ge u_i(\sigma_{-i}[\sigma_i'],P \mid s)$. 
A \emph{social-welfare optimal NE} (SWNE) is an NE maximising $u(\sigma,P\mid s)$.
We require these equilibrium conditions to hold from every state, yielding a \emph{subgame-perfect} equilibrium.

For an RCSG, we further require the same equilibrium requirement to hold robustly, i.e., across every $P\in\Punc$.
In addition, since exact stationary NE may not exist~\citep{bouyer2014mixed} but $\varepsilon$-NE exists for any $\varepsilon > 0$, we use \emph{approximate} robust equilibria:
\begin{definition}[$\varepsilon$-RNE]
\label{def:subgame-perfect-rne}
    A profile $\sigma$ is a \emph{robust $\varepsilon$-Nash equilibrium} ($\varepsilon$-RNE) iff for all players $i$, states $s$, and deviations $\sigma_i'$, 
    ${\inf_{P\in \Punc}{\left[ u_i(\sigma, P \mid s) - u_i(\sigma_{-i}[\sigma_i'], P\mid s) \right]} \geq -\varepsilon}$.
    The corresponding $\varepsilon$-RNE value is defined as $\rob{u}_{\Punc}(\sigma) := \inf_{P\in \Punc}{u(\sigma, P)}$. 
\end{definition}

The case $\varepsilon = 0$ yields a \emph{robust NE} (RNE). 
Following this definition, a \emph{robust social-welfare optimal $\varepsilon$-NE} ($\varepsilon$-RSWNE) is just an $\varepsilon$-RNE $\sigma$ that maximises $\rob{u}_{\Punc}(\sigma)$, the robust total utility of the players.
Henceforth we write $\epsRNE(\Punc)$ for the set of all $\varepsilon$-RNE under $\Punc$, $\RNE(\Punc)$ for the set of RNE, and similarly $\epsNE(P)$ and $\NE(P)$ for the set of $\varepsilon$-NE and NE, respectively, under kernel $P$.

\section{PAC-CSG}
\label{sec:PAC-framework}

% \subsection{PAC learning of CSGs}

\paragraph{Game setting.}
We model the true environment as a finite two-player general-sum CSG $\csg^\star = (N=\{1,2\}, S, \bar{s}, A, \Gamma, P^\star, r)$, where $P^\star$ is unknown but the state space, action availability, rewards, and transition support of $\csg^\star$ are known.
% its support $\supp(P^\star):= \{(s,a,s'): s,s'\in S, a\in A(s), P^\star_{sa}(s')>0\}$ and all other components of $\csg^\star$ are known.
%
When it exists, denote the SWNE value of $\csg^\star$ by 
$V^\star := \sup_{\sigma\in \NE(P^\star)}{u(\sigma,P^\star)} = u(\sigma^\star, P^\star)$, with $\sigma^\star$ being the corresponding profile.
We assume bounded~rewards: $|r_i(s,a)| \le \Rmax$.
We also enforce the standard \emph{graph preservation} constraint \citep{chatterjee2008model,meggendorfer2025solving}, requiring all $P \in \Punc$ to share the same support. This ensures the tractability of robust dynamic programming \citep{nilim2005robust,iyengar2005robust} over $(s,a)$-rectangular uncertainty models.
% \begin{assumption}[Graph preservation]
% \label{assmp:graph-preservation}
% All $P \in \Punc$ have the same support, i.e., for all $(s,a,s') \in S \times A \times S$, $\exists P \in \Punc. \ P_{sas'} = 0 \iff \forall{P' \in \Punc}. \ P'_{sas'} = 0$.
% \end{assumption}

\paragraph{Learning.}
\label{sec:learning-setup}
Throughout, we assume a \emph{centralised} learner that jointly controls both players. Learning proceeds episodically: quantities derived from data up to episode $t$ are indexed by $t$, and $T$ denotes the total number of episodes before termination.
At episode $t$, the learner maintains an uncertainty set $\Punc_t$ around an empirical kernel $\widehat{P}_t$. This set is defined via per-slot $L^1$ confidence balls: for each $(s,a)\in\ntslots$, $\| P_{sa} - \widehat{P}_{sa,t}\|_1 \le \alpha_t(s,a) := \sqrt{\tfrac{2}{n} \ln{\left((2^{|S|}-2)/\confCI_{sa,n} \right)}}$, where $n:=N_t(s,a)$ is the visit count of $(s,a)$ up to episode $t$ and $\confCI_{sa,n}$ is a suitable failure probability defined in \Cref{cor:true-kernel-containment}, \Cref{app:per-slot-proofs}. 
The radius $\alpha_t(s,a)$ follows from Weissman's $L^1$ concentration inequality~\citep{weissman2003inequalities} and is derived in \Cref{lem:per-slot-concentration}, \Cref{app:per-slot-proofs}. 
$\Punc_t$ contains all transition kernels consistent with the observed data so far with high probability, and induces an empirical $L^1$-CSG $\csg_t = (N, S, \bar{s}, A, \Gamma, \Punc_t, r)$.
Further, the learner selects an exploration profile $\explore{\sigma}_t$ by solving an exploration RMDP $\explore{\csg}_t$ (see 
\ifICLR
Def.~\ref{def:exploration-RMDP})
\else
\Cref{def:exploration-RMDP})
\fi
induced by the same uncertainty set $\Punc_t$, executes it in $\csg^\star$, and uses the resulting trajectories to update $\csg_t$.

\subsection{PAC-CSG}
\label{sec:PAC-CSG}

Unlike single-agent or zero-sum settings, exact stationary NE need not exist in our setting.

\begin{remark}[NE existence]
\label{rem:NE-existence}
For any finite CSG, $\epsNEop{\varepsilon}(P) \neq \emptyset$ for all $\varepsilon > 0$~\citep{bouyer2014mixed,de2007concurrent}. 
However, exact stationary NE ($\varepsilon = 0$) may fail to exist in CSGs. For finite-horizon CSGs, backward induction generally yields time-dependent equilibrium strategies. 
In the infinite-horizon setting, \citet{bouyer2014mixed} illustrate non-existence via two-player concurrent terminal-reward games with no NE at all (and hence no stationary NE), including the \textit{Hide-or-run} game used in our experiments; see~\Cref{sec:benchmarks}.
% \citet{bouyer2014mixed} prove that even deciding NE existence in concurrent terminal-reward games is undecidable (for $\geq 3$ players), and 
% Therefore, the non-existence branch~\ref{item:near-non-existence-cert} in \Cref{def:PAC-CSG} is vacuous for finite-horizon objectives, but becomes meaningful in the infinite-horizon extension considered in \Cref{sec:infh}.
%
Finally, for an RCSG $\csg$, even if every induced CSG $\csg_P$ admits an ($\varepsilon$-)NE, an ($\varepsilon$-)RNE need not exist in $\csg$~\citep{HP26}, as it must satisfy equilibrium conditions simultaneously for all admissible kernels $P$.
\end{remark}

To reason about equilibrium existence quantitatively, we introduce the notion of \emph{Nash margin}.
\begin{definition}[Nash margin]
\label{def:Nash-margin}
For a CSG with kernel $P$ and profile space $\Sigma$, define the \emph{Nash margin} of a profile $\sigma \in \Sigma$ as $\mu(\sigma, P) := \min_{i \in N} \min_{s\in S}\inf_{\sigma'_i \in \Sigma_i}
\left[ 
    u_i(\sigma,P\mid s) - u_i(\sigma_{-i}[\sigma_i'],P\mid s)
\right]$. 
Further define the \emph{global Nash margin} $\opt{\mu}$ in $\csg^\star$ as
$\opt{\mu} := \max_{\sigma \in \Sigma}\mu(\sigma, P^\star)$; 
and, if $\opt{\mu}\ge 0$, define the \emph{social-welfare optimal Nash margin} as $\mu^\star := \mu(\sigma^\star, P^\star)$.
\end{definition}
Thus, $\mu(\sigma,P)$ measures how close a profile is to being an NE, while $\opt{\mu}$ captures how close $\csg^\star$ is to admitting an equilibrium overall. In particular, $\mu(\sigma,P) \geq 0$ iff $\sigma \in \NE(P)$, and $\opt{\mu} \ge 0$ iff $\NE(P^\star) \neq \emptyset$. 
Since an $\varepsilon$-NE exists in any finite CSG for all $\varepsilon > 0$, the global margin satisfies $\opt{\mu} \ge -\varepsilon$ for arbitrarily small $\varepsilon$. The use of $\max$ (rather than $\sup$) in the definition of $\opt{\mu}$ ensures that this margin is well-defined only when the maximum is attained by some profile.
Finally, $V^\star$ is well-defined iff $\mu^\star \ge 0$, which holds precisely when $\opt{\mu} \ge 0$, i.e., $\csg^\star$ admits an exact NE. We formalise this relationship in \Cref{prop:opt-mu-to-SW-mu}, \Cref{app:prelims}.

% \angel{Changed $\mu^\star$ to $\opt{\mu}$ in \Cref{prop:solver-EQ-guarantees}; can prob remove $\mu^\star$ in main text (and check consistency)}

This characterises our PAC formulation in \Cref{def:PAC-CSG}, which conditions on whether $\opt{\mu} \ge 0$. We present a concrete realisation of this framework in \Cref{alg:PAC-learning-CSG}. 

\begin{definition}[PAC-CSG]
\label{def:PAC-CSG}
An algorithm for a finite CSG is \emph{PAC} if, for any $\varepsilon, \delta\in(0,1)$, with probability at least $1-\delta$, after a number of interactions
polynomial in $|S|, |A|, 1/\varepsilon, \log(1/\delta)$ and problem constants, it outputs one of the following:
\begin{enumerate}[(a)]
  \item \textup{(Approximate equilibrium)} \label{item:EQ-output}
  An $\varepsilon$-NE profile $\hat{\sigma}$ such that, if $\opt{\mu} \ge 0$, then $V^\star - u(\hat{\sigma}, P^\star) \leq \varepsilon$, i.e., $\hat\sigma$ is $\varepsilon$-optimal; 
  in this case, the algorithm always terminates in this branch.
  \item \textup{(Near-non-existence certificate)} 
  \label{item:near-non-existence-cert}
  A sound declaration that $\opt{\mu}$ is not well-defined or $\opt{\mu} < 0$, i.e., no exact NE exists in $\csg^\star$. % profile.
\end{enumerate}
\end{definition}

Importantly, here $\varepsilon$-NE and $\varepsilon$-optimality are distinct requirements: the former bounds unilateral incentives to deviate, while the latter bounds the social-welfare loss relative to the optimal NE.

The existence branch~\ref{item:EQ-output} certifies that the true game admits a well-separated (SW)NE and returns a corresponding profile, while the non-existence branch~\ref{item:near-non-existence-cert} certifies that no exact NE exists based on the (well-definedness) and value of $\bar\mu$. 
Branch~\ref{item:near-non-existence-cert}, however, provides a \emph{sound} but \emph{incomplete} certificate: if the algorithm outputs $\opt{\mu}<0$, then $\csg^\star$ admits no exact NE. 
The converse need not hold: even when $\opt{\mu}<0$, the algorithm may still terminate via branch~\ref{item:EQ-output}, since $\varepsilon$-NE always exist. % and guarantees under $P^\star$ are only approximate. 
Intuitively, non-existence is detected only with sufficiently strong evidence (i.e., enough samples). We develop this condition in the next section.

% \vspace{-1em}
\begin{figure*}[h]
\centering
\resizebox{.9\linewidth}{!}{%
\begin{minipage}{\linewidth}
\input{algorithm/pac}
\end{minipage}
}
\end{figure*}

% \angel{Better title?}\dave{PAC-CSG Algorithms?}
\section{PAC-CSG guarantees}
\label{sec:theory}

We begin with finite-horizon objectives and extend the analysis to the infinite-horizon setting in \Cref{sec:infh}. Our approach follows a three-step pipeline:
\begin{enumerate*}[\arabic*)]
    \item In \Cref{sec:theory1-estim-error}, we relate statistical uncertainty in the transition kernel to value estimation error via a sensitivity lemma (\Cref{lem:sensitivity}), showing that solving a robust equilibrium in $\Punc_t$ yields an approximate NE in the true game once $\Delta_t \le \DeltaVal$.
    \item In \Cref{sec:exploration}, we construct an exploration RMDP and show that, under \Cref{ass:reachability}, every relevant $(s,a)$-pair is reachable, ensuring that targeted exploration drives $\Delta_t$ below the required threshold.
    \item Finally, in \Cref{sec:main-PAC-CSG}, we combine these results and use a Freedman martingale argument to lift stochastic per-episode visitation into a high-probability bound on the number of episodes, yielding the overall sample complexity.
\end{enumerate*}
Throughout, we condition on the high-probability event $\eventCI$ that $P^\star \in \Punc_t$ for all episodes $t \ge 1$, as established in \Cref{cor:true-kernel-containment}, \Cref{app:per-slot-proofs}.

\subsection{Estimation Error and Value Stability}
\label{sec:theory1-estim-error}
% First, we relate statistical uncertainty in the transition kernel to errors in value estimates.
% Second, we show that these value errors translate into equilibrium guarantees, ensuring that solving a robust equilibrium in the empirical game yields an approximate equilibrium in the true game.
% Third, we analyse the exploration process and show that all state--action pairs are sampled sufficiently often with high probability, ensuring that the uncertainty shrinks at a controlled rate.
% Combining these steps yields the PAC guarantee.

Our analysis begins by relating uncertainty in the transition kernel to errors in value estimates.
% At episode $t$, let $\Punc_t$ denote the confidence set around the empirical kernel $\hat P_t$, and 
Suppose $\eventCI$ holds, i.e., $P^\star \in \Punc_t$.
We show that this uncertainty induces a controlled deviation in the value of any fixed strategy profile.\footnote{In \Cref{alg:PAC-learning-CSG}, $\Delta_t$ is computed conservatively by maximising the transition error over the confidence set $\Punc_t$; on $\eventCI$, this computed quantity upper-bounds the true $\Delta_t$, and hence subsequent claims on $\Delta_t$ apply to the computed quantity too.}

\begin{restatable}[Sensitivity]{lemma}{LinearSensitivity}
\label{lem:sensitivity}
Define $\Delta_t := \tfrac{1}{2} \Rmax H^2\diffLI_t$, where $\diffLI_t := \max_{(s,a)\in\ntslots}\|P^\star_{sa}-\widehat{P}_{sa}\|_1$.
Then for any $P\in\Punc_t$, profile $\sigma$, player $i\in N$, and state $s_0\in S$,
$\lvert u_i(\sigma,P\mid s_0)-u_i(\sigma,\widehat{P}\mid s_0) \rvert \le \Delta_t$.
\reforig{lem:sensitivity}
% \refproof{prf:linear-sensitivity}
\end{restatable}
\begin{proof}[Proof sketch]
\ifICLR
At step $h$, the deviation between state marginals satisfies $\|\mu_h - \hat{\mu}_h\|_1 \le h\,\diffLI_t$ (\Cref{prop:tv-propagation}); summing the induced reward difference over $h<H$ and using $\sum_{h<H} h \le H^2/2$ yields the bound. Full proof in \Cref{prf:linear-sensitivity}.
\else
We track how a per-$(s,a)$ $L^1$ perturbation of size $\diffLI_t$ propagates through the trajectory distribution. At step $h$, the deviation between state marginals satisfies $\|\mu_h - \hat{\mu}_h\|_1 \le h\,\diffLI_t$ (\Cref{prop:tv-propagation}, \Cref{app:per-slot-proofs}). Summing the induced reward difference $\Rmax \|\mu_h - \hat{\mu}_h\|_1$ over $h=0,\ldots,H-1$ and using $\sum_{h=0}^{H-1} h \le H^2/2$ yields the bound. Full proof in \Cref{prf:linear-sensitivity}.
\fi
\end{proof}

% \begin{restatable}[Per-episode error]{corollary}{PerEpisodeError}
% \label{cor:per-episode-error}
% % Fix some episode $t\ge 1$. 
% % Let $\Punc_t$ be an uncertainty set constructed around the empirical kernel $\widehat{P}$ at this episode. Suppose $P^\star\in\Punc_t$ and define $\Delta_t$ as in \Cref{lem:sensitivity}.
% If $P^\star \in \Punc_t$, then for any profile $\hat{\sigma}$, 
% $u(\hat{\sigma},P^\star) \ge \rob{u}_{\Punc_t}(\hat{\sigma})$.
% % - 2\Delta_t$.
% % \reforig{cor:per-episode-error}
% \end{restatable}
% \begin{proof}[Proof sketch]
% Since $P^\star \in \Punc_t$ on event $\eventCI$, by definition $\rob{u}_{\Punc_t}(\hat{\sigma})=\inf_{P\in\Punc_t}{u(\hat{\sigma},P)} \le u(\hat{\sigma},P^\star)$.
% % The result follows by comparing $P^\star$ and $\widehat{P}_t$ via \Cref{lem:sensitivity}, and using the fact that $\widehat{P}_t \in \Punc_t$.
% % Since $\widehat{P}_t \in \Punc_t$ by construction, $\rob{u}_{\Punc_t}(\hat\sigma) \le u(\hat\sigma, \widehat{P}_t)$. On $\eventCI$ we also have $P^\star \in \Punc_t$. Applying \Cref{lem:sensitivity} (once per player) with $P \leftarrow P^\star$ yields
% % $|u(\hat\sigma, \widehat{P}) - u(\hat\sigma, P^\star)| \leq 2\Delta_t$.
% % Combining the two inequalities gives the claim. Full proof in \Cref{prf:per-episode-error}.
% \end{proof}

Thus, on $\eventCI$, for any fixed profile $\hat\sigma$, its value in the true game $\csg^\star$ is close to its robust value over $\Punc_t$, up to an error of order $\Delta_t$, which vanishes as the confidence radius shrinks with additional data.
The following lemma further lets us reason about NE existence in $\csg^\star$ by solving for a robust equilibrium over $\Punc_t$ (using a black-box RCSG solver; see \Cref{ass:solver}, \Cref{app:EQ-existence-proofs}): it shows that any true NE induces an approximate RNE in $\Punc_t$, with quality degrading proportionally in $\Delta_t$. Consequently, solver failure to find an $\varepsilon$-RNE implies that no NE in $\csg^\star$ has sufficiently large margin, i.e., $\ge 4\Delta_t-\varepsilon$. 

\begin{restatable}[NE($P^\star$) $\Rightarrow$ approximate RNE($\Punc_t$)]{lemma}{NERNETransfer}
\label{lem:NE-RNE-transfer}
    If $\sigma$ is an (exact) NE under $P^\star$ achieving Nash margin $\mu := \mu(\sigma,P^\star) \ge 0$, then it is also an $(4\Delta_t-\mu)$-RNE under $\Punc_t$.
    \reforig{lem:NE-RNE-transfer}
    % \refproof{prf:NE-RNE-transfer}
\end{restatable}
\begin{proof}[Proof sketch]
For any player $i\in\{1,2\}$ and unilateral deviation $\sigma_i'$, \Cref{lem:sensitivity} bounds the difference in deviation gains under $P^\star$ and any $P\in\Punc_t$ by $4\Delta_t$.
Since $\sigma$ has Nash margin $\mu$ under $P^\star$, its deviation gain there is at most $-\mu$.
Thus, under every $P\in\Punc_t$, the deviation gain is at most $4\Delta_t-\mu$, implying that $\sigma$ is an $(4\Delta_t-\mu)$-RNE.
Full proof in \Cref{prf:NE-RNE-transfer}.
\end{proof}

% This yields a profile whose realised value in $\csg^\star$ is close to its robust value.
% This completes the first step of the analysis: controlling the transfer from robust to true values. 
% The remaining task is to relate this to the SWNE value $V^\star$ (if one exists). 
% and ensure that $\Delta_t$ decreases sufficiently fast, which we address via targeted exploration (\Cref{sec:exploration})

% \begin{restatable}[Episode-wise existence]{corollary}{EpisodewiseExistence}
% \label{prop:solver-EQ-guarantees}
% On $\eventCI$:
% \begin{enumerate}[label=(\arabic*)]
%   \item \textup{(Solver succeeds $\Rightarrow$ $\varepsilon$-NE in $\csg^\star$)}
%   \label{item:solver-success-cert}
%   If $\SolveRCSG(\Punc_t, \varepsilon)$ returns $\hat{\sigma} = \epsRSWNEop{\varepsilon}(\Punc_t)$, then $\hat{\sigma}$ is an $\varepsilon$-NE\footnote{Note that if $\epsRNEop{\varepsilon}(\Punc_t) \neq \emptyset$, then $\epsNEop{\varepsilon}(P^\star) \neq \emptyset$.} of $\csg^\star$.

%   \item \textup{(Solver failure $\Rightarrow$ small Nash margin)}
%   \label{item:epsRNE-existence-suff-cond}
%   For any $\sigma \in \epsNEop{\varepsilon}(P^\star)$, if $\mu(\sigma,P^\star) \ge \nashMarginThresh$, then $\sigma$ is also an $\varepsilon$-RNE under $\Punc_t$ and the solver returns $\Found$.
%   Equivalently, if the solver returns $\found = \false$, then it must be that $\mu(\sigma,P^\star) < \nashMarginThresh$.
% \end{enumerate}
% % \reforig{prop:solver-EQ-guarantees}
% % \refproof{prf:episodewise-existence}
% \end{restatable}

\begin{restatable}[Solver guarantees for equilibrium computation]{proposition}{EpisodewiseExistence}
\label{prop:solver-EQ-guarantees}
On the event $\eventCI$:
\begin{enumerate}[(1)]
\item (\emph{Equilibrium transfer}) \label{item:solver-success-cert}
If the solver returns an $\varepsilon$-RSWNE $\hat\sigma$ of $\Punc_t$, then $\hat\sigma$ is an $\varepsilon$-NE of 
% the true game~
$\csg^\star$.

\item (\emph{Failure implies near non-existence}) 
\label{item:epsRNE-existence-suff-cond}
% For any $\sigma \in \epsNEop{\varepsilon}(P^\star)$, if $\mu(\sigma,P^\star) \ge \nashMarginThresh$, then $\sigma$ is also an $\varepsilon$-RNE under $\Punc_t$ and the solver returns $\found$. Equivalently, 
If the solver returns $\notFound$, then either $\mu^\star$ is undefined or $\mu^\star < \nashMarginThresh$. 
\end{enumerate}
\reforig{prop:solver-EQ-guarantees}
\end{restatable}

\begin{proof}[Proof sketch]
\ifICLR
\begin{enumerate*}[(1)]
\item Any $\varepsilon$-RSWNE under $\Punc_t$ is an $\varepsilon$-RNE under $\Punc_t$, hence an $\varepsilon$-NE for every $P\in\Punc_t$, including $P^\star$ (on $\eventCI$).
\item If the solver returns $\notFound$, no $\varepsilon$-RNE exists under $\Punc_t$ (\Cref{ass:solver}). Either no NE exists under $P^\star$, leaving $\mu^\star$ undefined; or the SWNE exists, and $\mu^\star \ge 4\Delta_t-\varepsilon$ would make it an $\varepsilon$-RNE under $\Punc_t$ by \Cref{lem:NE-RNE-transfer}, contradicting the failure.
\end{enumerate*}
Full proof in \Cref{prf:episodewise-existence}.
\else
\begin{enumerate*}[(1)]
\item
Any $\varepsilon$-RSWNE under $\Punc_t$ is, by definition, an $\varepsilon$-RNE under $\Punc_t$, and hence an $\varepsilon$-NE for every $P\in\Punc_t$, including $P^\star$ (on $\eventCI$).

\item
If the solver returns $\notFound$, then by \Cref{ass:solver} no $\varepsilon$-RNE exists under $\Punc_t$. We then distinguish two cases. If no NE exists under $P^\star$, then the SWNE does not exist and $\mu^\star$ is undefined. Otherwise, the SWNE exists, and if $\mu^\star \ge 4\Delta_t-\varepsilon$, the NE--RNE transfer lemma (\Cref{lem:NE-RNE-transfer}) % , proved using \Cref{cor:per-episode-error}) 
would imply that the SWNE is an $\varepsilon$-RNE under $\Punc_t$, contradicting the solver's failure.
\end{enumerate*}
Full proof in \Cref{prf:episodewise-existence}.
\fi
\end{proof}

Together, these results link equilibrium computation in the empirical model to the true game: when $\mu^\star \ge \nashMarginThresh$, the solver returns an approximate NE; when $\mu^\star < \nashMarginThresh$, solver failure yields a sound certificate of non-existence once uncertainty is sufficiently small. Combined with the value stability results, this yields the following per-episode guarantee.

\begin{restatable}[Per-episode value gap]{corollary}{PerEpisodeGap}
\label{cor:anytime-value-gap}
% Suppose $\NE(P^\star) \ne \emptyset$, so that $V^\star$ is well-defined. 
If $\mu^\star \ge \nashMarginThresh$, then 
$V^\star - u(\hat{\sigma}_t, P^\star) \le \DeltaValCoef \Delta_t$.
\reforig{cor:anytime-value-gap}
\refproof{prf:anytime-value-gap}
\end{restatable}
% \angel{Think it's ok to write $\hat{\sigma}_t$ here since we introduced the index $t$ generically?}

It therefore remains to ensure that $\Delta_t$ falls below $\DeltaVal$ in a finite number of samples. Since $\Delta_t$ scales with the maximum per-$(s,a)$ confidence radius, this requires sufficient coverage of all relevant slots. We tackle this using a targeted, RMDP-based exploration scheme presented in the next section.

\subsection{Exploration}
\label{sec:exploration}
A key challenge in PAC learning of concurrent games is obtaining sufficient joint state--action coverage from stochastic trajectories under transition uncertainty. To isolate this statistical challenge, we assume centralised exploration, allowing the learner to coordinate the players’ joint actions. To this end, we construct an auxiliary \emph{exploration RMDP} that casts coverage of under-visited $(s,a)$-pairs as a robust reachability problem.
% whose robust objective maximises the worst-case probability of reaching under-visited state--action pairs.
For each $(s,a)$, we maintain a boolean flag indicating whether the slot is \emph{known}, i.e., whether it has been visited sufficiently many times; this threshold is precisely defined in \Cref{lem:nmin}, \Cref{app:exploration}. Let $\unknown_t$ and $\known_t$ denote the sets of unknown and known (non-target) slots at the beginning of episode $t$, respectively (so $\unknown_1=\ntslots$; see formal definition in \Cref{app:exploration}).

\paragraph{Exploration RMDP.}
% For a given episode $t$ and empirical uncertainty set $\Punc_t$, 
At each episode $t$, we construct an \emph{exploration RMDP} $\explore{\csg}_t$ that coincides with the empirical RCSG $\csg_t$ on all known slots. On each unknown slot $(s,a)\in \unknown_t$, the model instead transitions deterministically to an auxiliary absorbing state $z$, with a nonzero reward accrued upon this transition. Note that this absorbing construction is used for analysis only (see \Cref{app:exploration}). The learner’s objective is to maximise expected reward, which is equivalent to maximising the probability of reaching unknown slots under this construction. 
% \footnote{In the implementation, trajectories are not terminated upon visiting $\unknown_t$. Since continuing the trajectory can only increase the probability of visiting $\unknown_t$, all lower bounds in \Cref{lem:exploration-prob} remain valid.}

\begin{definition}[Exploration RMDP]
\label{def:exploration-RMDP}
The \emph{exploration RMDP} at episode $t$ is an $L^1$-MDP $\explore{\csg}_t = (\explore{S}, s, \explore{A}, \explore{\Punc}_t, \explore{r}_t)$ where $\explore{S} := S \union \{z\}$ and $\explore{A} := A_1 \times A_2$. 
For each $(s,a)\in \explore{S} \times \explore{A}$:~if $(s,a)\in \unknown_t \lor s=z \lor s\in S_T$ then $\explore{\Punc}_t(s,a) = \{\delta_z\}$, where $\delta_z$ is the Dirac distribution that assigns probability 1 to $z$ and $0$ to all others; else $\explore{\Punc}_t(s,a) = \Punc_t(s,a)$;
% \[
% \explore{\Punc}_t(s,a) =
% \begin{cases}
%     \delta_z & \ift (s,a)\in \unknown_t \lor s=z,\\
%     \Punc_t(s,a) & \otherwiset;
% \end{cases}
% \]
$\explore{r}_t$ is such that $\explore{r}_t(s,a) := \ind{(s,a)\in \unknown_t}$.
\end{definition}

Next, let $\explore{p}_t(P) := \prob^{\explore{\sigma}_t, P}\left[\reachopt\ \unknown_t \right]$, i.e., the probability of reaching any unknown $(s,a)$-pair during episode $t$. 
Define the exploration profile as
\begin{equation}
\label{eq:exploration-profile}
\explore{\sigma}_t
\in \arg\max_{\sigma\in \Sigma}
\rob{u}_{\explore{\Punc}_t}(\sigma)
= \arg\max_{\sigma\in \Sigma}
\inf_{P\in \explore{\Punc}_t}{\explore{p}_t(P)}.
\end{equation}
This pessimistic (robust) objective is essential for the high-probability global coverage guarantee (\Cref{lem:counting-episodes}, \Cref{app:PAC-CSG-proofs}) used to prove \Cref{thm:PAC-CSG}: a best-case visitation probability under $P\in\Punc_t$ need not lower-bound the visitation probability under the true kernel $P^\star$.

\begin{assumption}[Graph reachability]
\label{ass:reachability}
Let $\Punc_{\supp} := \{P: \supp(P)=\supp(P^\star)\}$.
The game graph of $\csg^\star$ is such that every $(s,a)\in \ntslots$ is reachable from $\bar{s}$ under some profile, i.e.,
\begin{equation}
\label{eq:p-reach}
  \reach{p} := 
  \min_{(s,a)\in \ntslots}
  \ \max_{\sigma\in\Sigma} \ 
  \inf_{P\in\Punc_{\supp}}
  \prob^{\sigma,P}[\emph{\text{reach $(s,a)$ before entering $S_T$}}] > 0.
\end{equation}
% This quantity depends only on the known graph structure of $\csg^\star$ and is an instance-dependent constant.
\end{assumption}
Reachability conditions like \Cref{ass:reachability} are standard in prior work %, even for MDPs 
\citep{jaksch2010near,al-marjani2023active,brafman2002RMAX}. 
Since the transition support of $P^\star$ is known as input to \Cref{alg:PAC-learning-CSG}, $\reach{p}$ can be precomputed by solving a corresponding reachability RMDP over the known game graph (see \Cref{app:helper-algo} for details).
% by solving a reachability RMDP over the known support: for each $(s,a)$ with $s\notin S_T$, compute $\max_{\sigma}\inf_{P: \supp(P)=\supp(P^\star)}\prob^{\sigma,P}[\reachopt \ (s,a)]$ via value iteration on the graph, then take the minimum. This can be done in polynomial time in $|S|$ and $|A|$ \citep{iyengar2005robust}. 
%
Together with the exploration RMDP construction, this yields the following result; the full proof is in \Cref{prf:exploration-prob}.

\begin{restatable}[Exploration probability]{lemma}{ExplorationProb}
\label{lem:exploration-prob}
For all episodes $t\geq 1$ with $\unknown_t\neq\emptyset$,
$\explore{p}_t(P^\star) \geq \reach{p}$.
\reforig{lem:exploration-prob}
% \refproof{prf:exploration-prob}
\end{restatable}
% \begin{proof}[Proof sketch]
% By \Cref{ass:reachability}, every $(s,a)$ is reachable under some profile.
% The exploration RMDP is constructed so that maximising reward corresponds to maximising the probability of reaching unknown slots.
% Full proof in \Cref{prf:exploration-prob}.
% \end{proof}

We execute $\explore{\sigma}_t$ in $\csg^\star$ to sample trajectories. By collecting multiple samples per episode (with the number specified in \Cref{prop:num-samples-per-episode}, \Cref{app:exploration}), each episode visits an unknown slot with nontrivial~probability. Repeating this across episodes yields high-probability coverage of all relevant slots, ensuring that each is visited sufficiently often. 
The precise sample and episode counts are derived in \Cref{lem:counting-episodes} and 
\Cref{cor:nmin} (\Cref{app:exploration}), and are used to establish the sample complexity bound in the next section.
% \angel{Note that the probability of reaching an unknown slot in the game that does not stop on the first visit of an unknown slot is $\ge \explore{p}_t(P^\star)$, thus \Cref{lem:exploration-prob} also holds for trajectories that continue visiting unknown $(s,a)$.}

\subsection{PAC-CSG}
\label{sec:main-PAC-CSG}
We now combine the previous results to obtain the PAC guarantee, in line with \Cref{def:PAC-CSG}.
% To ensure near-optimality, it suffices to achieve $\Delta_t \le \DeltaVal$.

\begin{restatable}[PAC guarantee for finite-horizon objectives]{theorem}{PACCSG}
\label{thm:PAC-CSG}
Set $n_{\min} = \bigOtilde{\Rmax^2 H^4 |S|/\varepsilon^2}$.
Then with probability at least $1-\delta$, \Cref{alg:PAC-learning-CSG} terminates after at most
\begin{equation}
\label{eq:PAC-sample-complexity}
\bigOtilde{\frac{\Rmax^2 H^4 |S|^2 |A|}{\varepsilon^2 \reach{p}}}
\end{equation}
trajectory samples, where $\bigOtilde{\cdot}$ suppresses logarithmic factors in the problem parameters. % $|S|$, $|A|$ and $1/\delta$, 
% and outputs one of: % the following:
% \begin{enumerate}[(a)]
%     \item \textup{(Approximate equilibrium)} A profile $\hat{\sigma}$ that is an $\varepsilon$-NE of $\csg^\star$. 
%     If $\opt{\mu}\ge 0$, then ${V^\star - u(\hat{\sigma},P^\star) \le \varepsilon}$, and the algorithm terminates in this branch.
%     % and an $\varepsilon$-RSWNE of the empirical $L^1$-CSG at termination. 
%     % This branch is always taken when $\opt{\mu} \ge \nashMarginThreshVal$.
%     \item \textup{(Near-non-existence certificate)} A sound declaration that $\opt{\mu}$ is not well-defined or $\opt{\mu} < 0$, i.e., no exact NE exists in $\csg^\star$.
% \end{enumerate}
% \refproof{prf:PAC-CSG}
Upon termination, the algorithm's output satisfies the PAC-CSG guarantee of \Cref{def:PAC-CSG}.
\reforig{thm:PAC-CSG}
\end{restatable}
\begin{proof}[Proof sketch]
\ifICLR
We split the failure budget as $\confCI=\confCount=\delta/2$ and condition on $\mathcal{E}:=\eventCI\cap \eventCount$, which holds w.p.\ $\ge 1-\delta$. On $\eventCI$, \Cref{cor:anytime-value-gap} reduces near-optimality to $\Delta_t\le \DeltaVal$, which holds once every $(s,a)\in\ntslots$ is visited $n_{\min}=\bigOtilde{\Rmax^2 H^4 |S|/\varepsilon^2}$ times (\Cref{lem:nmin}). Under \Cref{ass:reachability}, \Cref{lem:counting-episodes} achieves this after $\bigOtilde{n_{\min}|S||A|}$ episodes: letting $X_t$ indicate whether some sampled trajectory reaches an unknown slot, these indicators are history-dependent, so we apply Freedman's inequality to their martingale differences with the lower bound $\reach{p}$ from \Cref{lem:exploration-prob}. With $\bigOtilde{\log{t} / \reach{p}}$ trajectories per episode (\Cref{prop:num-samples-per-episode}), this gives the stated complexity; \Cref{prop:solver-EQ-guarantees} yields the output branching. Full details in \Cref{prf:PAC-CSG}.
\else
We split the failure budget as $\confCI=\confCount=\delta/2$ and condition on $\mathcal{E}:=\eventCI\cap \eventCount$, which holds with probability $\ge 1-\delta$. On $\eventCI$, \Cref{cor:anytime-value-gap} reduces near-optimality to ensuring $\Delta_t\le \DeltaVal$, which holds once every $(s,a)\in\ntslots$ is visited at least $n_{\min}=\bigOtilde{\Rmax^2 H^4 |S|/\varepsilon^2}$ times (\Cref{lem:nmin}). Under \Cref{ass:reachability}, \Cref{lem:counting-episodes} shows that this coverage is achieved after $\bigOtilde{n_{\min}|S||A|}$ episodes: 
% each episode discovers an unknown slot with probability at least $\reach{p}$ (\Cref{lem:exploration-prob}), and a Freedman martingale argument lifts this to a high-probability global coverage bound. 
for each episode, let $X_t$ indicate whether at least one sampled trajectory reaches an unknown slot. Since these indicators are history-dependent, we apply Freedman's inequality to their martingale differences, using the lower bound $\reach{p}$ from \Cref{lem:exploration-prob}, to obtain a high-probability bound on the number of episodes needed until all relevant slots are sufficiently visited.
Moreover, \Cref{prop:num-samples-per-episode} yields $\bigOtilde{\log{t} / \reach{p}}$ trajectories per episode $t$, giving the stated sample complexity. Finally, \Cref{prop:solver-EQ-guarantees} yields the output branching. Full details are in \Cref{prf:PAC-CSG}.
\fi
\end{proof}

% \begin{proof}[Proof sketch]
% The result follows by combining three components:
% \begin{enumerate*}[(i)]
%     \item containment ensures $P^\star \in P_t$ for all $t$,
%     \item value stability implies that $\Delta_t \le \varepsilon/4$ yields near-optimality, and
%     \item exploration guarantees that every state--action pair is visited sufficiently often.
% \end{enumerate*}
% Together, these imply that the stopping condition is reached after
% $\bigOtilde{\Rmax^2 H^4 |S|^2 |A| / \varepsilon^2}$ episodes.
% Full proof in \Cref{prf:PAC-CSG}.
% \end{proof}

\paragraph{Infinite-horizon extension.}
\label{sec:infh}
% Unlike the finite-horizon case, where exact NE exist by backward induction, infinite-horizon objectives may admit no NE, activating the non-existence branch~\ref{item:near-non-existence-cert} of \Cref{def:PAC-CSG}. 
% For our objectives of interest (\Cref{sec:objectives}),
\ifICLR
We further assume almost-sure reachability of $S_T$ (\Cref{ass:nz-almost-sure-reachability}), following prior work~\citep{kwiatkowska2021automatic,bertsekas1991analysis,HP26}; this ensures value iteration convergence but not the existence of stationary NE. The process then admits an \emph{effective horizon} $\Heff \le |S|/p_T$ upper-bounding the stopping time (\Cref{lem:effective-horizon}), where $p_T$ is the minimum worst-case probability of reaching $S_T$ within $|S|$ steps from any state. Replacing $H$ with $\Heff$ yields the corresponding infinite-horizon guarantees; see \Cref{app:infh} for full details.
\else
We further assume almost-sure reachability of $S_T$, i.e., $S_T$ is reached with probability $1$ from every state under every $P\in\Punc_{\supp}$ by some profile (\Cref{ass:nz-almost-sure-reachability}), following prior work~\citep{kwiatkowska2021automatic,bertsekas1991analysis,HP26}. This ensures value iteration convergence but not the existence of stationary NE.
Under this assumption, the process admits an \emph{effective horizon} $\Heff$ which upper bounds the stopping time, with $\Heff \le |S|/p_T$ (see \Cref{lem:effective-horizon}). Here, $p_T$ is the minimum worst-case probability of reaching $S_T$ within $|S|$ steps from any state. % , under any support-preserving kernel $P$.
Replacing $H$ with $\Heff$ yields the corresponding infinite-horizon guarantees, with constants adjusted accordingly. Full details are in \Cref{app:infh}.
\fi

% \paragraph{Objectives.}
% Fix a target set $S_T \subseteq S$, and define the stopping time 
% $\tau_T := \min\{j \geq 0 : s_j \in S_T\}$.
% The objectives are:
% \begin{itemize}[nosep]
%   \item \emph{Probabilistic reachability}:
%     $X(\pi) = \ind{\tau_T < \infty}$; and
%   \item \emph{Reachability reward / Stochastic shortest path (SSP)}:
%     $X(\pi) = \sum_{h=0}^{\tau_T-1} r(\pi(h), \pi[h])$.
% \end{itemize}

\section{Experimentation}
\label{sec:experiments}

We evaluate our method on a suite of benchmark CSGs to validate our theoretical guarantees.

\paragraph{Experimental setup.}
\label{sec:experimental-setup}
% \ifICLR
% We implement \Cref{alg:PAC-learning-CSG} in Java and evaluate it in the PRISM-games model checker~\citep{KNPS20}, comparing our learners' output against its CSG solver as ground truth. Unless otherwise stated, we set $\delta = 0.05$, $\varepsilon=0.2$ and $\Rmax = 1$; for RQ1 we instead use $\varepsilon=0.1$. Results are averaged over 10 random seeds for trajectory sampling.
% For the learning-dynamics experiments in \Cref{app:RQ-learning-dynamics}, we additionally evaluate the empirical model and strategy against the known ground-truth CSG; these quantities are used only for evaluation and are not available to the learner.
% Experiments ran on a 3.2~GHz Apple M1 CPU, 16~GB RAM; our implementation, models, and property specifications are available at \githubRepo.
% \else
We implement \Cref{alg:PAC-learning-CSG} in Java and evaluate it in the PRISM-games model checker~\citep{KNPS20}. Unless otherwise stated, we set $\delta = 0.05$, $\varepsilon=0.2$ and $\Rmax = 1$.
For RQ1, we instead use $\varepsilon=0.1$. Results are averaged over 10 random seeds for trajectory sampling.
% ; error bars indicate one standard deviation.
For the learning-dynamics experiments in \Cref{app:RQ-learning-dynamics}, we additionally evaluate the empirical model and strategy against the known ground-truth CSG; these quantities are used only for evaluation and are not available to the learner.
% Experiments ran on a 3.2~GHz Apple M1 CPU, 16~GB RAM.
Experiments ran on an Apple M4 Pro with 24~GB RAM. Peak memory usage never exceeded 0.5~GB; the benchmarks are compute-bound rather than memory-bound, with the largest configuration requiring approximately 5~hours per seed.
Our implementation, including the models and property specifications, is available at \githubRepo. We compare our learners' output against the PRISM-games CSG solver as ground truth.
% \fi
% Full results and additional details are given in \Cref{app:additional-experimental-results}.

\paragraph{Benchmarks.}
\label{sec:benchmarks}
We evaluate on six small, interpretable CSGs, chosen to enable repeated exact NE computation and comparison against exact solutions. These are illustrated in \Cref{fig:benchmarks}, \Cref{app:additional-experimental-results}. Each benchmark is paired with a set of finite- and infinite-horizon properties expressed in \emph{probabilistic alternating-time temporal logic with rewards} (rPATL)~\citep{kwiatkowska2021automatic,chen2012automatic}, designed to isolate a distinct challenge:
\ifICLR
\emph{Cyclic Preferences}, \emph{Delayed Coordination}, \emph{Hide-or-Run}~\citep{de2007concurrent}, \emph{Mixed NE}, \emph{Safe vs.\ Risky}, and \emph{Traffic Merge}; we detail the challenge each isolates in \Cref{app:benchmark-descriptions}.
\else
\begin{enumerate*}[1)]
\item \emph{Cyclic Preferences} has no exact stationary NE for infinite-horizon reachability reward, testing detection of non-existence;
\item \emph{Delayed Coordination} isolates hard exploration under sparse, delayed rewards;
\item \emph{Hide-or-Run} adapted from \citet{de2007concurrent} has no stationary NE profiles; its value is only achieved in the limit via mixed strategies converging to ``hide w.p.\ $1$'' for the zero-sum objective of player~1 eventually reaching $s_1$; 
\item \emph{Mixed NE} requires a mixed equilibrium;
\item \emph{Safe vs.\ Risky} contrasts actions with differing sensitivity to transition probabilities; and
\item \emph{Traffic Merge} evaluates safety-critical multi-step coordination.
\end{enumerate*}%

\fi%
Together, they cover all objective types in our framework.

% \item \emph{Hide-or-Run} exhibits a value--strategy gap (equilibrium value exists but no attaining strategy does); 

We aim to address the following research questions:

\paragraph{RQ1: Correctness.}
\begin{enumerate*}[(a)]
    \item \textbf{Near-optimality.} 
    When $\csg^\star$ admits an NE, does the algorithm output an $\varepsilon$-optimal $\varepsilon$-NE?  
    % We compare the profiles from our learner and the CSG solver oracle to assess the performance of the learned profile in the true game.

    \item \textbf{Equilibrium existence detection.} 
    When no exact stationary NE exists, when does the algorithm return a sound non-existence certificate, and when does incompleteness arise?
    % Can the algorithm correctly distinguish between games that admit an exact NE and those that do not? 
    % We evaluate whether the method returns an $\varepsilon$-NE when $\opt{\mu} \ge 0$, and produces a valid non-existence certificate when $\opt{\mu} < 0$.
\end{enumerate*}

\paragraph{RQ2: Exploration strategy.}
How does the robust exploration RMDP compare with simpler or optimistic exploration rules under an otherwise identical PAC loop and stopping condition? We compare against:
\begin{enumerate*}[(i)]
    \item uniform-random joint-action exploration;
    \item round-robin exploration targeting the least-visited unknown slot under the point-estimate model; and
    \item optimistic exploration, which maximises \emph{best-case} rather than worst-case reachability (see \Cref{eq:exploration-profile}) over the current uncertainty set.
\end{enumerate*}
We define each explorer in detail in \Cref{app:exploration-baselines}.

\paragraph{RQ3: Empirical sample-complexity scaling.}
\label{RQ3-sample-complexity}
How does the number of sampled trajectories until convergence scale with the number of states $|S|$, joint actions $|A|$, horizon $H$ and precision $\varepsilon$?
We use parametrised variants of \textit{Safe vs. Risky} that vary one problem dimension at a time while preserving the underlying equilibrium structure. Specifically, we test $|S|\in \{4,8,12,20,36,60\}$, $|A| \in \{4,9,16,25,36,49,64\}$, $H\in \{1,2,3,5,7,9,11,15\}$ and $\varepsilon\in \{0.05, 0.1, 0.15, 0.2, 0.3, 0.4\}$.
For each scaling experiment, all parameters other than the varied dimension are held at the default values $|S|=4, |A|=4$ (as in \Cref{fig:safe-vs-risky}), with $H=3,\varepsilon=0.2$, and the same property used throughout.
% \fi
% We study the bounded reachability objective $\llangle p_1:p_2 \rrangle_{\max=?} \big(\Pt[\eventually^{\le H} s_1] + \Pt[\eventually^{\le H} s_2]\big)$ in the \emph{Safe vs.\ Risky} game

\subsection{Results and Discussion}
\label{sec:discussion}

% \angel{Check what happened to the strategies in delayed-coord1 here}

\begin{table}[!t]
\centering
\caption{Comparison of our algorithm's output with the CSG-solver oracle across benchmarks. All cases use $\varepsilon=0.1$. The returned profile attains the oracle's true-game value whenever a stationary NE exists. For \emph{Cyclic Preferences} and the first \emph{Hide-or-Run} case, no stationary NE exists; the algorithm returns $\notFound$ in the former and an $\varepsilon$-NE approaching the limiting equilibrium value in the latter. For the second \emph{Safe vs.\ Risky} case, the NE value is $\infty$, which our algorithm also recovers. The final column reports the robust-value estimate gap $V^\star-\hat{V}$. Full results are in \Cref{tab:full-results}, App.~\ref{app:full-experimental-results}.}
\label{tab:correctness}
\resizebox{\columnwidth}{!}{
\begin{tabular}{l l c c c c}
\toprule
\textbf{Case study} & \textbf{Property} & $\boldsymbol{\varepsilon}$\textbf{-NE found?} & \textbf{NE exists?} & $\boldsymbol{V^\star - \hat{V}}$ ($\times10^{-3}$) \\
% & $\boldsymbol{V^\star - \rob{u}_{\Punc_T}(\hat{\sigma})}$ \\
\midrule
Cyclic Prefs. 
& $\llangle p_1:p_2 \rrangle_{\max=?}\left(\rewop{r_1}[\eventually s_3] + \rewop{r_2}[\eventually s_3] \right)$ 
& \xmark  & \xmark  & -- \\

% delayed\_coord
% & $\llangle p_1:p_2 \rrangle_{\max=?}
% \left(\probop{}[\eventually \mathrm{goal}] + \probop{}[\eventually \mathrm{goal}]\right)$ 
% & \cmark & \cmark & $0.2001$ \\
% action at state 4 should not matter, but is the only part where the true and robust strategy differ

Delayed Coord.
& $\llangle p_1:p_2 \rrangle_{\max=?} \left( \rewop{r_1}[\Ct^{\le 5}] + \rewop{r_2}[\Ct^{\le 5}] \right)$ 
& \cmark & \cmark & $1.573 \pm 0.069$  \\

Hide-or-run
& $\llangle p_1 \rrangle {\probop{\max=?}}
\left[\eventually s_1\right]$
& \cmark & \xmark & $0.000 \pm 0.000$ \\

Hide-or-run 
& $\llangle p_1 \rrangle \probop{\max=?}
\left[\eventually^{\le 5} s_1\right]$
& \cmark & \cmark & $0.000\pm 0.000$\\

Mixed NE
& $\llangle p_1:p_2 \rrangle_{\max=?}
\left(\Pt[\eventually s_1] + \Pt[\eventually s_2]\right)$
& \cmark & \cmark & $3.083\pm 0.000$\\

Safe vs.\ Risky
& $\llangle p_1:p_2 \rrangle_{\max=?}
\left(\Pt[\eventually s_1] + \Pt[\eventually s_2]\right)$
& \cmark & \cmark & $1.743\pm 0.000$ \\

% safe\_risky 
% & \texttt{<<p1:p2>>max=?( P[F<=k s=1] + P[F<=k s=2] )} 
% & \cmark & \cmark & -- & $0.0026$ \\

Safe vs.\ Risky 
& $\llangle p_1:p_2 \rrangle_{\max=?}
\left(\rewop{r_1}[\eventually s_1] + \rewop{r_2}[\eventually s_2]\right)$
& \cmark & \cmark ($\infty$) & -- \\

Traffic Merge & $\llangle p_1 \rrangle \probop{\max=?} [ \eventually^{\le 5} s_4 ]$ & \cmark & \cmark & $0.000\pm 0.000$ \\

\bottomrule
\end{tabular}
}
\end{table}

\ifICLR\else

\begin{table}[!t]
\centering
\caption{Samples to convergence (millions) and value gap $V^\star-\hat{V}$ for our robust exploration RMDP and the optimistic, round-robin, and uniform-random baselines. Values are reported as mean $\pm$ standard deviation; the lowest value gap and sample count for each benchmark are shown in bold.}
\label{tab:explorer-cmp}

\renewcommand{\arraystretch}{1.2}
\setlength{\tabcolsep}{5pt}

\resizebox{\columnwidth}{!}{
\large
\begin{tabular}{l cc cc cc cc}
\toprule
& \multicolumn{2}{c}{\makecell{\textbf{Safe vs.\ Risky}\\
$\llangle p_1:p_2 \rrangle_{\max=?}
\left( \probop{}[\eventually^{\le 3} s_1] + \probop{}[\eventually^{\le 3} s_2] \right)$}}
& \multicolumn{2}{c}{\makecell{\textbf{Traffic Merge}\\
$\llangle p_1 \rrangle \probop{\max=?}
\left[ \eventually^{\le 5} s_4 \right]$}}
& \multicolumn{2}{c}{\makecell{\textbf{Cyclic Prefs.}\\
$\llangle p_1:p_2 \rrangle_{\max=?}
\left( \rewop{r_1}[\eventually s_3] + \rewop{r_2}[\eventually s_3] \right)$}}
& \multicolumn{2}{c}{\makecell{\textbf{Delayed Coord.}\\
$\llangle p_1:p_2 \rrangle_{\max=?}
\left( \probop{}[\eventually s_3] + \probop{}[\eventually s_3] \right)$}} \\
\cmidrule(lr){2-3}
\cmidrule(lr){4-5}
\cmidrule(lr){6-7}
\cmidrule(lr){8-9}

\textbf{Explorer}
& \textbf{Samples} & $V^\star-\hat{V}$
& \textbf{Samples} & $V^\star-\hat{V}$
& \textbf{Samples} & $V^\star-\hat{V}$
& \textbf{Samples} & $V^\star-\hat{V}$ \\

\midrule

Robust (ours) &
$\mathbf{5.90 \pm 0.00}$ &
$\mathbf{0.005 \pm 0.000}$ &
$\mathbf{6.90 \pm 0.00}$ &
$\mathbf{0.000 \pm 0.000}$ &
$\mathbf{18.41 \pm 0.01}$ &
-- &
$\mathbf{42.95 \pm 0.00}$ &
$\mathbf{0.000 \pm 0.000}$ \\

Optimistic &
$\mathbf{5.90 \pm 0.00}$ &
$\mathbf{0.005 \pm 0.000}$ &
$\mathbf{6.90 \pm 0.00}$ &
$\mathbf{0.000 \pm 0.000}$ &
$\mathbf{18.41 \pm 0.01}$ &
-- &
$43.15 \pm 0.87$ &
$0.060 \pm 0.097$ \\

Round-robin &
$6.09 \pm 0.09$ &
$0.005 \pm 0.000$ &
$16.22 \pm 0.00$ &
$\mathbf{0.000 \pm 0.000}$ &
$19.07 \pm 0.01$ &
-- &
$44.92 \pm 0.83$ &
$0.060 \pm 0.097$ \\

Uniform &
$11.61 \pm 0.01$ &
$0.005 \pm 0.000$ &
$7.02 \pm 0.00$ &
$\mathbf{0.000 \pm 0.000}$ &
$46.84 \pm 0.02$ &
-- &
$730.09 \pm 0.28$ &
$0.180 \pm 0.063$ \\

\bottomrule
\end{tabular}
}
\end{table}

\fi

\ifICLR\else
\begin{figure}[!t]
\centering

\begin{subfigure}{0.245\textwidth}
    \centering
    \includegraphics[width=\linewidth]{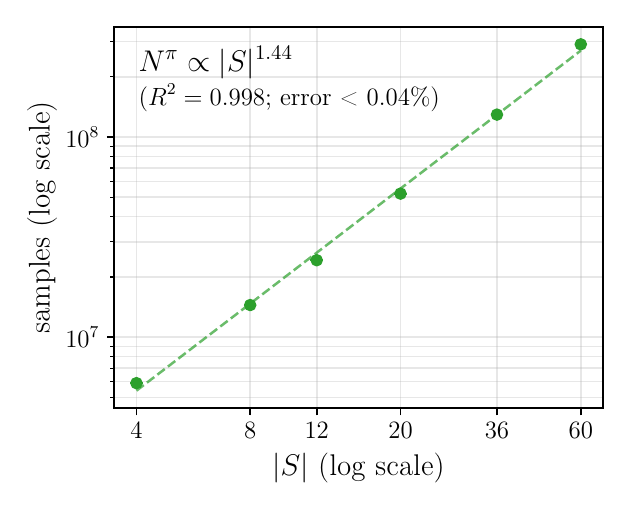}
    \caption{$|S|$-scaling}
    \label{fig:S-scaling}
\end{subfigure}
\hfill
\begin{subfigure}{0.245\textwidth}
    \centering
    \includegraphics[width=\linewidth]{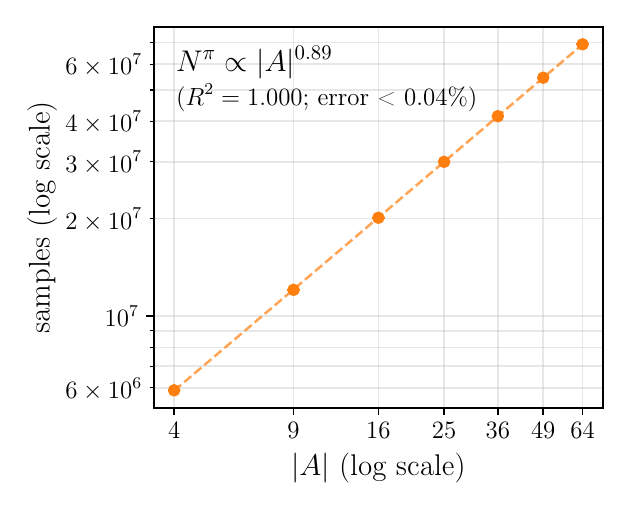}
    \caption{$|A|$-scaling}
    \label{fig:A-scaling}
\end{subfigure}
\hfill
\begin{subfigure}{0.245\textwidth}
    \centering
    \includegraphics[width=\linewidth]{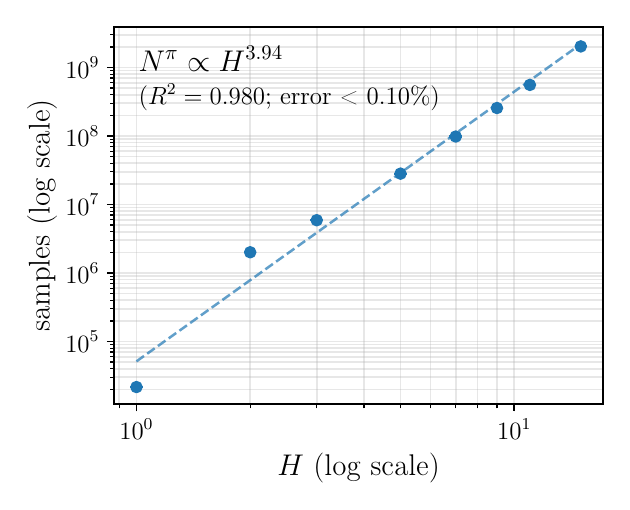}
    \caption{$H$-scaling}
    \label{fig:H-scaling}
\end{subfigure}
\hfill
\begin{subfigure}{0.245\textwidth}
    \centering
    \includegraphics[width=\linewidth]{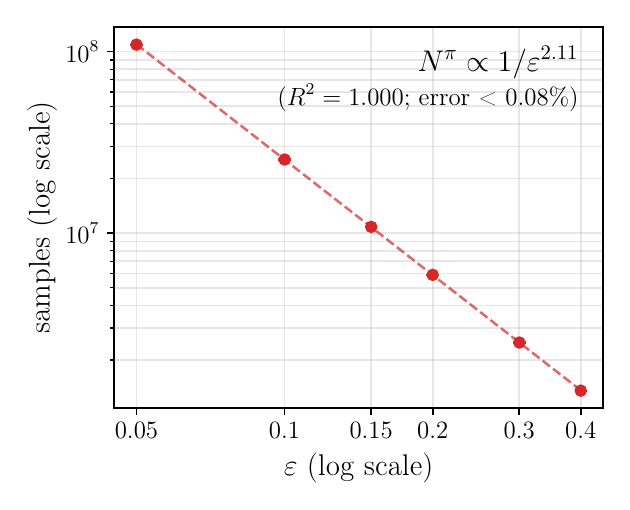}
    \caption{$\varepsilon$-scaling}
    \label{fig:eps-scaling}
\end{subfigure}

\caption{
Scaling of sample complexity (log--log scale) in the \textit{Safe vs.\ Risky} game, following the setup in \hyperref[RQ3-sample-complexity]{RQ3}, with respect to (a) $|S|$, (b) $|A|$, (c) $H$, and (d) $\varepsilon$. Error bars show standard deviation over 10 runs; all are plotted but negligible ($<0.10\%$ of the mean) and thus not visible.
}
\label{fig:scaling-plots}
\vspace*{-1.5em}
\end{figure}

\fi

% \subsection{Discussion}

\paragraph{RQ1(a) Near-optimality.}
For all case studies in \Cref{tab:correctness} that admit an exact NE (i.e., $\opt{\mu} \ge 0$), the learned profile $\hat{\sigma}$ attains the same true-game value as the oracle up to ties between indifferent actions. For the finite-valued cases, this yields a value gap $V^\star - u(\hat{\sigma}, P^\star)=0 \ll \varepsilon$.
% ; for the second \emph{Safe vs.\ Risky} property, both the oracle and our learner correctly return the unbounded value $\infty$. 
The robust value estimate $\hat{V}$ can, however, differ slightly from $V^\star$. Nevertheless, the estimation gap $V^\star-\hat{V}$ consistently remains below $4\Delta_t\le \varepsilon$, as guaranteed by \Cref{lem:anytime-bound-Vstar} (\Cref{app:anytime-proofs}). 
This empirically validates the stopping condition $\Delta_t \le \varepsilon/4$, which ensures near-optimality of the output profile via \Cref{cor:anytime-value-gap}: once $\Delta_t$ falls below this threshold, the robust solution computed on $\Punc_t$ transfers to the true model with error at most $4\Delta_t \le \varepsilon$.
% The one exception is the \emph{Delayed Coordination} benchmark (value gap $0.2001> \varepsilon$): this game has the lowest $\reach{p}$ in the suite, requiring significantly more samples per episode ($N^\omega_t \propto 1/\reach{p}$ as we derive in \Cref{sec:exploration}), and the reported gap reflects early termination at the episode budget rather than a failure of the guarantee. Extending the budget closes the gap, consistent with the theory.
% \ifICLR
% The second \emph{Safe vs.\ Risky} property is a degenerate case with unbounded SWNE value; the algorithm correctly returns $V^\star = \infty$ with an attaining profile, showing that our implementation handles unbounded total reward despite the analysis assuming bounded per-step rewards (see \Cref{app:unbounded-value}).
% \else
Further, the second \emph{Safe vs.\ Risky} property exhibits a degenerate case in which the SWNE value is unbounded: since both $s_1$ and $s_2$ are absorbing self-loops, once either is reached the game loops indefinitely, allowing the corresponding player to accumulate unbounded reward (while the reachability objective is effectively zero-sum); see \Cref{fig:benchmarks}, \Cref{app:additional-experimental-results}. The algorithm correctly returns $V^\star = \infty$ together with a profile that attains it, demonstrating that our implementation handles unbounded total reward despite assuming bounded per-step rewards in the analysis.
% \fi

\ifICLR
\else
\paragraph{RQ1(b) Equilibrium existence detection.}
Two benchmarks specifically consider the NE existence problem.
\begin{enumerate*}[1)]
\item \emph{Cyclic Preferences} provides a \emph{clean non-existence} case. The game does not exhibit a fixed-point solution and possesses no exact NE.
% (the CSG-solver raises an error reporting failure to converge after $10^5$ iterations). 
Our solver correctly returns $\notFound$ in 10/10 runs, yielding a sound certificate that no exact stationary NE exists. This is consistent with \Cref{prop:solver-EQ-guarantees} which stipulates that solver failure implies $\mu^\star < 4\Delta_t - \varepsilon$, which is $\le 0$ at termination; and thus $\opt{\mu}\le 0$ by \Cref{prop:opt-mu-to-SW-mu}, \Cref{app:prelims}.
\item \emph{Hide-or-Run} illustrates \emph{incompleteness} of the non-existence certificate. The game has no optimal stationary profile, but only an NE value in the limit (i.e., approached by a sequence of mixed strategies). In all runs, our finite-sample learner returns an $\varepsilon$-NE approaching the limiting equilibrium (hide w.p.\ $0.999$ at $s_0$) rather than $\notFound$. This does not violate soundness, however: the non-existence certificate is required to be correct when emitted, but need not be emitted for every game lacking an exact stationary NE.
\end{enumerate*}
% This also suggests that our algorithm is most likely able to detect equilibrium \emph{value} non-existence (e.g., in \emph{Cyclic Preferences}) as opposed to equilibrium \emph{profile} non-existence (e.g., in \emph{Hide-or-Run}).

\paragraph{RQ2 Exploration strategy.} 
\Cref{tab:explorer-cmp} compares all explorers under the same PAC loop and stopping criterion. Uniform exploration is consistently less sample-efficient, requiring $2.0\times$ more trajectories on \textit{Safe vs. Risky}, $2.5\times$ on \textit{Cyclic Preferences}, and $17\times$ on \textit{Delayed Coordination}.
Round-robin performs similarly to the robust explorer on the shallow benchmarks but requires $2.3\times$ more trajectories on \textit{Traffic Merge}. Optimistic exploration has similar sample consumption to robust exploration on most benchmarks, but uses slightly more samples on \emph{Delayed Coordination}, where it also gives less stable and less accurate robust-value estimates: its estimated value differs from the true value by $0.06$, whereas our robust explorer gives a near-perfect estimate. These results suggest that targeted exploration matters most when useful state-action pairs are difficult to reach, while pessimistic or robust planning provides more reliable coverage under transition uncertainty.
We further analyse the learning dynamics of our robust explorer in \Cref{app:RQ-learning-dynamics}, including the evolution of uncertainty, known-slot coverage, confidence-radius imbalance, and per-episode sample consumption.
\fi

\paragraph{RQ3 Sample complexity scaling.}
\Cref{fig:scaling-plots} 
\ifICLR
(App. \ref{app:scaling-plots})
\fi
shows the dependence of total sample consumption on different problem parameters. Log--log fits give empirical exponents of $1.44$ for $|S|$, $0.89$ for $|A|$,
% (excluding the degenerate $|A|=1$ point, which has no actual choice of action and lies off the fitted trend), 
$3.94$ for $H$ and $2.11$ for $1/\varepsilon$. These are consistent with the worst-case bound of $|S|^2 |A|H^4 / \varepsilon^2$ in \Cref{thm:PAC-CSG}. In particular, the horizon exponent closely tracks the theoretical dependence over the tested range, while state and action space growth is milder than the upper bound. These experiments do not establish tight asymptotic rates, but show that the observed scaling is consistent with the theorem over the tested instances.
%
% Due to computational and time constraints, the range of horizons explored is limited; see \Cref{app:intrinsic-limitations} for further discussion.

% Our experiments confirm three key findings:
% \begin{itemize}
%     \item \textbf{Robustness is essential:} point-estimate equilibria can degrade significantly under model error, while robust equilibria remain stable.
%     \item \textbf{Exploration is the bottleneck:} multi-agent coordination makes state--action coverage fundamentally harder.
%     \item \textbf{Delayed rewards amplify difficulty:} benchmarks with sparse and delayed feedback require substantially more samples.
% \end{itemize}
% These observations are consistent with the theory developed in \Cref{sec:theory}.

\section{Conclusion}
\label{sec:conclusion}

\ifICLR
We introduced \textsc{PAC-CSG}, the first PAC learning framework for general-sum concurrent stochastic games under transition uncertainty that handles stationary NE (non-)existence. The algorithm combines $L^1$ uncertainty sets with a new RMDP formulation that casts joint state--action coverage as a robust reachability problem, and uses a black-box RCSG solver to compute $\varepsilon$-RSWNEs. Under a graph reachability condition, it terminates after $\bigOtilde{\Rmax^2 H^4 |S|^2 |A| / (\varepsilon^2 \reach{p})}$ trajectory samples and outputs either an $\varepsilon$-social-welfare-optimal $\varepsilon$-NE or a sound certificate of NE non-existence.
Beyond these guarantees, our results highlight two key insights:
\begin{enumerate*}[1)]
\item robust equilibrium computation in the empirical model transfers to the true game; and
\item equilibrium non-existence can be detected in a sound, data-driven manner.
\end{enumerate*}
Empirically, our framework achieves near-optimality, sound non-existence detection, and sample complexity consistent with theory.
Future work includes relaxing the standard assumptions of centralised exploration, known transition support, and reachability, plus improving scalability in $|S|$ and $|A|$, e.g., through value-aware exploration; see \Cref{app:limitations} for details.
\else
We introduced \textsc{PAC-CSG}, the first PAC learning framework for general-sum concurrent stochastic games under transition uncertainty that handles stationary NE (non-)existence. The algorithm combines $L^1$ uncertainty sets with a new RMDP formulation that casts joint state--action coverage as a robust reachability problem, and uses a black-box RCSG solver to compute $\varepsilon$-RSWNEs. Under a graph reachability condition, it terminates after $\bigOtilde{\Rmax^2 H^4 |S|^2 |A| / (\varepsilon^2 \reach{p})}$ trajectory samples and outputs either an $\varepsilon$-social-welfare-optimal $\varepsilon$-NE or a sound certificate of NE non-existence.
Beyond these guarantees, our results highlight two key insights:
\begin{enumerate*}[1)]
\item robust equilibrium computation in the empirical model transfers to the true game; and
\item equilibrium non-existence can be detected in a sound, data-driven manner.
\end{enumerate*}
Empirically, our framework achieves near-optimality, sound non-existence detection, and sample complexity consistent with theory.
Future work includes relaxing the standard assumptions of centralised exploration, known transition support, and reachability, as well as improving scalability in $|S|$ and $|A|$, e.g., through value-aware exploration; see \Cref{app:limitations} for further discussion.
\fi

\begin{ack}
This work is supported by the EPSRC Centre for Doctoral Training no. EP/Y035070/1 and the UKRI AI Hub on Mathematical Foundations of AI.
The authors declare no competing interests.
\end{ack}

% \section*{References}

\bibliography{references}

%%%%%%%%%%%%%%%%%%%%%%%%%%%%%%%%%%%%%%%%%%%%%%%%%%%%%%%%%%%%

\appendix

\crefalias{section}{appendix}
\crefalias{subsection}{subappendix}
\crefalias{subsubsection}{subsubappendix}

\section{Proof for Nash margin property (Definition~\ref{def:Nash-margin})}
\label{app:prelims}
\begin{proposition}
\label{prop:opt-mu-to-SW-mu}
    $\opt{\mu} \ge 0$ iff $\mu^\star \ge 0$. 
\end{proposition}
\begin{proof}
We prove both directions.
\begin{itemize}
    \item[$(\Rightarrow)$] Suppose $\opt{\mu} \ge 0$. By definition of $\opt{\mu}$, there exists a profile $\sigma$ with Nash margin at least $0$, i.e., an NE profile. Therefore, the set of NE profiles is non-empty. 
    By definition, an SWNE is an NE that maximises social welfare among all NE profiles. Since the set of NE profiles is non-empty, such a profile exists. As every SWNE is an NE, it must also have non-negative Nash margin, and hence $\mu^\star \ge 0$.

    \item[$(\Leftarrow)$] Since $\opt{\mu}$ is the maximum achievable Nash margin, by definition $\opt{\mu} \ge \mu^\star \ge 0$.
\end{itemize}
\end{proof}

\section{Concentration bounds and confidence radii}
\label{app:per-slot-proofs}
In this section, we establish the Weissman concentration result (\Cref{lem:per-slot-concentration}) and the containment event $\eventCI$ (via \Cref{cor:true-kernel-containment}) used in \Cref{sec:theory1-estim-error}.

We begin by analysing the concentration of empirical transition probabilities, which are used to construct the per-slot $L^1$ uncertainty sets $\Punc_t(s,a)$ for each episode $t \ge 1$ in \Cref{alg:PAC-learning-CSG}.

For a state--action pair $(s,a)\in \ntslots$, let $N_{sa} \ge 0$ denote the number of independent samples drawn from the true transition distribution $P^\star_{sa}(\cdot)$, and let $N_{sas'}$ denote the number of observed transitions to state $s'$. 
% Note that for the reachability objectives, where $S_T\neq\emptyset$, transitions after entering $S_T$ do not affect the objective, so we only need to explore and track visits to non-target slots. For the bounded cumulative-reward objective, $S_T=\emptyset$, so $\ntslots=\slots$ and all slots are relevant.

Further let
\[
\Succ(s,a):= \{s'\mid (s,a,s')\in \supp(P^\star)\},
\]
and define the support-preserving simplex
\[
\Delta_+(\Succ(s,a)):= \left\{p\in \Delta(\Succ(s,a)): \supp(p)=\Succ(s,a) \right\}.
\]
Equivalently, $\Delta_+(\Succ(s,a))$ contains exactly those transition distributions whose support equals the known successor support of that slot.
In the theoretical analysis, we use this support-preserving class; in implementation, we replace it with the corresponding closed simplex by imposing a $10^{-5}$ lower bound on probabilities within the support. This ensures graph preservation.

If $N_{sa} = 0$, we define the empirical uncertainty set $\Punc_t(s,a)=\Delta_+(\Succ(s,a))$ and set the point estimate $\widehat{P}_{sa}$ to be the uniform distribution over this support.

If $N_{sa} > 0$, we define the empirical estimate as $\widehat{P}_{sa}(s') = N_{sas'}/N_{sa}$, and construct the corresponding $L^1$ uncertainty set as described below.

\begin{restatable}[Per-slot concentration]{lemma}{PerSlotConcentration}
\label{lem:per-slot-concentration}
Consider a given state--action pair $(s,a)$ with $n \ge 1$ i.i.d. draws from $P^\star_{sa}(\cdot)$, and fix a failure probability $\delta^{sa}\in (0,1)$. 
Define the Weissman $L^1$ confidence radius
\[
\alpha_{sa}(n; \delta^{sa}) := \sqrt{\frac{2}{n}\, \ln{\frac{2^{|S|}-2}{\delta^{sa}}}}.
\]
Then
\[
\Pr\left( \| P_{sa}^\star(\cdot) - \widehat{P}_{sa}(\cdot) \|_1 \le \alpha_{sa}(n) \right) \ge 1-\delta^{sa}.
\]
\reforig{sec:learning-setup}
% \refproof{prf:per-slot-concentration}
\end{restatable}

\begin{proof}
\label{prf:per-slot-concentration}
We apply the Weissman inequality \citep{weissman2003inequalities},
% Let $\widehat{P}_{sa}$ be the empirical distribution formed from $n$ i.i.d. draws from $P^\star_{sa}$. 
which states that for any $\varepsilon > 0$,
\[
\Pr\left(\|\widehat{P}_{sa} - P^\star_{sa}\|_1 \geq \varepsilon\right) 
\leq (2^{|S|} - 2)\exp\left(-\frac{n \varepsilon^2}{2}\right).
\]
% This follows from the identity $\|{\widehat{P}} - P\|_1 = 2\sup_{A \subseteq S}|\widehat{P}(A) - P(A)|$ combined with a union bound over all $2^{|S|} - 2$ non-trivial subsets
% $A \subsetneq S$ (excluding $\emptyset$ and $S$, for which the
% difference is always zero), and applying Hoeffding's inequality to the bounded random variable $\ind{s'\in A}$ for each subset.
%
Setting the right-hand side equal to $\delta^{sa}$ and solving for $\varepsilon$:
\[
(2^{|S|} - 2)\exp \left(-\frac{n\varepsilon^2}{2}\right) = \delta^{sa}
\ \implies \
\varepsilon = \sqrt{\frac{2}{n}\ln\frac{2^{|S|} - 2}{\delta^{sa}}} =: \alpha_{sa}(n; \delta_{sa}).
\]
\end{proof}

\begin{restatable}[Containment of $P^\star$ in $\Punc_t$]{corollary}{TrueKernelContainment}
\label{cor:true-kernel-containment}
Fix $\confCI\in(0,1)$. For brevity, write $n:=N_t(s,a)$ and define
$\confCI_{sa,n}:=\confCI/(\nTrans n(n+1))$.
For each episode $t\ge1$, define the slot-wise uncertainty set
\[
\Punc_t(s,a):=
\left\{
p\in\Delta_+(\Succ(s,a)):
\|p-\widehat{P}_{sa,n}\|_1
\le
\alpha_{sa}\!\left(n;\confCI_{sa,n}\right)
\right\},
\]
and the global uncertainty set
\[
\Punc_t:=\prod_{(s,a)\in \ntslots}\Punc_t(s,a),
\]
% \[
% \Punc_t := \left\{P \in \Delta(S)^{S \times A}
%   : \|P_{sa} - \widehat{P}_{sa} \|_1 \leq \alpha_{sa}\!\left(n; \ \confCI_{sa,n} \right) \quad \forall (s,a)\right\},
% \]
where $\alpha_{sa}\!\left(n;\confCI_{sa,n}\right)$ is the Weissman radius from
\Cref{lem:per-slot-concentration}, instantiated at
$\delta^{sa}:=\confCI_{sa,n}$.
Then, with probability at least $1-\confCI$,
\[
P^\star\in\Punc_t
\qquad\text{for all }t\ge1.
\]
% \refproof{prf:true-kernel-containment}
% \reforig{cor:true-kernel-containment}
\end{restatable}

\begin{proof}
\label{prf:true-kernel-containment}
For each $(s,a)$ and $n\ge 1$, define the event
\[
\eventCI_{sa,n} := \left\{ \|P^\star_{sa} - \widehat{P}_{sa,n}\|_1 > \alpha_{sa}\!\left(n; \confCI_{sa,n} \right).
\right\}
\]
By \Cref{lem:per-slot-concentration},
\[
\Pr(\eventCI_{sa,n}) \le \confCI_{sa,n} 
= \frac{\confCI}{\nTrans n(n+1)}.
\]
Taking a union bound over all $(s,a) \in \ntslots$ and all possible sample counts $n \geq 1$,
\[
\Pr\left( \bigcup_{(s,a)}\bigcup_{n\ge 1}{\eventCI_{sa,n}} \right)
\le \sum_{(s,a)}\sum_{n\ge 1}{\Pr(\eventCI_{sa,n})}
= \confCI\sum_{n=1}^\infty \frac{1}{n(n+1)} = \confCI,
\]
% using the telescoping identity $\sum_{n=1}^\infty \frac{1}{n(n+1)} = 1$.
Hence with probability at least $1-\confCI$, no such failure occurs for any $(s,a)$ and any $n$, which implies $P^\star \in \Punc_t$ for all episodes $t\ge1$.
\end{proof}

\paragraph{Conditioning on containment of the true kernel.}
Define the per-episode and all-episode containment event:
\[
\eventCI_t := \{ P^\star \in \Punc_t \}, 
\qquad 
\eventCI := \bigwedge_{t\ge 1}{\eventCI_t}.
\]
Following \Cref{cor:true-kernel-containment}, we use the summable failure schedule $\confCI_n=\confCI/(n(n+1))$ to ensure that 
\[
\Pr(\eventCI) \ge 1-\confCI.
\]
In the remainder of the analysis we condition on $\eventCI$. 
% All probabilities and expectations are therefore w.r.t. the true system given that $\eventCI$ holds.
% (When we state final bounds ``with probability at least $1-\delta$'' we mean unconditioning the above event together with the remaining failure events appearing in the proof.)

\section{Proofs for value sensitivity (Lemma~\ref{lem:sensitivity})}
\label{app:per-episode-proofs}
We continue to derive the results in \Cref{sec:theory1-estim-error} of the main paper, focusing on the sensitivity lemma (\Cref{lem:sensitivity}) and presenting several useful corollaries.
In the following, consider a fixed episode $t\ge 1$ and a given strategy profile $\sigma\in \Sigma$.
% , and assume bounded per-step reward $|r_i|\le \Rmax$ for all player $i\in N$. 
% Let $\widehat{P} \in \Punc_t$ be a transition kernel on the same finite CSG as $P^\star$. 
% Define
% \[
% \diffLI_t :=\max_{(s,a)\in \ntslots} \|P^\star_{sa}-\widehat{P}_{sa}\|_1 \le \max_{s,a}{\alpha_t(s,a)}.
% \]

\begin{proposition}[One-step error propagation]
\label{prop:tv-propagation}
Let $\mu_h, \widehat{\mu}_h$ denote the stopped state marginals at time-step $h$ under $P$ and $\widehat{P}$ respectively, following $\sigma$, started from the same initial state $s_0\in S$.
Define
\[
\diffLI_t := \max_{(s,a)\in\ntslots}{\alpha_t(s,a)}.
\]
Then for all $h\ge1$,
\[
\|\mu_h-\widehat{\mu}_h\|_1 \le h \diffLI_t.
\]
\end{proposition}

\begin{proof}
We prove the recursive bound
\begin{equation}
\label{eq:one-step-tv-propagation}
\|\mu_{h+1}-\widehat{\mu}_{h+1}\|_1
\le \|\mu_h-\widehat{\mu}_h\|_1 + \diffLI_t,
\end{equation}
from which the result follows by induction using
$\mu_0=\widehat{\mu}_0=\delta_{s_0}$, i.e., the Dirac distribution that assigns probability $1$ to state $s_0$ and $0$ to all others.

For $s\not\in S_T$, define the mixed one-step kernels
\[
K_s := \sum_a \sigma(a|s) P(\cdot|s,a),
\qquad
\widehat{K}_s := \sum_a \sigma(a|s) \widehat{P}(\cdot|s,a).
\]
Then
\[
\mu_{h+1} = \sum_{s\not\in S_T} \mu_h(s) K_s,
\qquad
\widehat{\mu}_{h+1} = \sum_{s\not\in S_T}\widehat{\mu}_h(s)\widehat{K}_s,
\]
and hence
\[
\mu_{h+1}-\widehat{\mu}_{h+1}
= \sum_{s\not\in S_T} (\mu_h(s)-\widehat{\mu}_h(s))K_s
  + \sum_{s\not\in S_T} \widehat{\mu}_h(s)(K_s-\widehat{K}_s).
\]

Taking $L^1$ norms and applying the triangle inequality:
\begin{equation}
\label{eq:mu-L1-upper-bound}
\|\mu_{h+1}-\widehat{\mu}_{h+1}\|_1
\le \Big\|\sum_{s\not\in S_T} (\mu_h(s)-\widehat{\mu}_h(s))K_s\Big\|_1
   + \Big\|\sum_{s\not\in S_T} \widehat{\mu}_h(s)(K_s-\widehat{K}_s)\Big\|_1.
\end{equation}

\paragraph{First term.}
Let $c_s := \mu_h(s)-\widehat{\mu}_h(s)$. Since each $K_s$ is a probability distribution,
\[
\sum_{s'} \left|\sum_{s\not\in S_T} c_s K_s(s')\right|
\le \sum_{s\not\in S_T} |c_s|\sum_{s'} K_s(s')
= \sum_{s\not\in S_T} |c_s|,
\]
hence
\[
\Big\|\sum_{s\not\in S_T} (\mu_h(s)-\widehat{\mu}_h(s))K_s\Big\|_1
\le \|\mu_h-\widehat{\mu}_h\|_1.
\]

\paragraph{Second term.}
By convexity,
\[
\left\|\sum_{s\not\in S_T} \widehat{\mu}_h(s)(K_s-\widehat{K}_s)\right\|_1
\le \sum_{s\not\in S_T} \widehat{\mu}_h(s)\|K_s-\widehat{K}_s\|_1
\le \max_{s\not\in S_T} \|K_s-\widehat{K}_s\|_1.
\]
For each $s\not\in S_T$,
\begin{align*}
\|K_s-\widehat{K}_s\|_1
&\le \sum_a \sigma(a|s)\|P_{sa}-\widehat{P}_{sa}\|_1 
\le \max_a \|P_{sa}-\widehat{P}_{sa}\|_1 \\
&= \max_a \sum_{s'} |P_{sa}(s')-\widehat{P}_{sa}(s')|
\le \max_a \alpha_t(s,a) = \diffLI_t.
\end{align*}
where the last inequality follows from the construction of $\Punc_t$ (see \Cref{cor:true-kernel-containment}).

Substituting into \eqref{eq:mu-L1-upper-bound} gives
\[
\|\mu_{h+1}-\widehat{\mu}_{h+1}\|_1
\le \|\mu_h-\widehat{\mu}_h\|_1 + \diffLI_t.
\]
Unrolling the recursion yields
\[
\|\mu_h-\widehat{\mu}_h\|_1 \le h\,\diffLI_t.
\]
\end{proof}

%-----------------------------------------------------
% playerwise version
%-----------------------------------------------------
\LinearSensitivity*
\begin{proof}
\label{prf:linear-sensitivity}
The left-hand side could be rewritten as:
\begin{align*}
\lvert u_i(\sigma,P\mid s_0)-u_i(\sigma,\widehat{P}\mid s_0) \rvert 
&= \left| 
    \ev^{\sigma,P}\!\left[\sum_{h=0}^{H-1} r_i(s_h,a_h)\right]
    - \ev^{\sigma,\widehat{P}}\!\left[\sum_{h=0}^{H-1} r_i(s_h,a_h)\right]
     \right| \\
&= \left|
   \sum_{h=0}^{H-1}\!\left[
     \ev^{\sigma,P}[r_i(s_h,a_h)]
     - \ev^{\sigma,\widehat{P}}[r_i(s_h,a_h)]
   \right]
 \right|.
\end{align*}
Note that for bounded probabilistic reachability, we use the equivalent finite-horizon terminal-reward encoding, so the same argument applies. 
Moreover,
\begin{align*}
  \ev^{\sigma,P}[r_i(s_h,a_h)]
  &= \sum_s\sum_a \prob^{\sigma,P}(s_h=s)\sigma(a\mid s)\cdot r_i(s,a) \\
   &= \sum_s \mu_h(s)\!\sum_a \sigma(a\mid s)r_i(s,a)
   = \sum_s \mu_h(s)r_i^\sigma(s),
\end{align*}
where $r_i^\sigma(s):=\ev_{a\sim\sigma(s)}[r_i(s,a)]$.
By assumption, $|r_i^\sigma(s)|\le\Rmax$. Thus,
\begin{align*}
  LHS
  &= \left|\sum_{h=0}^{H-1}\sum_s r_i^\sigma(s)
           \bigl(\mu_h(s)-\widehat{\mu}_h(s)\bigr)\right| \\
  &\le \sum_{h=0}^{H-1}\max_s|r_i^\sigma(s)|\cdot
       \sum_s\left|\mu_h(s)-\widehat{\mu}_h(s)\right|
   \le \Rmax\sum_{h=0}^{H-1}\|\mu_h-\widehat{\mu}_h\|_1 \\
  &\le \Rmax\diffLI_t\sum_{h=0}^{H-1} h
   = \Rmax\diffLI_t \frac{H(H-1)}{2}
  &&\text{(by \Cref{prop:tv-propagation})}\\
  &\le \frac12 H^2 \Rmax\diffLI_t 
  =: \Delta_t.
\end{align*}
\end{proof}

Henceforth, we define 
\[
\Delta_t:=\diffStatExpr,
\quad \text{where} \quad 
\diffLI_t := \max_{(s,a)\in\ntslots}\|P^\star_{sa}-\widehat{P}_{sa}\|_1.
\]
% as in \Cref{lem:sensitivity}.

%------------------------------------------------------------------
% player-wise bound
%------------------------------------------------------------------
\begin{corollary}
\label{cor:sensitivity-P-Pstar}
% Suppose $P^\star\in\Punc_t$. 
For any $P, P'\in\Punc_t$, any strategy profile $\sigma$, and each player $i\in N$:
\[
\left|u_i(\sigma,P)-u_i(\sigma,P')\right| \le 2\Delta_t.
\]
% \reforig{lem:sensitivity}
\end{corollary}

\begin{proof}
By the triangle inequality and \Cref{lem:sensitivity} (applied twice),
% , since $P,P'\in\Punc_t$:
\[
\left|u_i(\sigma,P)-u_i(\sigma,P')\right|
\le \left|u_i(\sigma,P)-u_i(\sigma,\widehat{P})\right|
 + \left|u_i(\sigma,\widehat{P})-u_i(\sigma,P')\right|
= 2\Delta_t.
\]
\end{proof}

%------------------------------------------------------------------
% social-welfare versions
%------------------------------------------------------------------
\begin{corollary}[Social-welfare sensitivity]\label{cor:sensitivity-joint}
For any $P\in\Punc_t$ and any strategy profile $\sigma$:
\[
\left|u(\sigma,P)-u(\sigma,\widehat{P})\right| \le 2\Delta_t.
\]
Consequently, for any $P, P' \in\Punc_t$:
\[
\left|u(\sigma,P)-u(\sigma,P')\right| \le 4\Delta_t.
\]
\end{corollary}

\begin{proof}
Apply \Cref{lem:sensitivity} to each player $i\in N$ and sum up:
\[
\left|u(\sigma,P)-u(\sigma,\widehat{P})\right|
\le \sum_{i\in N}\left|u_i(\sigma,P)-u_i(\sigma,\widehat{P})\right|
\le \sum_{i\in N}{\Delta_t}
= 2\Delta_t.
\]
The pairwise bound follows identically from \Cref{cor:sensitivity-P-Pstar}.
\end{proof}

% \PerEpisodeError*
% \begin{proof}
% \label{prf:per-episode-error}
% Since $\widehat{P}_t \in\Punc_t$, by definition we have $\rob{u}_{\Punc_t}(\hat{\sigma})\le u(\hat{\sigma},\widehat{P}_t)$.
% %
% Then, applying \Cref{cor:sensitivity-joint} with $P \leftarrow P^\star$ gives
% \[
% \left| u(\hat{\sigma},P^\star) - u(\hat{\sigma},\widehat{P}_t) \right|
% \le 2\Delta_t.
% \]
% Rearranging and combining the inequalities gives
% \begin{align*}
% u(\hat{\sigma},P^\star) 
% \ge u(\hat{\sigma},\widehat{P}_t) - 2\Delta_t
% \ge \rob{u}_{\Punc_t}(\hat{\sigma}) - 2\Delta_t.
% \end{align*}
% \end{proof}

% These per-episode error bounds feed directly into the analyses for equilibrium existence (\Cref{app:EQ-existence-proofs}) and anytime guarantees (\Cref{app:anytime-proofs}).

\section{Proofs for equilibrium existence (Proposition~\ref{prop:solver-EQ-guarantees})}
\label{app:EQ-existence-proofs}

% \begin{proposition}
% % [Robust/Optimistic value bounds true value]
% \label{prop:robu-optu}
% For a given uncertainty set $\Punc$, if $P^\star\in\Punc$ then for every strategy profile $\sigma$,
% \begin{align*}
% \rob{u}_{\Punc}(\sigma) &:= \inf_{P\in\Punc} u(\sigma,P) \le u(\sigma,P^\star)\\
% \opt{u}_{\Punc}(\sigma) &:= \sup_{P\in\Punc} u(\sigma,P) \ge u(\sigma,P^\star)
% \end{align*}
% \end{proposition}
% \begin{proof}
% Immediate by definition of $\rob{u}_{\Punc}(\sigma)$.
% \end{proof}

This section establishes the key properties of the solver oracle (\Cref{prop:solver-EQ-guarantees}) that link robust equilibrium computation in the empirical model to equilibrium existence in the true game. We first formalise the solver oracle (\Cref{ass:solver}), then show that any exact NE of the true game induces an approximate RNE in the empirical model (\Cref{lem:NE-RNE-transfer}), and use this to derive the solver guarantees (\Cref{prop:solver-EQ-guarantees}).

Henceforth, we assume access to a black-box ($L^1$-)RCSG solver $\SolveRCSG$ with tolerance $\varepsilon \ge 0$, which computes an $\varepsilon$-RSWNE of the empirical RCSG $\csg_t$.
Below we use the notation of \Cref{alg:PAC-learning-CSG}.

\begin{assumption}[$\SolveRCSG$ oracle]
\label{ass:solver}
Consider an ($L^1$-)RCSG $\csg$ with uncertainty set $\Punc$. 
Given tolerance $\varepsilon \ge 0$, the solver returns $(\hat{\sigma}, \hat{V}, \found) = \SolveRCSG(\csg, \varepsilon)$ such that 
$\found = \true$ iff $\hat{\sigma}$ is an $\varepsilon$-RSWNE of $\csg$, in which case $\hat{V} = \rob{u}_{\Punc}(\hat{\sigma})$. 
Equivalently, $\found = \false$ iff $\epsRNEop{\varepsilon}(\Punc) = \emptyset$.
\end{assumption}
Note that $\hat{\sigma} \in \epsRNEop{\varepsilon}(\Punc)$ implies $\epsRNEop{\varepsilon}(\Punc) \neq \emptyset$, so the failure condition is equivalently stated as $\epsRNEop{\varepsilon}(\Punc) = \emptyset$.

\Cref{ass:solver} abstracts equilibrium computation as an oracle, yielding an information-theoretic PAC guarantee that is decoupled from the solver’s computational complexity. Practical solvers for general-sum $L^1$-RCSGs under our objectives do exist (e.g., in PRISM-games~\citep{HP26}), although polynomial-time implementations are currently known only in certain settings (e.g., the zero-sum case \citep{HP26}). We defer a detailed discussion of computational aspects to \Cref{app:limitations}.

At each episode $t$, the solver either succeeds or fails to find an $\varepsilon$-RSWNE of $\Punc_t$. The following result characterises how each outcome relates to equilibrium properties of the true game $\csg^\star$.

\NERNETransfer*
\begin{proof}
\label{prf:NE-RNE-transfer}
Let $\sigma$ be an NE under $P^\star$ with Nash margin $\mu \ge 0$. By the triangle inequality, for any $P \in \Punc_t$ and player $i \in N$, we have:
\begin{align*}
u_i(\sigma_{-i}[\sigma_i'],P) - u_i(\sigma,P)
&\le \left[ u_i(\sigma_{-i}[\sigma_i'],P) - u_i(\sigma_{-i}[\sigma_i'],P^\star) \right] 
+ \left[ u_i(\sigma_{-i}[\sigma_i'],P^\star) - u_i(\sigma,P^\star) \right] \\
&\qquad + \left[ u_i(\sigma,P^\star) - u_i(\sigma,P) \right]\\
&\le 2\Delta_t -\mu + 2\Delta_t
= 4\Delta_t -\mu
\end{align*}
where the last inequality follows from \Cref{cor:sensitivity-P-Pstar} and the definition of $\mu$ (\Cref{def:Nash-margin}).
Since the bound holds for every $P\in\Punc_t$, taking $\inf$ preserves the inequality, giving
\[
\inf_{P\in\Punc_t} \left[ u_i(\sigma_{-i}[\sigma_i'],P) - u_i(\sigma,P) \right] \le 4\Delta_t -\mu
\qquad \forall{i \in N}.
\]
Hence $\sigma$ is an $(4\Delta_t -\mu)$-RNE under $\Punc_t$. 
\end{proof}

\EpisodewiseExistence*
\begin{proof}
\label{prf:episodewise-existence}
We prove each part individually.

\begin{enumerate}

\item Conditioning on $\eventCI$ we have $P^\star \in \Punc_t$. By definition of an $\varepsilon$-RSWNE (see after \Cref{def:subgame-perfect-rne}), every $\varepsilon$-RSWNE under $\Punc_t$ is also an $\varepsilon$-RNE under $\Punc_t$, and every $\varepsilon$-RNE under $\Punc_t$ is an $\varepsilon$-NE of every $P\in \Punc_t$ including $P^\star$. Therefore, if the solver returns an $\varepsilon$-RSWNE $\hat{\sigma}$ under $\Punc_t$, then $\hat{\sigma}$ is an $\varepsilon$-NE of $P^\star$ (the transition kernel of $\csg^\star$).

\item 
% Immediately following \Cref{lem:NE-RNE-transfer} with $\mu \ge \nashMarginThresh$. 
Suppose the solver returns $\notFound$. By \Cref{ass:solver}, this means that $\epsRNEop{\varepsilon}(\Punc_t)=\emptyset$, i.e., no $\varepsilon$-RNE exists under $\Punc_t$. 
We now distinguish two cases according to the sign of $\opt{\mu} := \max_{\sigma\in\Sigma}{\mu(\sigma,P^\star)}$ (defined in \Cref{def:Nash-margin}). 
\begin{enumerate}[(a)]
    \item If $\opt{\mu} < 0$, then no exact NE, and hence no SWNE (i.e., the NE that achieves the highest social welfare) exists under $P^\star$. However, the social-welfare optimal Nash margin $\mu^\star := \mu(\sigma^\star,P^\star)$ (\Cref{def:Nash-margin}) is well-defined only when the SWNE $\sigma^\star$ exists, therefore trivially $\mu^\star$ is undefined. 
    \item If $\opt{\mu} \ge 0$, then $\NE(P^\star) \ne \emptyset$, so the SWNE $\sigma^\star$ exists and $\mu^\star = \mu(\sigma^\star,P^\star)$ is well-defined.
    Now suppose, for contradiction, that $\mu^\star \ge 4\Delta_t - \varepsilon$. 
    Then \Cref{lem:NE-RNE-transfer} implies that $\sigma^\star$ is also a $(4\Delta_t - \mu^\star)$-RNE under $\Punc_t$. Since 
    \[
    4\Delta_t - \mu^\star 
    \le 4\Delta_t - (4\Delta_t - \varepsilon) = \varepsilon,
    \]
    $\sigma^\star$ is also an $\varepsilon$-RNE under $\Punc_t$. Hence $\epsRNEop{\varepsilon}(\Punc_t) \ne \emptyset$, and \Cref{ass:solver} guarantees that the solver would return $\Found$ with an $\varepsilon$-RSWNE under $\Punc_t$, contradicting the assumption that the solver returned $\notFound$.
\end{enumerate}
% Subsequently, since $\sigma\in \epsRNEop{\varepsilon}(\Punc_t) \ne \emptyset$, \Cref{ass:solver} guarantees that the solver would succeed and find an $\varepsilon$-RSWNE under $\Punc_t$.

\end{enumerate}
\end{proof}

\section{Proofs for anytime guarantees (Corollary~\ref{cor:anytime-value-gap})}
\label{app:anytime-proofs}
We now derive \emph{anytime} guarantees on the robust estimation error $V^\star - \rob{u}_{\Punc_t}(\hat{\sigma})$ (\Cref{lem:anytime-bound-Vstar}) and the value gap $V^\star-u(\hat{\sigma},P^\star)$ (\Cref{cor:anytime-value-gap}, \Cref{sec:theory1-estim-error}).
In this section we assume $\opt{\mu} \ge 0$ so that $\NE(P^\star) \ne \emptyset$ and hence $V^\star$ is well-defined. 
Also recall that $\hat{\sigma}$ is an $\varepsilon$-RSWNE under $\Punc_t$ by \Cref{ass:solver}.
% We continue to condition on the high-probability event $\eventCI$. 
% We begin by bounding the robust estimation error $V^\star - \rob{u}_{\Punc_t}(\hat{\sigma}_t)$ in the following result.

\begin{restatable}[Anytime upper bound for $V^\star$]{lemma}{AnytimeUpperBoundForVstar}
\label{lem:anytime-bound-Vstar}
% Suppose $\NE(P^\star) \ne \emptyset$, so that $V^\star$ is well-defined. 
If $\mu^\star \ge \nashMarginThresh$, then 
$V^\star - \rob{u}_{\Punc_t}(\hat{\sigma}_t) \le 4\Delta_t$.
% \reforig{lem:anytime-bound-Vstar}
% \refproof{prf:anytime-bound-Vstar}
\end{restatable}
% two applications of \Cref{cor:sensitivity-P-Pstar} are required (one to lower-bound the robust minimax value, one to transfer back to $P^\star$)

\begin{proof}
\label{prf:anytime-bound-Vstar}
Since $\mu^\star \ge \nashMarginThresh$, \Cref{lem:NE-RNE-transfer} implies that $\sigma^\star \in \epsRNEop{\varepsilon}(\Punc_t) \ne\emptyset$.
By \Cref{cor:sensitivity-joint} applied to $\sigma^\star$, for every $P\in\Punc_t$:
\[
u(\sigma^\star, P)
\ge u(\sigma^\star, P^\star) - 4\Delta_t
= V^\star - 4\Delta_t.
\]
Taking $\inf$ over $P\in\Punc_t$:
\[
\inf_{P\in\Punc_t}{u(\sigma^\star, P)} = \rob{u}_{\Punc_t}(\sigma^\star) \ge V^\star - 4\Delta_t.
\]
Since $\hat{\sigma}_t$ maximises the robust value over $\epsRNEop{\varepsilon}(\Punc_t)$ and $\sigma^\star \in \epsRNEop{\varepsilon}(\Punc_t)$:
\[
\rob{u}_{\Punc_t}(\hat{\sigma}_t) 
= \sup_{\sigma\in \epsRNEop{\varepsilon}(\Punc_t)}{\rob{u}_{\Punc_t}(\sigma)}
\ge \rob{u}_{\Punc_t}(\sigma^\star)
\ge V^\star - 4\Delta_t
\quad \implies \quad
V^\star - \rob{u}_{\Punc_t}(\hat{\sigma}_t) \le 4\Delta_t.
\]
\end{proof}

\PerEpisodeGap*
\begin{proof}
\label{prf:anytime-value-gap}
Immediate from \Cref{lem:anytime-bound-Vstar} since $u(\hat{\sigma}_t,P^\star) \ge \rob{u}_{\Punc_t}(\hat{\sigma}_t)$ conditioning on $\eventCI$.
% \[
% V^\star-u(\hat{\sigma}_t,P^\star) 
% \le V^\star-\rob{u}_{\Punc_t}(\hat{\sigma}_t)
% \le 4\Delta_t.
% \]
\end{proof}

\section{Supplementary details and proofs for exploration principles (\Cref{sec:exploration})}
\label{app:exploration}
This section provides the technical results underlying \Cref{sec:exploration}. We first show that the confidence allocation $\confCount_t$ sums to $\confCount$ (\Cref{prop:coverage-confidence-allocation}), then bound the number of trajectories per episode required to discover an unknown slot with probability at least $1-\confCount_t$ (\Cref{prop:num-samples-per-episode}), and finally establish the exploration probability lower bound (\Cref{lem:exploration-prob}).

Recall the following notation for the exploration analysis. A non-target slot $(s,a)$ is \emph{known} if it has been visited at least $n_{\min}$ times (defined in \Cref{lem:nmin}, \Cref{app:PAC-CSG-proofs}), and \emph{unknown} otherwise. Let
\[
\unknown_t=\{(s,a)\in \ntslots: N_t(s,a) < n_{\min}\}, 
\qquad 
\known_t = \ntslots\setminus \unknown_t
\]
denote the sets of unknown and known slots, respectively, at the beginning of episode $t$. Here, $N_t(s,a)$ is the visit count of $(s,a)$ by episode $t$. 
Further define $\totSlots := |\ntslots|n_{\min} \le |S||A|n_{\min}$, the total number of visits required across all relevant slots.
% We continue to define the total error as in \Cref{cor:per-episode-error}:
% \[
% \Delta_t := \tfrac{1}{2}\Rmax H^2 \max_{(s,a)} \big\|\widehat{P}_{sa}-P^\star_{sa}\big\|_1.
% \]

\paragraph{Interpreting the exploration RMDP construction.}
\label{sec:interpret-exploration-RMDP}
Consider \Cref{def:exploration-RMDP} of the \emph{exploration RMDP}.
The absorbing construction is used purely for analysis: in the implementation,
trajectories are not terminated upon visiting $\unknown_t$ or $S_T$. This does not affect the guarantee, since visiting some $(s,a)\in\unknown_t$ depends only on the trajectory prefix up to the first such visit. Continuing the trajectory can only increase the probability of this event, so the lower bound in \Cref{lem:exploration-prob} remains valid.
The state $z$ ensures that expected reward coincides with visitation probability (each trajectory contributes at most once); without it, the objective would instead reflect expected visit counts.
Finally, the exploration RMDP remains unchanged as long as $\unknown_t$ does not change. This allows an implementation optimisation: $\explore{\csg}_t$ need only be re-solved when new $(s,a)$ pairs become known, instead of doing so at every episode. 

% \paragraph{Interpretation.}
% Due to the absorbing state construction, any trajectory can visit $\unknown_t$ at most once. Hence if $\tau_{\unknown_t}$ denotes the hitting time of $\unknown_t$, then for any trajectory
% \[
% \sum_{h=0}^{H-1} \explore{r}_t(s_h,a_h)
% = \ind{\tau_{\unknown_t}<\infty}.
% \]
% Therefore, maximising expected cumulative reward is equivalent to maximising the probability of reaching $\unknown_t$. 
% In our implementation we use $g(\alpha_t)=\alpha_t^2$ to prioritise visiting highly uncertain state--action pairs, although we show in our ablation studies in \Cref{app:exploration-rew-experiment} that this does not influence the learning dynamics.

We now show that the reachability constant $\reach{p}$ lower bounds the exploration probability $\explore{p}_t(P^\star)$.

\ExplorationProb*
\begin{proof}
\label{prf:exploration-prob}
Fix episode $t$ with $\unknown_t\neq\emptyset$ and any $(s^\dagger,a^\dagger)\in \unknown_t$.

% Existence of a reaching policy
By \Cref{ass:reachability}, there exists a policy $\reach{\sigma}\in\Sigma$ such that
\[
\inf_{P:\,\supp(P)=\supp(P^\star)}
\prob^{\reach{\sigma},P}[\text{reach $(s^\dagger,a^\dagger)$ before $S_T$}]
\ge \reach{p}.
\]

% Optimality of the exploration policy
By definition of $\explore{\sigma}_t$ \eqref{eq:exploration-profile} as a maximiser of the robust value under $\explore{r}_t$, for any comparator $\sigma'\in\Sigma$,
\[
\inf_{P\in \explore{\Punc}_t} 
\ev^{\explore{\sigma}_t,P}\!\left[\sum_{h=0}^{H-1} \explore{r}_t(s_h,a_h)\right]
= \rob{u}_{\explore{\Punc}_t}(\explore{\sigma}_t)
\ge \rob{u}_{\explore{\Punc}_t}(\sigma')
=
\inf_{P\in \explore{\Punc}_t} 
\ev^{\sigma',P}\!\left[\sum_{h=0}^{H-1} \explore{r}_t(s_h,a_h)\right].
\]

% Key equivalence: reward equals first-hit event
Due to the absorbing construction of $\explore{\csg}_t$, any trajectory visits $\unknown_t$ at most once. Hence for any policy $\sigma$ and kernel $P$,
\[
\sum_{h=0}^{H-1} \explore{r}_t(s_h,a_h)
= \ind{\reachopt\ \unknown_t},
\]
and hence maximising expected cumulative reward is equivalent to maximising $\prob[\reachopt\ \unknown_t]$.

Therefore, by optimality of $\explore{\sigma}_t$ and taking $\sigma'=\reach{\sigma}$,
\begin{align}
\label{eq:exploration-prob-proof-step1}
\notag
\inf_{P\in \explore{\Punc}_t} 
\prob^{\explore{\sigma}_t,P}[\reachopt\ \unknown_t]
&\ge 
\inf_{P\in \explore{\Punc}_t} 
\prob^{\reach{\sigma},P}[\reachopt\ \unknown_t] \\
&\ge
\inf_{P\in \explore{\Punc}_t} 
\prob^{\reach{\sigma},P}[\text{reach $(s^\dagger,a^\dagger)$ before $S_T$}]
\ge
\reach{p}
\end{align}
where the last inequality follows from the definition of $\reach{p}$ \eqref{eq:p-reach}.

% Reduction to known slots
Now, since $(s^\dagger,a^\dagger) \in \unknown_t$, the event $\{\reachopt\ (s^\dagger,a^\dagger)\}$ depends only on the trajectory prefix
before the first visit to $\unknown_t$. By construction of $\explore{\csg}_t$ (\Cref{def:exploration-RMDP}), any unknown slot $(s,a)\in\unknown_t$ transitions deterministically to the absorbing state $z$, so a trajectory reaching $(s^\dagger,a^\dagger)$ must avoid all other unknown slots beforehand. Consequently, the trajectory up to the hitting time of $\unknown_t$ depends only on transitions over $\known_t$.

% On the event $\eventCI$, every $P \in \explore{\Punc}_t$ agrees with $P^\star$ on the support of transitions from $\known_t$ by graph preservation, i.e.,
% \[
% \supp(P\vert_{\known_t}) = \supp(P^\star\vert_{\known_t}).
% \]
% Hence,
% \[
% \inf_{P\in \explore{\Punc}_t} \prob^{\reach{\sigma},P}
% [\reachopt\ (s^\dagger,a^\dagger)]
% \geq
% \inf_{\substack{P:\,\supp(P\vert_{\known_t})=\supp(P^\star\vert_{\known_t})}}
% \prob^{\reach{\sigma},P}
% [\reachopt\ (s^\dagger,a^\dagger)].
% \]
% Since this class of kernels is contained in the support-preserving class considered in \Cref{ass:reachability}, we obtain
% \begin{equation}
% \label{eq:exploration-prob-proof-step2}
% \inf_{P\in \explore{\Punc}_t} \prob^{\reach{\sigma},P}
% [\reachopt\ (s^\dagger,a^\dagger)]
% \geq \reach{p}.
% \end{equation}

Finally, to relate this bound to the true kernel, define $P^{\exploreopt,\star}\in \explore{\Punc}_t$ to coincide with $P^\star\in \Punc_t$ (conditioning on $\eventCI$) on $\known_t$ and use the absorbing transitions on $\unknown_t\cup S_T$. Since the event $\{\reachopt\ \unknown_t\}$ depends only on transitions prior to the first visit to $\unknown_t$, and these transitions lie entirely within $\known_t$, we have
\[
\explore{p}_t(P^\star)
= \prob^{\explore{\sigma}_t, P^\star}[\reachopt\ \unknown_t]
\ge \prob^{\explore{\sigma}_t, P^{\exploreopt,\star}}[\reachopt\ \unknown_t]
\ge
\inf_{P\in \explore{\Punc}_t}
\prob^{\explore{\sigma}_t,P}[\reachopt\ \unknown_t]
\overset{\eqref{eq:exploration-prob-proof-step1}}{\ge} \reach{p}.
\]
% Since the implemented trajectories are not terminated at $\unknown_t$ or $S_T$, their probability of visiting $\unknown_t$ is at least $p_t^e(P_t^{e,\star})$.
% Conclusion
% Combining \eqref{eq:exploration-prob-proof-step1} and \eqref{eq:exploration-prob-proof-step2} yields
% \[
% \explore{p}_t(P^\star) \ge \reach{p}.
% \]
\end{proof}

\paragraph{Multi-trajectory exploration.}
\label{app:multi-trajectory-exploration}

Let $(\mathcal F_t)_{t\ge 0}$ be the filtration generated by the interaction history up to the end of episode $t$, and let $\mathcal F_{t-1}$ denote the history available at the start of episode $t$. At each episode $t$, we draw $\numSamples_t$ (conditionally) independent trajectories under $\explore{\sigma}_t$.

For $i\in\{1,\dots,\numSamples_t\}$, define
\[
X_{t,i} := \ind{\text{trajectory $i$ visits some } (s,a)\in \unknown_t},
\]
and let
\[
X_t := \ind{\exists\, i\le \numSamples_t : X_{t,i}=1}.
\]

\begin{proposition}
\label{prop:per-episode-exploration-prob}
Conditioned on $\mathcal{F}_{t-1}$, the variables $X_{t,1},\dots, X_{t,\numSamples_t}$ are independent and satisfy
\[
\Pr(X_{t,i}=1 \mid \mathcal F_{t-1}) \ge \explore{p}_t(P^\star)
\qquad \text{for all } i\in\{1,\dots,\numSamples_t\}.
\]
Consequently,
\[
\Pr(X_t=0 \mid \mathcal F_{t-1})
\le \bigl(1-\explore{p}_t(P^\star)\bigr)^{\numSamples_t}.
\]
\end{proposition}
\begin{proof}
Immediate from conditional independence of trajectories given $\mathcal{F}_{t-1}$ and the definition of $\explore{p}_t(P^\star)$ (see \Cref{sec:exploration}).
\end{proof}

\begin{corollary}[Per-episode exploration probability]
\label{cor:per-episode-exploration-prob-LB}
\[
\prob(X_t=0 \mid \mathcal{F}_{t-1}) 
\le \exp(-\reach{p} \numSamples_t),
\]
\end{corollary}
\begin{proof}
Combining \Cref{lem:exploration-prob} and \Cref{prop:per-episode-exploration-prob} we have
% $\explore{p}_t(P^\star)\ge \reach{p}$
\[
\prob(X_t=0 \mid \mathcal{F}_{t-1})
= \bigl(1-\explore{p}_t(P^\star)\bigr)^{\numSamples_t}
\le (1-\reach{p})^{\numSamples_t}
\le \exp(-\reach{p} \numSamples_t),
\]
where the last inequality follows from the fact that $1-x \le e^{-x}$ for $x\in(0,1)$.    
\end{proof}

\begin{proposition}[Number of samples per-episode]
\label{prop:num-samples-per-episode}
Consider a given failure probability $\confCount_t>0$. If
\[
\numSamples_t := \left\lceil \frac{1}{\reach{p}} \log\frac{1}{\confCount_t} \right\rceil
\]
then
\[
\prob(X_t=0 \mid \mathcal{F}_{t-1}) \le \confCount_t.
\]
\end{proposition}
\begin{proof}
% From the bound above,
% \[
% \prob(X_t=0 \mid \mathcal{F}_{t-1}) \le \exp(-\reach{p} \numSamples_t).
% \]
% By definition of $\numSamples_t$,
% \[
% \numSamples_t \ge \frac{1}{\reach{p}}\log\frac{1}{\confCount_t}.
% \]
Substituting the given definition of $\numSamples_t$ into the upper bound from \Cref{cor:per-episode-exploration-prob-LB} yields 
\[
\exp(-\reach{p} \numSamples_t)
\le \exp\left(-\reach{p} \cdot \frac{1}{\reach{p}}\log\frac{1}{\confCount_t}\right)
= \exp\left(-\log\frac{1}{\confCount_t}\right)
= \confCount_t.
\]
Hence
$\prob(X_t=0 \mid \mathcal{F}_{t-1}) \le \confCount_t$
as required.
\end{proof}

\begin{proposition}[Coverage confidence allocation]
\label{prop:coverage-confidence-allocation}
Let
\begin{equation}
\label{eq:confCount_t}
\confCount_t := \frac{\confCount}{t(t+1)}.
\end{equation}
Then, for every episode $t\ge 1$ with $\unknown_t\ne \emptyset$, the probability of visiting some unknown $(s,a)$ is at least $1-\confCount$.
\end{proposition}
\begin{proof}
Using the telescoping identity
\[
\sum_{t\ge1} \frac{1}{t(t+1)} = \sum_{t\ge1}\left( \frac{1}{t} - \frac{1}{t+1} \right) = 1,
\]
we obtain
\[
\sum_{t\ge 1} \confCount_t = \confCount.
\]
By \Cref{prop:num-samples-per-episode}, each episode satisfies
\[
\prob(X_t=0 \mid \mathcal{F}_{t-1}) \le \confCount_t.
\]
Applying a union bound over all episodes (with $\unknown_t\ne \emptyset$) yields
\[
\prob(\exists t \ge 1 : X_t=0) \le \sum_{t\ge1} \confCount_t = \confCount,
\]
so with probability at least $1-\confCount$, every episode visits some unknown $(s,a)$.
\end{proof}

Together, \Cref{prop:coverage-confidence-allocation,prop:num-samples-per-episode} give the multi-trajectory exploration strategy used in \Cref{alg:PAC-learning-CSG}: drawing
$\numSamples_t = \lceil (1/\reach{p}) \log(1/\confCount_t) \rceil$ trajectories per episode guarantees that each episode is exploratory with probability at least $1 - \confCount_t$.
%
% In practice, $\reach{p}$ may be unknown. Nevertheless, $\reach{p}$ depends only on the game graph and captures the intrinsic difficulty of reaching unknown state--action pairs: smaller $\reach{p}$ implies harder exploration and hence more trajectories may be required per episode.
%
% In the absence of an explicit estimate of $\reach{p}$, we adopt a practical proxy based on the size of the remaining unknown set. Intuitively, when many state--action pairs remain under-explored, more trajectories are needed to ensure sufficient coverage, while fewer samples suffice as the unknown set shrinks. This leads to the heuristic
% \[
% \numSamples_t \propto |\unknown_t| \log(1/\confCount_t),
% \]
% which scales exploration effort with the number of unresolved state--action pairs while retaining the standard logarithmic dependence on the confidence budget.
%

\section{Proofs for PAC-CSG (Theorem~\ref{thm:PAC-CSG})}
\label{app:PAC-CSG-proofs}
In this section we prove the main PAC guarantee, \Cref{thm:PAC-CSG}, by combining the three components developed in the previous Appendix items (\ref{app:per-slot-proofs}--\ref{app:exploration}). 
The proof proceeds in four steps:
\begin{enumerate}[(1)]
\item \emph{Confidence split}: we set $\confCI = \confCount = \delta/2$ and condition on $\mathcal{E} :=\eventCI \cap \eventCount$, which holds with probability $\geq 1 - \delta$.
\item \emph{Episode count}: \Cref{lem:counting-episodes} bounds the number of episodes until all relevant slots are known at $T^\star = \bigOtilde{\Rmax^2 H^4 |S|^2 |A| / \varepsilon^2}$.
\item \emph{Sample count}: at each episode $t$, \Cref{prop:num-samples-per-episode} (\Cref{app:exploration}) requires $\numSamples_t = \bigO{\log{t}/\reach{p}}$ sample trajectories, so the total number of samples required is $\sum_{t\le T^\star}{\numSamples_t} = \bigOtilde{T^\star / \reach{p}}$.
\item \emph{Correctness}: at termination, \Cref{cor:Delta-stopping-condition-correctness} applies, yielding either an $\varepsilon$-approximate equilibrium in $\csg^\star$ or the non-existence certificate.
\end{enumerate}

Henceforth we fix the target accuracy $\varepsilon\in (0,1)$ and confidence $\delta\in(0,1)$. 
Following \Cref{cor:anytime-value-gap}, to attain the near-optimality guarantee, we aim for
\[
4\Delta_t \le \varepsilon 
\quad \implies \quad \Delta_t\le \DeltaValFrac,
\]
and under this scheme the non-existence margin threshold in \Cref{prop:solver-EQ-guarantees}\ref{item:epsRNE-existence-suff-cond} becomes $\nashMarginThresh = 0$.

For the confidence parameter $\delta$, we distribute it as $\confCI = \confCount := \delta/2$. 

\begin{restatable}[Counting episodes]{lemma}{CountingEpisodes} \label{lem:counting-episodes} Fix a confidence parameter $\confCount\in (0,1)$. With probability at least $1-\confCount$, the total number of episodes executed by the algorithm until all relevant slots are known is at most 
\begin{equation} 
\label{eq:Mstar-upper-bound} T^\star := 2\totSlots + \frac{8}{3}\ln(1/\confCount). \end{equation} 
% \refproof{prf:counting-episodes} 
% \reforig{lem:counting-episodes} 
\end{restatable} 
% Note that multi-trajectory sampling removes the $1/\reach{p}$ dependence from the number of episodes, as each episode succeeds with probability at least $1-\confCount_t$.
\begin{proof}
\label{prf:counting-episodes}
Let $T$ be a fixed % (deterministic) 
number of episodes before all relevant slots are known. For $t=1,\dots,T$, recall the definitions of $X_{t,i}$, $X_t$, and $\mathcal{F}_{t-1}$ from \Cref{app:multi-trajectory-exploration}.
% We will show that with probability at least $1-\confCount$, all slots become known within $T$ episodes.
% For $t=1,\dots,T$, define
% \[
% X_{t,i} := \ind{\text{trajectory $i$ in episode $t$ visits some }(s,a)\in \unknown_t}, 
% \quad i=1,\dots,\numSamples_t,
% \]
% and
% \[
% X_t := \ind{\exists i\le \numSamples_t : X_{t,i}=1}.
% \]
% Let $\mathcal{F}_{t-1}$ be the history up to the start of episode $t$, and define
Further define
\[
\mu_t := \ev[X_t\mid\mathcal{F}_{t-1}],\qquad
\numExplEpisodes:=\sum_{t=1}^T X_t \le T.
\]
So $\numExplEpisodes$ is the number of exploratory episodes.

Define martingale differences $Y_t := X_t-\mu_t$ and $Z_t := -Y_t$. Then $Z_t$ are martingale differences with respect to the filtration $\mathcal{F}_{t-1}$. 
Thus
\[
M_T := \sum_{t=1}^T Z_t = \sum_{t=1}^T\mu_t - \numExplEpisodes.
\]
We have $\ev[Z_t \mid\mathcal{F}_{t-1}]=0$ and $|Z_t|\le1$. The predictable variance is
\begin{equation}
\label{eq:variance-freedman}
v := \sum_{t=1}^T \ev[Z_t^2\mid\mathcal{F}_{t-1}]
= \sum_{t=1}^T \mu_t(1-\mu_t)
\le \sum_{t=1}^T \mu_t
\le T,
\end{equation}
where the last inequality uses $\mu_t \le 1$.

We aim to ensure with high probability that $\numExplEpisodes\ge \totSlots$, the total number of visits required across all relevant slots. That is,
\begin{equation}
\label{eq:numExplEpisodes-failure-goal}
\Pr\left(\numExplEpisodes\le \totSlots\right) \le \confCount.
\end{equation}

% By \Cref{lem:exploration-prob}, for each trajectory $i \le \numSamples_t$ we have
% \[
% \Pr(X_{t,i}=1 \mid \mathcal{F}_{t-1}) \ge \reach{p}.
% \]
% Since, conditioned on $\mathcal{F}_{t-1}$, the policy $\explore{\sigma}_t$ is fixed and each trajectory is obtained by independently executing $\explore{\sigma}_t$ in $\csg^\star$, the trajectories are conditionally independent (see \Cref{app:multi-trajectory-exploration}), and thus
% \[
% \Pr(\forall i \le \numSamples_t : X_{t,i}=0 \mid \mathcal{F}_{t-1})
% = \prod_{i=1}^{\numSamples_t} \Pr(X_{t,i}=0 \mid \mathcal{F}_{t-1})
% \le (1-\reach{p})^{\numSamples_t}.
% \]
% Therefore,
% \[
% \mu_t 
% = \Pr(\exists i \le \numSamples_t : X_{t,i}=1 \mid \mathcal{F}_{t-1})
% \ge 1 - (1-\reach{p})^{\numSamples_t}.
% \]
% Using $1-p \le e^{-p}$ and the choice 
% \[
% \numSamples_t = \left\lceil \frac{1}{\reach{p}}\log\frac{1}{\confCount_t} \right\rceil,
% \]
% we obtain
% \[
% \mu_t \ge 1 - \confCount_t.
% \]
% Hence
% \[
% \sum_{t=1}^T \mu_t
% \ge \sum_{t=1}^T (1 - \confCount_t)
% = T - \sum_{t=1}^T \confCount_t
% \ge T - \confCount.
% \]
By \Cref{prop:num-samples-per-episode}, $\prob(X_t=0 \mid \mathcal{F}_{t-1}) \le \confCount_t$. Therefore $\mu_t = \prob(X_t=1 \mid \mathcal{F}_{t-1}) \ge 1-\confCount_t$ and
\[
\sum_{t=1}^T \mu_t
\ge \sum_{t=1}^T (1 - \confCount_t)
= T - \sum_{t=1}^T \confCount_t
\overset{\eqref{eq:confCount_t}}{\ge} T - \confCount.
\]
Therefore, on the event $\{\numExplEpisodes\le \totSlots\}$ we have
\begin{equation}
\label{eq:Te-le-totSlots-implies-M_T-le-Tslots}
M_T = \sum_{t=1}^T \mu_t - \numExplEpisodes
\ge T - \confCount - \totSlots
=: a.
\end{equation}
Thus,
\begin{equation}
\label{eq:prob-numExplEpisodes-to-M_T}
\Pr\!\left(\numExplEpisodes \le \totSlots\right)
\le \Pr\!\left(M_T \ge a\right).
\end{equation}
% Define the (deterministic) quantity
% \[
% a := T - \totSlots - \confCount.
% \]

Applying Freedman's inequality \citep{freedman1975tail}, 
and using the bound on the variance $v \le T$ \eqref{eq:variance-freedman}, we obtain
\begin{equation}
\label{eq:a-freedman-ineq}
\Pr(M_T\ge a) \le \exp{\left(-\frac{a^2}{2(v + a/3)}\right)}
\le \exp{\left(-\frac{a^2}{2(T + a/3)}\right)}
\end{equation}

We now choose
\begin{equation}
\label{eq:T-wrt-CL-totSlots}
T = 2\totSlots + C L,
\quad L:=\ln(1/\confCount),
\end{equation}
for a constant $C>0$ to be determined. Then
\[
a = T - \totSlots - \confCount = \totSlots + CL - \confCount.
% \qquad
% v = T = 2\totSlots + CL.
\]

Substituting into the exponent in \eqref{eq:a-freedman-ineq}, to achieve \eqref{eq:numExplEpisodes-failure-goal} it suffices to ensure
\[
\frac{(\totSlots + CL)^2}{2\big(T + a/3\big)} 
\ge
\frac{(\totSlots + CL - \confCount)^2}{2\big(T + a/3\big)} \ge L,
\]
since $\confCount>0$. Further,
% \[
% \totSlots + CL - \confCount \ge \totSlots + CL - 1,
% \]
% and
\[
T + \frac{a}{3}
= 2\totSlots + CL + \frac{1}{3}\big(\totSlots + CL - \confCount\big)
\le 2\totSlots + CL + \frac{1}{3}\big(\totSlots + CL\big)
= \frac{1}{3}(7\totSlots + 4CL)
\]

Therefore, it suffices to ensure
\begin{equation}
\label{eq:freedman-exponent-ineq-fixed}
\frac{3(\totSlots + CL)^2}{2\left(7\totSlots + 4CL \right)} 
\ge L.
\end{equation}

% Since the left-hand side is increasing in $\totSlots$ for $\totSlots \ge 0$, replacing $(\totSlots + CL - 1)^2$ by $(\totSlots + CL)^2$ only strengthens the inequality. Hence it suffices to require
% \begin{equation}
% \label{eq:freedman-exponent-ineq-fixed}
% \frac{(\totSlots + CL)^2}{2\left(2\totSlots + CL + \frac{1}{3}(\totSlots + CL)\right)} \ge L.
% \end{equation}

We find that $C=8/3$ is the minimal constant ensuring that \eqref{eq:freedman-exponent-ineq-fixed} holds; we defer the case analysis establishing this to the end (see \hyperref[sec:solving-for-C]{\textbf{Solving for $C$}}).

With this choice of $C$, \eqref{eq:a-freedman-ineq} yields
\[
\Pr(M_T \ge a) \le \exp(-L) = \confCount.
\]
By \eqref{eq:prob-numExplEpisodes-to-M_T}, this implies
\[
\Pr(\numExplEpisodes \le \totSlots)
\le \Pr(M_T \ge a)
\le \confCount.
\]
Thus, with probability at least $1-\confCount$, we have $\numExplEpisodes \ge \totSlots$.

Since each exploratory episode discovers at least one previously unknown slot, a standard counting argument implies that after at most $\totSlots$ such episodes all relevant slots are known. Hence, with probability at least $1-\confCount$, all relevant slots become known within $T$ episodes.

Finally, substituting $C=8/3$ into \eqref{eq:T-wrt-CL-totSlots} yields
\[
T^\star = 2\totSlots + \tfrac{8}{3}\ln(1/\confCount).
\]

\paragraph{Solving for $C$.}
\label{sec:solving-for-C}
We now return to \eqref{eq:freedman-exponent-ineq-fixed} and solve for the smallest constant $C$ that ensures the inequality holds.
Introduce $x := \totSlots/L \ge 0$ (so $\totSlots = xL$). Then
\begin{alignat}{4}
&\qquad &&\frac{3L^2(x+C)^2}{2L(7x+4C)}
= \frac{3L(x+C)^2}{2(7x+4C)} \ge L \nonumber \\
&\iff &&\frac{3(x+C)^2}{2(7x+4C)} \ge 1 \nonumber \\
&\iff &&3(x+C)^2 \ge 14x + 8C \nonumber \\
&\iff &&h(x) := 3(x+C)^2 - 14x - 8C 
= 3x^2 + (6C-14)x + (3C^2 - 8C) \ge 0.
\label{eq:freedman-ht-ge-0}
\end{alignat}

We require $h(x)\ge0$ for all $x\ge0$. Since $h$ is convex, its minimum occurs at
\[
x^\star = \frac{7}{3}-C.
\]

We distinguish two cases.

\paragraph{Case A: $x^\star \ge 0$ (i.e., $C \le 7/3$).}
The minimum of $h$ on $[0,\infty)$ occurs at $x^\star$, yielding
\[
h(x^\star) = 6C - \frac{49}{3}.
\]
Requiring $h(x^\star)\ge0$ gives $C \ge 49/18$, contradicting $C \le 7/3$. Hence no feasible $C$ exists in this regime.

\paragraph{Case B: $x^\star < 0$ (i.e., $C > 7/3$).}
In this case, the minimum of $h$ on $[0,\infty)$ occurs at the boundary $x=0$, giving
\[
h(0) = C(3C-8).
\]
Requiring $h(0)\ge0$ yields $C \ge 8/3$.

Therefore, the smallest constant $C$ such that \eqref{eq:freedman-ht-ge-0} holds for all $x \ge 0$ is
\[
\boxed{C = \tfrac{8}{3}}.
\]
\end{proof}

\begin{lemma}[Counting per-slot visits]
\label{lem:nmin}
Define $\confCI_{sa,n} := \confCI/(\nTrans n(n+1))$.
If a slot $(s,a)\in \ntslots$ is visited
\[
n_{\min} := \nminExpr = \nminBigOtilde
\]
times, then
\[
\alpha_{sa}(n; \confCI_{sa,n}) \leq \varepsilon/(2\Rmax H^2).
\]
% with probability $\ge 1-\confCI_{sa, n_{\min}}$ we have that,
% \[
% \|\widehat{P}_{sa} - P^\star_{sa}\|_1 \le \frac{\varepsilon}{4\Rmax H^2}.
% \]
\end{lemma}

\begin{proof}
% Recall from \Cref{cor:true-kernel-containment} that $\confCI_{sa,n} = \confCI/(\nTrans n(n+1))$.
%
Applying \Cref{lem:per-slot-concentration}, we require
\[
\sqrt{\frac{2}{n}\ln\frac{(2^{|S|}-2)}{\confCI_{sa,n}}}
= \sqrt{\frac{2}{n}\ln\frac{(2^{|S|}-2)\nTrans n(n+1)}{\confCI}}
\leq \frac{\varepsilon}{2\Rmax H^2}.
\]

\paragraph{Solving for $n$.}
Squaring both sides and rearranging gives 
\[
\frac{2}{n}\ln\frac{(2^{|S|}-2)\nTrans n(n+1)}{\confCI} \le \frac{\varepsilon^2}{4\Rmax^2 H^4}
\quad\implies \quad
\ln\frac{(2^{|S|}-2)\nTrans n(n+1)}{\confCI}
\le \frac{\varepsilon^2}{8\Rmax^2 H^4}n.
\]
Let
\[
C:=\frac{(2^{|S|}-2)\nTrans}{\confCI},
\qquad 
\beta:=\frac{\varepsilon^2}{8\Rmax^2 H^4}.
\]
Then we need
\[
\ln{(Cn(n+1))} \le \beta n.
\]
Since $n(n+1)\le 2n^2$ for all $n\ge 1$, it is enough to require
\[
\ln{(2Cn^2)} \le \beta n.
\]
Solving the corresponding equality gives
\begin{equation}
\label{eq:nmin-W-1}
n = -\frac{2}{\beta} W_{-1}\!\left(-\frac{\beta}{2\sqrt{2C}}\right),
\end{equation}
where $W$ is the Lambert W function \citep{lambert1758observationes,euler1783serie} defined by $W(x)e^{W(x)} = x$, and $W_{-1}$ is the lower real branch of $W$ mapping $x\in [-1/e, 0)$\footnote{The argument of $W_{-1}$ % in \eqref{eq:nmin-W-1} 
lies in $[-1/e,0)$ under the standard PAC parameter regimes (e.g., ($\confCI\in(0,1)$).} to $W_{-1}(x) \le -1$.

Therefore a sufficient condition is
\[
n \ge \nminExpr =: n_{\min}
\]
Asymptotically,
\begin{equation}
\label{eq:nmin-bigO-tilde}
n_{\min} = \bigO{\frac{\Rmax^2 H^4}{\varepsilon^2}\log\frac{(2^{|S|}-2)\nTrans}{\confCI}}
= \nminBigOtilde.
\end{equation}
% \[
% n \geq \frac{2\cdot 16\Rmax^2 H^4}{\varepsilon^2}
% \ln\frac{(2^{|S|}-2)\nTrans}{\confCI_t}
% = \frac{\nminCoef \Rmax^2 H^4}{\varepsilon^2}
% \ln\frac{(2^{|S|}-2)\nTrans}{\confCI_t}.
% \]
\end{proof}

Henceforth we define $n_{\min}$ as in \Cref{lem:nmin}.

\begin{corollary}
\label{cor:nmin}
    % On $\eventCI$,
    % (from \Cref{cor:true-kernel-containment}), 
    If every $(s,a)\in\ntslots$ is visited at least $n_{\min}$ times,
    % , where $n_{\min}$ is defined as in \Cref{lem:nmin}, 
    then $\Delta_t \le \DeltaVal$.
\end{corollary}
\begin{proof}
By \Cref{lem:nmin}, if every $(s,a)\in \ntslots$ is visited at least $n_{\min}$ times, then on $\eventCI$ we have: 
\[
\|\widehat{P}_{sa} - P^\star_{sa}\|_1 \le \alpha_{sa}(n_{\min}; \confCI_{sa, n_{\min}}) \leq \frac{\varepsilon}{2\Rmax H^2}
\qquad \forall (s,a)\in \ntslots.
\]
% where the last inequality follows from \Cref{lem:nmin}. 
Therefore, by definition of $\Delta_t$:
\begin{align*}
\Delta_t 
= \frac{1}{2}\Rmax H^2 \max_{(s,a)\in\ntslots}{\|\widehat{P}_{sa} - P^\star_{sa}\|_1}
\le \frac{1}{2}\Rmax H^2 \cdot \frac{\varepsilon}{2\Rmax H^2}
= \DeltaValFrac.
\end{align*}
\end{proof}

\begin{corollary}
% [$\Delta$ stopping condition]
\label{cor:Delta-stopping-condition-correctness}
% At episode $t$, % conditioning on $\eventCI$, 
Consider the $\SolveRCSG$ oracle from \Cref{ass:solver} with tolerance $\varepsilon$. Suppose $\Delta_t \le \DeltaVal$. Then:
\begin{enumerate}[(a)]
    \item If $\opt{\mu}\ge 0$, the solver returns $\Found$, and the output profile $\hat{\sigma}$ is an $\varepsilon$-NE under $P^\star$, satisfying
    \[
    V^\star - u(\hat{\sigma}, P^\star) \le \varepsilon.
    \]
    \item If $\opt{\mu}<0$, then:
    \begin{enumerate}[(i)]
        \item if the solver returns $\Found$, the output profile $\hat{\sigma}$ is still an $\varepsilon$-NE under $P^\star$;
        \item if the solver returns $\notFound$, this provides a sound certificate that $\opt{\mu}$ is not well-defined or $\opt{\mu}<0$.
    \end{enumerate}
\end{enumerate}
\end{corollary}
\begin{proof}
Suppose $\Delta_t \le \DeltaVal$.
% the Nash margin threshold from \Cref{prop:solver-EQ-guarantees} satisfies 
Then $4\Delta_t -\varepsilon \le 0$. We now proceed by a case analysis on the sign of $\opt{\mu}$.
\begin{enumerate}[(a)]
    \item Suppose $\opt{\mu} \ge 0$. 
    By \Cref{prop:opt-mu-to-SW-mu}, this is equivalent to $\mu^\star \ge 0$. 
    Since $\mu^\star \ge 0 \ge 4\Delta_t-\varepsilon$, % (by assumption $4\Delta_t -\varepsilon \le 0$)
    the contrapositive of \Cref*{prop:solver-EQ-guarantees}\ref{item:epsRNE-existence-suff-cond} 
    % says: if ($\mu^\star$ is defined and) $\mu^\star \ge 4\Delta_t-\varepsilon$, then the solver returns $\Found$. 
    implies that the solver returns $\Found$.
    Now by \Cref{ass:solver}, returning $\Found$ means that the output profile $\hat{\sigma}$ is an $\varepsilon$-RSWNE under $\Punc_t$. 
    Then condition on $\eventCI$, \Cref*{prop:solver-EQ-guarantees}\ref{item:solver-success-cert} implies that $\hat{\sigma}$ is also an $\varepsilon$-NE of $P^\star \in \Punc_t$.

    Furthermore, since $\mu^\star \ge \nashMarginThresh$, we may apply \Cref{cor:anytime-value-gap}, which gives
    \[
    V^\star - u(\hat{\sigma}, P^\star)
    \le \DeltaValCoef \Delta_t
    \le \varepsilon.
    \]

    \item Now suppose $\opt{\mu}<0$.
    \begin{enumerate}[(i)]
        \item If the solver returns $\Found$, outputting an $\varepsilon$-RSWNE $\hat{\sigma}$ of $\Punc_t$, then \Cref{prop:solver-EQ-guarantees}\ref{item:solver-success-cert} implies that $\hat{\sigma}$ is also an $\varepsilon$-NE of $P^\star$.

        \item If the solver returns $\notFound$, then by \Cref{prop:solver-EQ-guarantees}\ref{item:epsRNE-existence-suff-cond}, either $\mu^\star$ is undefined or $\mu^\star < 4\Delta_t-\varepsilon \le 0$. 
        Therefore, by \Cref{prop:opt-mu-to-SW-mu}, $\opt{\mu}$ is either undefined or satisfies $\opt{\mu}<0$.
    \end{enumerate}
    % if the solver succeeds then its output $\hat{\sigma}$ is an $\epsEqVal$-RSWNE under $\Punc_t$ and an $\varepsilon$-NE of $P^\star$.
\end{enumerate}
\end{proof}

\PACCSG*
\begin{proof}
\label{prf:PAC-CSG}
Recall the containment and coverage events
\[
\eventCI := \bigcap_{t\ge 1}\{P^\star \in \Punc_t\}, 
\qquad
\eventCount := \{\text{all $(s,a)\in\ntslots$ are visited at least } n_{\min} \text{ times}\}.
\]
By \Cref{cor:true-kernel-containment} and \Cref{lem:counting-episodes}, 
\[
\Pr(\eventCI) \ge 1-\confCI, 
\qquad 
\Pr(\eventCount) \ge 1-\confCount.
\]
Setting $\confCI=\confCount=\delta/2$ and defining $\mathcal{E}:=\eventCI\cap\eventCount$, we have $\Pr(\mathcal{E}) \ge 1-\delta$.

\paragraph{Sample complexity.}
Let 
\[
T := \inf\{t : \Delta_t \le \DeltaVal \}.
\]
For any $t < T$, the stopping condition is false, in particular $\Delta_t > \DeltaVal$. 
% Note that once all relevant slots are known, each $(s,a)\in\ntslots$ has been visited at least $n_{\min}$ times, which by \Cref{cor:nmin} implies $\Delta_t \le \DeltaVal$, so the stopping condition is satisfied.

Substituting $n_{\min}$ from \Cref{cor:nmin} into the episode bound $T^\star$ from \Cref{lem:counting-episodes} yields
\[
T^\star 
= 2|S||A|n_{\min} + \frac{8}{3}\ln\!\left(\frac{1}{\confCount}\right)
= \bigO{\frac{\Rmax^2 H^4 |S|^2 |A|}{\varepsilon^2}}.
\]
By \Cref{prop:num-samples-per-episode}, each episode requires 
\[
\left\lceil \frac{1}{\reach{p}}\log\frac{1}{\confCount_t} \right\rceil 
= \bigO{\frac{1}{\reach{p}}\log\frac{t}{\confCount}}
\]
trajectories. Summing over episodes,
\begin{align*}
\sum_{t=1}^{T^\star} \frac{1}{\reach{p}}\log\frac{t}{\confCount}
&= \frac{1}{\reach{p}}\left( \sum_{t=1}^{T^\star}\log t - T^\star \log \confCount \right) \\
&= \frac{1}{\reach{p}}\left(\log(T^\star!) - T^\star \log \confCount \right) \\
&= \bigO{\frac{T^\star}{\reach{p}} \log\frac{T^\star}{\confCount}}
= \bigOtilde{ \frac{T^\star}{\reach{p}} },
\end{align*}
where we use Stirling’s approximation~\citep{stirling1730methodus} $\log(T!) = \bigO{T\log T}$.

\paragraph{Termination.}
Let $\countop{T}$ denote the first episode at which $\eventCount$ holds. On $\eventCount$, \Cref{lem:counting-episodes} gives
\[
\countop{T} \le T^\star.
\]
By \Cref{cor:nmin}, on $\eventCI$, once $\eventCount$ holds we have
\[
\Delta_{\countop{T}} \le \DeltaVal.
\]
Hence, on $\mathcal{E}$, the stopping condition is satisfied by episode $\countop{T}$, and the algorithm terminates no later than $T^\star$.

\paragraph{Correctness of the output.}
On $\mathcal{E}$, we have both $P^\star \in \Punc_t$ for all $t$ and $\Delta_T \le \DeltaVal$. Therefore, \Cref{cor:Delta-stopping-condition-correctness} applies at termination, yielding the required output guarantee and completing the proof.
\end{proof}

\section{Supplementary details and proofs for infinite-horizon objectives}
\label{app:infh}
In this section, we extend the PAC-CSG framework to infinite-horizon objectives (unbounded probabilistic reachability and reachability reward / SSP). Following standard SSP formulations~\citep{bertsekas1991analysis,kwiatkowska2021automatic,HP26}, we assume the existence of at least one proper profile and restrict attention to this class throughout the analysis.
Let 
\[
\tau_T := \inf\{h \ge 0 : s_h \in S_T\}
\]
denote the hitting time of the target set $S_T$. 
Further let $\proper{\Sigma} \subseteq \Sigma$ be the set of proper profiles, under which $\tau_T < \infty$ almost surely and all values are well-defined and that the underlying $\SolveRCSG$ oracle converges. The assumption can be verified via standard graph-based reachability analyses under the known support structure induced by graph preservation~\citep{de2000graphalgo}.
% We proceed as follows: \Cref{ass:nz-almost-sure-reachability} formalises almost-sure reachability; \Cref{prop:inductive-tauT} establishes a geometric tail bound on the stopping time; \Cref{lem:effective-horizon} bounds $\Heff \leq |S|/p_T$; \Cref{prop:tv-propagation-infh} and \Cref{lem:sensitivity-infh} extend the value sensitivity result; finally, \Cref{thm:PAC-CSG-infh} states the full PAC guarantee.

\begin{assumption}[Existence of a proper profile]
\label{ass:nz-almost-sure-reachability}
There exists a proper profile $\sigma \in \proper{\Sigma}$, i.e., one that reaches the target set $S_T$ almost surely from all states under all support-preserving kernels:
\[
\prob^{\sigma,P}_s(\tau_T < \infty) = 1
\quad \forall s \in S,\ \forall P:\supp(P)=\supp(P^\star).
\]
Additionally, the stopping probability is uniformly bounded away from zero:
\[
p_T := \min_{s \in S}\ \inf_{\sigma \in \proper{\Sigma}}\
\inf_{P: \supp(P) = \supp(P^\star)}
\prob^{\sigma,P}\!\left[
\tau_T \leq |S| \mid s_0 = s
\right] > 0.
\]
\end{assumption}

This condition strengthens properness by requiring a uniform lower bound on finite-time reachability. It is analogous to standard reachability and diameter assumptions in PAC and reinforcement learning analyses~\citep{jaksch2010near,azar2017minimax,kearns2002nearE3,brafman2002RMAX}, and to the properness conditions used in SSP formulations~\citep{bertsekas1991analysis}. It is required to obtain the geometric tail bound in \Cref{prop:inductive-tauT} and the finite effective horizon in \Cref{lem:effective-horizon}. 
% We henceforth restrict attention to the set of proper profiles, under which $\tau_T < \infty$ almost surely and all quantities are well-defined.
We discuss this assumption and its practical implications further in \Cref{app:limitations}.

Next, we establish a geometric tail bound on the stopping time in terms of $p_T$.
\begin{proposition}
\label{prop:inductive-tauT}
Under \Cref{ass:nz-almost-sure-reachability}, for all $\sigma \in \proper{\Sigma}$ and all support-preserving $P$:
\[
\prob^{\sigma,P}[\tau_T > k|S|] \leq (1 - p_T)^k
\qquad \forall k \geq 0.
\]
\end{proposition}
\begin{proof}
We proceed by induction on $k$. The base case $k=0$ is immediate.

For the inductive step, suppose $\prob^{\sigma,P}[\tau_T > k|S|] \leq (1-p_T)^k$.
Conditioned on $\{\tau_T > k|S|\}$ we have $s_{k|S|} \notin S_T$.

By the Markov property, the probability of reaching $S_T$ within the next $|S|$ steps depends only on the current state $s_{k|S|}$. By \Cref{ass:nz-almost-sure-reachability}, from any such state the probability of reaching $S_T$ within the next $|S|$ steps is at least $p_T$, uniformly over the realised state $s_{k|S|}$:
\[
\prob^{\sigma,P}[\tau_T \leq (k+1)|S| \mid \tau_T > k|S|,\ s_{k|S|}] \geq p_T.
\]
Hence,
\[
\prob^{\sigma,P}[\tau_T > (k+1)|S| \mid \tau_T > k|S|] \leq 1 - p_T.
\]
Therefore,
\[
\prob^{\sigma,P}[\tau_T > (k+1)|S|] 
= \prob^{\sigma,P}[\tau_T > (k+1)|S| \mid \tau_T > k|S|] \cdot \prob^{\sigma,P}[\tau_T > k|S|]
\leq (1-p_T)^{k+1},
\]
where the final inequality uses the inductive hypothesis.
\end{proof}

Under \Cref{ass:nz-almost-sure-reachability}, we have $\tau_T < \infty$ almost surely under all proper policies and support-preserving kernels, so the reachability reward is finite. Moreover, the assumption provides a uniform lower bound $p_T>0$ on the probability of reaching $S_T$ within $|S|$ steps from any state; this quantity can be computed by solving a reachability RMDP for $S_T$. This induces an \emph{effective horizon}.

\paragraph{Effective horizon.}
Define the effective horizon as
\[
\Heff := \sup_{\sigma\in \proper{\Sigma}}\sup_{P:\,\supp(P)=\supp(P^\star)} \ev^{\sigma,P}[\tau_T].
\]
The following result shows that $\Heff$ is finite and admits the bound $\Heff \le |S|/p_T$.

\begin{restatable}[Effective horizon]{lemma}{EffectiveHorizon}
\label{lem:effective-horizon}
\[
\Heff \leq \frac{|S|}{p_T} < \infty.
\]
\end{restatable}

\begin{proof}
\label{prf:effective-horizon}
For any $\sigma\in \proper{\Sigma}$ and support-preserving $P$,
\begin{align*}
\ev^{\sigma,P}[\tau_T]
&= \sum_{n=0}^{\infty} \prob^{\sigma,P}[\tau_T > n] 
= \sum_{k=0}^{\infty} \sum_{j=0}^{|S|-1} \prob^{\sigma,P}[\tau_T > k|S| + j] \\
&\le \sum_{k=0}^{\infty} \sum_{j=0}^{|S|-1} \prob^{\sigma,P}[\tau_T > k|S|]
\le |S| \sum_{k=0}^{\infty} \prob^{\sigma,P}[\tau_T > k|S|] \\
&\le |S| \sum_{k=0}^{\infty} (1-p_T)^k
= \frac{|S|}{p_T},
\end{align*}
where the last inequality follows from \Cref{prop:inductive-tauT}.
\end{proof}

Hence, $\Heff$ can be conservatively bounded by $|S|/p_T$ when unknown. In the following, we extend the sensitivity lemma (\Cref{lem:sensitivity}) to the infinite-horizon setting by replacing $H$ with $\Heff$, which induces only a constant-factor change in $\Delta_t$.

We first prove the stopped analogue of \Cref{prop:tv-propagation} from the finite-horizon setting (\Cref{app:per-episode-proofs}), following the same proof structure.

\begin{proposition}[Stopped one-step error propagation]
\label{prop:tv-propagation-infh}
Let $\mu_h, \widehat{\mu}_h$ denote the stopped marginals at time-step $h$
under $P$ and $\widehat{P}$ respectively, following a proper profile $\sigma$ from initial state $s_0\in S$:
\[
\mu_h(s) := \prob^{\sigma,P}[s_h = s,\ h < \tau_T], 
\quad 
\widehat{\mu}_h(s) := \prob^{\sigma,\widehat{P}}[s_h = s,\ h < \tau_T].
\]
Define
\[
\diffLI_t := \max_{(s,a)\in\ntslots}\|P_{sa}-\widehat{P}_{sa}\|_1.
\]
Then for all $h \ge 1$,
\[
\|\mu_h - \widehat{\mu}_h\|_1 \le h\,\diffLI_t.
\]
\end{proposition}

\begin{proof}
We prove the recursive bound
\[
\|\mu_{h+1}-\widehat{\mu}_{h+1}\|_1
\le \|\mu_h-\widehat{\mu}_h\|_1 + \diffLI_t,
\]
from which the result follows by induction using $\mu_0=\widehat{\mu}_0=\delta_{s_0}$.

Define the stopped one-step kernels, for $s \notin S_T$,
\[
K_s := \ind{\cdot \notin S_T}\sum_a \sigma(a|s)P(\cdot|s,a),
\qquad
\widehat{K}_s := \ind{\cdot \notin S_T}\sum_a \sigma(a|s)\widehat{P}(\cdot|s,a).
\]
These are substochastic, i.e., $\sum_{s'} K_s(s') \le 1$.

Then
\[
\mu_{h+1} = \sum_s \mu_h(s)K_s,
\qquad
\widehat{\mu}_{h+1} = \sum_s \widehat{\mu}_h(s)\widehat{K}_s,
\]
so
\[
\mu_{h+1}-\widehat{\mu}_{h+1}
= \sum_s (\mu_h(s)-\widehat{\mu}_h(s))K_s
  + \sum_s \widehat{\mu}_h(s)(K_s-\widehat{K}_s).
\]

Taking $L^1$ norms and applying the triangle inequality:
\[
\|\mu_{h+1}-\widehat{\mu}_{h+1}\|_1
\le \Big\|\sum_s (\mu_h-\widehat{\mu}_h)K_s\Big\|_1
   + \Big\|\sum_s \widehat{\mu}_h(s)(K_s-\widehat{K}_s)\Big\|_1.
\]

\paragraph{First term.}
Substochasticity implies
\[
\Big\|\sum_s (\mu_h-\widehat{\mu}_h)K_s\Big\|_1
\le \|\mu_h-\widehat{\mu}_h\|_1.
\]

\paragraph{Second term.}
By convexity,
\[
\Big\|\sum_s \widehat{\mu}_h(s)(K_s-\widehat{K}_s)\Big\|_1
\le \max_{s\notin S_T} \|K_s-\widehat{K}_s\|_1.
\]
For each $s\notin S_T$,
\[
\|K_s-\widehat{K}_s\|_1
\le \sum_a \sigma(a|s)\|P_{sa}-\widehat{P}_{sa}\|_1
\le \diffLI_t.
\]

Substituting gives
\[
\|\mu_{h+1}-\widehat{\mu}_{h+1}\|_1
\le \|\mu_h-\widehat{\mu}_h\|_1 + \diffLI_t.
\]

Unrolling yields
\[
\|\mu_h-\widehat{\mu}_h\|_1 \le h\,\diffLI_t.
\]
\end{proof}

We now lift \Cref{prop:tv-propagation-infh} to derive the infinite-horizon counterpart of the sensitivity lemma (\Cref{lem:sensitivity}). 
% Note that, unlike the finite-horizon case, the factor $1/2$ does not appear in the definition of $\Delta_t$ here because the infinite-horizon proof uses a blockwise tail bound, which yields the coarser estimate $\sum_{h\ge0} h \prob[h<\tau_T]\le \Heff^2$ directly. 
\begin{restatable}[Infinite-horizon sensitivity]{lemma}{InfhSensitivity}
\label{lem:sensitivity-infh}
Under \Cref{ass:nz-almost-sure-reachability}, define 
\[
\Delta_t := \Rmax \Heff^2 \diffLI_t, \qquad \Heff := \left\lceil \frac{|S|}{p_T} \right\rceil.
\]
Then for any $P \in \Punc_t$, profile $\sigma\in \proper{\Sigma}$, player $i \in N$ and state $s_0\in S$,
\[
|u_i(\sigma,P\mid s_0) - u_i(\sigma,\widehat{P}\mid s_0)| \le \Delta_t.
\]
\end{restatable}
% We therefore absorb the constant into $\Delta_t$.

\begin{proof}
Let $V_P(s):=u_i(\sigma,P\mid s)$ and
$\widehat V(s):=u_i(\sigma,\widehat P\mid s)$, with
$V_P(s)=\widehat V(s)=0$ for $s\in S_T$.
For $s\notin S_T$, let
\[
P^\sigma_s:=\sum_a \sigma(a\mid s)P(\cdot\mid s,a),
\qquad
\widehat P^\sigma_s:=\sum_a \sigma(a\mid s)\widehat P(\cdot\mid s,a).
\]
Then
\[
V_P(s) = r_i^\sigma(s) + \sum_{s'}P^\sigma_s(s')V_P(s'),
\]
and analogously for $\widehat V$.
Since $|\tau_T|$ has expectation at most $\Heff$ under
\Cref{ass:nz-almost-sure-reachability},
\[
\|V_P\|_\infty,\|\widehat V\|_\infty
\le \Rmax \Heff.
\]
Moreover, for every $s\notin S_T$,
\[
\|P^\sigma_s-\widehat P^\sigma_s\|_1
\le \sum_a \sigma(a\mid s)\|P_{sa}-\widehat P_{sa}\|_1
\le \diffLI_t.
\]
Hence, for $s\notin S_T$,
\[
|V_P(s)-\widehat V(s)|
\le \Rmax\Heff\diffLI_t
+ \sum_{s'}P^\sigma_s(s')|V_P(s')-\widehat V(s')|.
\]
Unrolling this inequality until $\tau_T$ gives
\[
|V_P(s_0)-\widehat V(s_0)|
\le \Rmax\Heff\diffLI_t\, \ev^{\sigma,P}_{s_0}[\tau_T]
\le \Rmax\Heff^2\diffLI_t
= \Delta_t.
\]
\end{proof}

\Cref{lem:sensitivity-infh} recovers \Cref{lem:sensitivity} (up to constant factors) with the substitution $H \leftarrow \Heff$. All intermediate results in the finite-horizon case 
% (\Cref{cor:per-episode-error,cor:anytime-value-gap,lem:NE-RNE-transfer,lem:exploration-prob,lem:counting-episodes,cor:nmin}) 
carry over with $H \leftarrow \Heff$ and $n_{\min}$ as stated in \Cref{thm:PAC-CSG-infh} below.

\begin{theorem}[PAC guarantee for infinite-horizon objectives]
\label{thm:PAC-CSG-infh}
% Under \Cref{ass:nz-almost-sure-reachability}, 
\Cref{alg:PAC-learning-CSG} with 
\[
n_{\min} 
= \nminExprInfh 
= \nminBigOtildeInfh
\]
satisfies the PAC-CSG guarantee of \Cref{def:PAC-CSG} under the unbounded probabilistic reachability and reachability reward objectives, with sample complexity
\[
T^\star = \bigOtilde{\frac{\Rmax^2 |S|^6 |A|}{\varepsilon^2 p_T^4 \reach{p}}}.
\]
\end{theorem}

% The sample complexity $\bigOtilde{|S|^6/p_T^4}$ is conservative, arising from the worst-case bound $\Heff \le |S|/p_T$. When $\Heff$ is known, this improves to $\bigOtilde{\Rmax^2 \Heff^4 |S|^2 |A| / (\varepsilon^2 \reach{p})}$ as in \Cref{thm:PAC-CSG}.

\section{Extension: Learning zero-sum CSGs}
\label{app:zero-sum}
We now specialise the framework to two-player zero-sum CSGs. Since the players have directly opposing objectives, i.e., $X_1 = -X_2$, we write $X := X_1$, so that $X_2 = -X$. Accordingly, we consider a single utility and reward function, writing $u := u_1 = -u_2$ and $r := r_1 = -r_2$.

In a zero-sum CSG $\csg$, the (minimax) \emph{value} of $\csg$ with respect to $X$ is defined whenever the game is \emph{determined}, i.e., when
\[
\sup_{\sigma_1\in\Sigma_1} \inf_{\sigma_2\in\Sigma_2} u(\sigma_1,\sigma_2) = \inf_{\sigma_2\in\Sigma_2} \sup_{\sigma_1\in\Sigma_1} u(\sigma_1,\sigma_2).
\]
In this case, the value is unique, and the corresponding optimal strategies form the set of Nash equilibria of $\csg$~\citep{von1944theory}. 

Extending to robust CSGs, let $\csg$ be an RCSG with uncertainty set $\Punc$.
We say the game is \emph{robustly determined} if determinacy holds under
worst-case transition uncertainty, i.e., with respect to $\rob{u}_{\Punc}(\sigma) = \inf_{P \in \Punc} u(\sigma,P)$.
The resulting \emph{robust value} of $\csg$ is 
\[
\rob{V} := \sup_{\sigma_1} \inf_{\sigma_2} \inf_{P \in \Punc} u(\sigma,P),
\]
which coincides with the value of the induced CSG under an adversarial kernel (cf.~\citet[Definition~6]{HP26}). The corresponding
\emph{robust optimal strategies} guarantee a value of at least $\rob{V}$ for all $P \in \Punc$.

For our objectives of interest, zero-sum CSGs are always determined in the finite-horizon setting, and optimal strategies exist due to the finite game tree. In contrast, in the infinite-horizon setting, a value exists (e.g., in the limit-sure sense for reachability objectives), but optimal strategies need not exist~\citep{de2007concurrent}.

\paragraph{Simplifications relative to the general-sum case.}
Since only a single utility function $u$ is involved, the sensitivity (\Cref{lem:sensitivity}) and per-episode value gap (\Cref{cor:anytime-value-gap}) from the general-sum case apply without the factor-of-$2$ loss incurred when summing over players. Consequently, the stopping threshold on $\Delta_t$ improves from $\varepsilon/4$ to $\varepsilon/2$ (line~\ref{algline:stop} in \Cref{alg:PAC-learning-CSG}), and the constant in $n_{\min}$ improves accordingly:
\[
n_{\min} = \nminExprFn{4} = \nminBigOtilde.
\]
All other bounds remain unchanged. 
Moreover, the RCSG solver $\SolveRCSG$ reduces to a minimax $L^1$-CSG solver, which is equivalent to solving an RMDP and admits a polynomial-time solution~\citep{panaganti2022robustRL,nilim2005robust,wiesemann2013robust,suilen2024robust}.
% For zero-sum reward objectives, the only restriction is the following:
% % for infinite-horizon reward properties on CSGs:
% \begin{assumption}[Stopping game]
% \label{ass:stopping}
%     From any state $s$ where $r_S(s)<0$ or $r_A(s,a)<0$ for some action $a$, under all profiles of $\csg$, with probability 1 we reach either a target state in $S_T$, or a zero-reward state that cannot be left with probability 1 under all profiles.
% \end{assumption}

\section{Helper functions for Algorithm~\ref{alg:PAC-learning-CSG}}
\label{app:helper-algo}
The helper routines, \Cref{alg:initialise,alg:compute-Punc,alg:update}, are called by \Cref{alg:PAC-learning-CSG} but omitted from the main paper for brevity.
Note that in \Cref{alg:initialise}, $\reach{p}$ is computed by solving a reachability problem on an RMDP with the same support as $P^\star$, full-simplex transition uncertainty over that support, and $S_T$ as the terminal set; this does not require knowing the transition probabilities. 
In implementation, we instantiate the support-preserving simplex with a $10^{-5}$ lower bound on probabilities within the support, which ensures graph preservation and yields a strictly positive $\reach{p}$ consistent with \Cref{ass:reachability}.

\begin{algorithm}[h]
\caption{Initialise}
\label{alg:initialise}
\begin{algorithmic}[1]
\Require $S, A, \Gamma, \Rmax, H, \supp(P^\star)$
\State $\reach{p} \gets \min_{(s,a)\in \ntslots}
  \ \max_{\sigma\in\Sigma} \ 
  \inf_{P\in \Punc_{\supp}}
  \prob^{\sigma,P}[\text{reach $(s,a)$ before entering $S_T$}]$ 
\Comment{\Cref{eq:p-reach}}
\For{$(s,a)\in \ntslots$}
    \State $N(s,a)\gets 0, \ N(s,a,s')\gets 0\ \forall s'$
    \State $\known(s,a)\gets false$
    \State $\widehat{P}(\cdot\mid s,a)\gets \text{Uniform}$
\EndFor
\State $c \gets -\nminCoef\Rmax^2 H^4/\varepsilon^2; 
\quad n_{\min} \gets 
\left\lceil
c \cdot W_{-1}\!\left(
\sqrt{ \confCI / (2(2^{|S|}-2)\nTrans)} \big/ c
\right)\right\rceil$ 
\Comment{by Cor.~\ref{cor:nmin}}
\State \Return $N, \known, \widehat{P}, \reach{p}$
\end{algorithmic}
\end{algorithm}

\begin{algorithm}[h]
\caption{Construct uncertain transition set $\Punc(s,a)$}
\label{alg:compute-Punc}
\begin{algorithmic}[1]
\Require $N, \confCI, S, A, \Gamma, \supp(P^\star)$
\For{$(s,a)\in \ntslots$}
    \If{$N(s,a)=0$}
        \State $\widehat{P}_{sa} \gets \text{Uniform}(\Succ(s,a))$
        \State $\Punc(s,a) \gets \Delta_+(\Succ(s,a))$
        % \Comment{enforce graph preservation with minimum mass $\epsilon>0$}
        % \State \Continue
    \Else
    \State $\widehat{P}_{sa}(s') \gets N(s,a,s')/N(s,a) \quad \forall{s'\in \Succ(s,a)}$
    \State $\confCI_{sa} \gets \confCI/\left[{\nTrans \cdot N(s,a)(N(s,a)+1)} \right]$
    \State $\alpha \gets \sqrt{(2/N(s,a)) \ln{[(2^{|S|}-2)/\confCI_{sa}]}}$
    \Comment{by Lem.~\ref{lem:per-slot-concentration}}
    % $\alpha_t(s,a,s') \gets  \sqrt{\frac{1}{2N_t(s,a)}\,
    % \ln \left(\frac{2}{\confCI_{sa}}\right)}$
    % 
    % \State $\underline{P}(s'|s,a) \gets \max\{0,\ \widehat P(s'|s,a)-\alpha_t(s,a)\}$
    % \State $\overline{P}(s'|s,a) \gets \min\{1,\ \widehat P(s'|s,a)+\alpha_t(s,a)\}$
    \State $\Punc(s,a) \gets \big\{ p\in\Delta_+(\Succ(s,a)):\ \| p -\widehat{P}_{sa}\|_1 \le \alpha \big\}$
    \EndIf
\EndFor
\State \Return $\Punc$
\end{algorithmic}
\end{algorithm}

\begin{algorithm}[h]
\caption{Update}
\label{alg:update}
\begin{algorithmic}[1]
\Require $N, \known, \pi=(\pi_1,\ldots,\pi_{\numSamples}), n_{\min}, S, A, \Gamma$
\For{$(s,a)\in \ntslots$}
    \For{$i\gets 1$ \textbf{to} $\numSamples$}
        \State $N(s,a)\gets N(s,a) + \#(s,a) \text{ in } \pi_i$
        \For{$s'\in S$}
            \State $N(s,a,s') \gets N(s,a,s') + \#(s,a,s') \text{ in } \pi_i$
        \EndFor
    \EndFor
    \State $\known(s,a)\gets (N(s,a)\ge n_{\min})$
\EndFor
\State \Return $N, \known$
\end{algorithmic}
\end{algorithm}

\section{Additional experimental details}
\label{app:additional-experimental-results}
\ifICLR
\subsection{Benchmark descriptions}
\label{app:benchmark-descriptions}

The six benchmarks introduced in \Cref{sec:benchmarks} are each designed to isolate a distinct challenge:
\begin{enumerate*}[1)]
\item \emph{Cyclic Preferences} has no exact stationary NE for infinite-horizon reachability reward, testing detection of non-existence;
\item \emph{Delayed Coordination} isolates hard exploration under sparse, delayed rewards;
\item \emph{Hide-or-Run} adapted from \citet{de2007concurrent} has no stationary NE profiles; its value is only achieved in the limit via mixed strategies converging to ``hide w.p.\ $1$'' for the zero-sum objective of player~1 eventually reaching $s_1$; 
\item \emph{Mixed NE} requires a mixed equilibrium;
\item \emph{Safe vs.\ Risky} contrasts actions with differing sensitivity to transition probabilities; and
\item \emph{Traffic Merge} evaluates safety-critical multi-step coordination.
\end{enumerate*}%

\fi

\begin{figure}[h]
\centering

\begin{subfigure}{0.33\textwidth}
\centering
\begin{tikzpicture}[
    scale=0.8,
    transform shape,
    state/.style={draw, circle, minimum size=0.6cm},
    dist/.style={draw, circle, fill=black, inner sep=1.2pt},
    >=stealth
]

% States
\node[state] (s0) at (0,0) {$s_0$};
\node[above=1pt of s0] {\scriptsize $r_1=1$};

\node[state] (s1) at (4,0) {$s_1$};
\node[above=1pt of s1] {\scriptsize $r_1=1$};

\node[state] (s2) at (0,-3) {$s_2$};
\node[below=1pt of s2] {\scriptsize $r_2=1$};

\node[state, draw=red] (s3) at (4,-3) {$s_3$};

% --- STATE s0 transitions ---
\node[dist] (d1) at (1.5,-1.2) {};

\draw[->] (s0) -- (d1) node[pos=0.7, left, xshift=-2pt] {\scriptsize $[a,a]$};
\draw[->, bend right=20] (s0) to node[midway, left] {\scriptsize $[b,a]$} (s2);
\draw[->, bend left=20] (s0) to node[midway, above] {\scriptsize $[a,b], [b,b]$} (s1);

\draw[->, bend right=25] (d1) to node[midway, right] {\scriptsize $0.9$} (s0);
\draw[->, bend right=10] (d1) to node[midway, right] {\scriptsize $0.1$} (s3);

% --- STATE s1 transitions ---
\node[dist] (d2) at (4,-1.5) {};

\draw[->] (s1) -- (s2) node[midway, left] {}; % also [a,a], collapsed with the transition (s2) -- (s1)
\draw[->] (s1) -- (s0) node[midway, below] {\scriptsize $[b,a]$};
\draw[->] (s1) -- (d2) node[midway, left] {\scriptsize $[b,b]$};

\draw[->, bend right=20] (d2) to node[midway, right] {\scriptsize $0.9$} (s1);
\draw[->, bend left=10] (d2) to node[midway, right] {\scriptsize $0.1$} (s3);

% --- self-loop ---
\draw[->] (s1) edge[loop right, looseness=6] node {\scriptsize $[a,b]$} (s1);

% --- STATE s2 transitions ---
\node[dist] (d3) at (2.5,-3) {};

\draw[->] (s2) -- (s1) node[pos=0.7, left] {\scriptsize $[a,a]$};
\draw[->] (s2) -- (d3) node[midway, below] {\scriptsize $[a,b], [b,a]$};
\draw[->] (s2) -- (s0) node[pos=0.4, right] {\scriptsize $[b,b]$};

\draw[->, bend right=20] (d3) to node[midway, above] {\scriptsize $0.9$} (s2);
\draw[->, bend right=10] (d3) to node[midway, below] {\scriptsize $0.1$} (s3);

% --- terminal self-loop ---
\draw[->] (s3) edge[loop below, looseness=6] (s3);

\end{tikzpicture}
\caption{Cyclic Preferences}
\end{subfigure}
\hfill
\begin{subfigure}{0.33\textwidth}
\centering
\begin{tikzpicture}[
    scale=0.85,
    transform shape,
    state/.style={draw, circle, minimum size=0.65cm},
    dist/.style={draw, circle, fill=black, inner sep=1.3pt},
    >=stealth
]

% States
\node[state] (s0) at (0,0.5) {$s_0$};
\node[state] (s1) at (3,0.5) {$s_1$};
\node[state] (s2) at (1.5,-2) {$s_2$};

\node[state, draw=green] (s3) at (3.5,-2) {$s_3$};
\node[below=1pt of s3] {\small $\substack{r_1=1, \\ r_2=0.8}$};

\node[state] (s4) at (0,-2) {$s_4$};

% --- actions from s0 ---
\draw[->] (s0) to
    node[pos=0.55, above] {\scriptsize $[c,c]$}
    (s1);

\draw[->] (s0) to[bend right=25]
    node[pos=0.5, left] {\scriptsize others}
    (s4);

% --- actions from s1 ---
\node[dist] (d1) at (3,-0.7) {};

\draw[->] (s1) -- (d1)
    node[pos=0.4, right] {\scriptsize $[c,c]$};

\draw[->] (s1) to[bend right=15]
    node[pos=0.4, above left, xshift=3pt] {\scriptsize others}
    (s4);

\draw[->, bend left=12] (d1) to
    node[pos=0.5, left] {\scriptsize $0.9$}
    (s2);

\draw[->, bend right=12] (d1) to
    node[pos=0.55, above] {\scriptsize $0.1$}
    (s4);

% --- actions from s2 ---
\draw[->] (s2) to
    node[pos=0.5, below] {\scriptsize $[c,c]$}
    (s3);

\draw[->] (s2) to
    node[pos=0.5, below] {\scriptsize others}
    (s4);

% --- trap dynamics ---
\draw[->, bend right=20] (s4) to
    node[pos=0.55, left, xshift=2pt] {\scriptsize $0.8$}
    (s0);

\draw[->, loop below, looseness=6] (s4)
    edge node[pos=0.5, left] {\scriptsize $0.2$}
    (s4);

% --- goal absorbing ---
\draw[->] (s3) edge[loop above, looseness=6] (s3);

\node[draw, rounded corners, align=left, fill=white, inner sep=4pt]
    at (1.5,1.4) {\scriptsize
    others $= [c,d], [d,c], [d,d]$
};

\end{tikzpicture}

% others $= [c,d], [d,c], [d,d]$
\vspace*{.5em}
\caption{Delayed Coordination}
\label{fig:delayed-coord}
\end{subfigure}
\hfill
\begin{subfigure}{0.31\textwidth}
\centering
\begin{tikzpicture}[
    scale=0.85,
    transform shape,
    state/.style={draw, circle, minimum size=0.65cm},
    >=stealth
]

% States
\node[state] (s0) at (0,0) {$s_0$};
\node[state, draw=green] (s1) at (3,1.5) {$s_1$};
\node[state] (s2) at (3,0) {$s_2$};
\node[state, draw=red] (s3) at (3,-1.5) {$s_3$};

% --- actions from s0 ---
\draw[->] (s0) edge[loop above, looseness=6]
    node[above] {\scriptsize $[hide,wait]$}
    (s0);

\draw[->, bend left=10] (s0) to
    node[pos=0.8, above left, xshift=3pt] {\scriptsize $[run,wait]$}
    (s1);
    
\draw[->] (s0) to
    node[pos=0.55, above] {\scriptsize $[hide,throw]$}
    (s2);

\draw[->, bend right=10] (s0) to
    node[pos=0.6, below left, xshift=3pt] {\scriptsize $[run,throw]$}
    (s3);

% --- absorbing / dynamics ---
\draw[->] (s1) edge[loop right, looseness=6] (s1);

\draw[->] (s2) -- (s1);

\draw[->] (s3) edge[loop right, looseness=6] (s3);

\end{tikzpicture}
\vspace*{1em}
\caption{Hide-or-run}
\end{subfigure}

\vspace{0.5cm}

\begin{subfigure}{0.33\textwidth}
\centering
\begin{tikzpicture}[
    scale=0.85,
    transform shape,
    state/.style={draw, circle, minimum size=0.65cm},
    dist/.style={draw, circle, fill=black, inner sep=1.3pt},
    >=stealth
]

% States
\node[state] (s0) at (0,0) {$s_0$};
\node[state, draw=green] (s1) at (3,0) {$s_1$};
\node[state, draw=green] (s2) at (3,-3) {$s_2$};
\node[state, draw=red] (s3) at (0,-3) {$s_3$};

% --- joint action nodes ---
\node[dist] (d1) at (3,-1.5) {};

% --- actions from conflict state ---
\draw[->] (s0) -- (d1) node[pos=0.3, above right, xshift=-2pt, yshift=-3pt] {\scriptsize $[push,push]$};

\draw[->, bend right=20] (s0) to node[pos=0.6, left] {\scriptsize $[wait,wait]$} (s3);

\draw[->] (s0) -- node[pos=0.6, above] {\scriptsize $[push,wait]$} (s1);
\draw[->, bend left=0] (s0) to node[pos=0.75, left] {\scriptsize $[wait,push]$} (s2);

% --- stochastic outcome (only for push,push) ---
\draw[->, bend right=12] (d1) to node[pos=0.55, right] {\scriptsize $0.5$} (s1);

\draw[->, bend left=12] (d1) to node[pos=0.55, right] {\scriptsize $0.5$} (s2);

% --- reset loop ---
\draw[->] (s3) edge[loop left, looseness=6] node {} (s3);

% --- absorbing states ---
\draw[->] (s1) edge[loop above, looseness=6] (s1);
\draw[->] (s2) edge[loop left, looseness=6] (s2);

% optional return (if you want explicit reset)
\draw[->] (s3) -- (s0);

\end{tikzpicture}
\caption{Mixed NE}
\end{subfigure}
\hfill
\begin{subfigure}{0.33\textwidth}
\centering
\begin{tikzpicture}[
    scale=0.8,
    transform shape,
    state/.style={draw, circle, minimum size=0.5cm},
    dist/.style={draw, circle, fill=black, inner sep=1.2pt},
    >=stealth
]

% States (compact square)
\node[state] (s0) at (0,0) {$s_0$};
\node[state, draw=green] (s1) at (3,0) {$s_1$};
\node[above=1pt of s1] {\scriptsize $r=1$};

\node[state, draw=green] (s2) at (0,-3) {$s_2$};
\node[below=1pt of s2] {\scriptsize $r=0.6$};

\node[state, draw=red] (s3) at (3,-3) {$s_3$};

% Distribution nodes
\node[dist] (d1) at (1.5,0.5) {};
\node[dist] (d2) at (0,-1.3) {};
\node[dist] (d3) at (3,-1.5) {};
\node[dist] (d4) at (1.2,-2.2) {};

% --- Actions from s0 ---
\draw[->] (s0) -- (d1) node[midway, above] {\scriptsize $[a,a]$};
\draw[->] (s0) -- (d2) node[midway, left]  {\scriptsize $[b,b]$};
\draw[->] (s0) -- (d3) node[pos=0.25, above] {\scriptsize $[a,b]$};
\draw[->] (s0) -- (d4) node[pos=0.7, left] {\scriptsize $[b,a]$};

% --- Probabilistic branching ---
\draw[->, bend left=12] (d1) to node[pos=0.55, above] {\scriptsize $0.9$} (s1);
\draw[->, bend right=18] (d1) to node[pos=0.7, right] {\scriptsize $0.1$} (s3);

\draw[->, bend left=10] (d2) to node[pos=0.75, above] {\scriptsize $0.6$} (s1);
\draw[->, bend right=10] (d2) to node[pos=0.55, left] {\scriptsize $0.4$} (s2);

\draw[->, bend right=12] (d3) to node[pos=0.55, above right] {\scriptsize $0.5$} (s1);
\draw[->, bend left=18] (d3) to node[pos=0.6, below right] {\scriptsize $0.5$} (s3);

\draw[->, bend left=14] (d4) to node[pos=0.15, above, xshift=-3pt] {\scriptsize $0.5$} (s1);
\draw[->, bend right=14] (d4) to node[pos=0.5, left] {\scriptsize $0.5$} (s3);

% Self-loops
\draw[->] (s1) edge[loop right, looseness=6] (s1);
\draw[->] (s2) edge[loop left, looseness=6]  (s2);
\draw[->] (s3) edge[loop right, looseness=6] (s3);

\end{tikzpicture}
\caption{Safe vs.\ Risky}
\label{fig:safe-vs-risky}
\end{subfigure}
\hfill
\begin{subfigure}{0.31\textwidth}
\centering
\begin{tikzpicture}[
    scale=0.8,
    transform shape,
    state/.style={draw, circle, minimum size=0.55cm},
    dist/.style={draw, circle, fill=black, inner sep=1.2pt},
    >=stealth
]

% States
\node[state] (si) at (0,0) {$s_{i\in [3]}$};
\node[above=1pt of si] {\scriptsize $r=1$};

\node[state] (sip) at (3,-1.5) {$s_{i+1}$};
\node[state, draw=green] (goal) at (3,-3) {$s_4$};
\node[state, draw=red] (crash) at (0,-3) {$s_5$};

% Distribution nodes
\node[dist] (d1) at (0,-1.5) {};     

% Actions
\draw[->] (si) -- (sip) node[pos=0.1, right, xshift=-2pt] {\scriptsize $\substack{[wait,wait], [wait,go],\\ [go,wait]}$};
\draw[->, dashed] (sip) to[out=30, in=50] node[midway, above right, yshift=-1pt] {\scriptsize if $i<4$} (si);
\draw[->] (si) -- (d1) node[midway, left] {\scriptsize $[go,go]$};

% Outcomes
\draw[->, bend left=10] (d1) to node[pos=0.5, above] {\scriptsize $0.7$} (sip);
\draw[->, bend right=15] (d1) to node[pos=0.6, left] {\scriptsize $0.3$} (crash);

% Final step indication
\draw[->, dashed] (sip) -- node[midway, right] {\scriptsize if $i{+}1=4$} (goal);

% Self-loops
\draw[->] (goal) edge[loop right, looseness=6] (goal);
\draw[->] (crash) edge[loop left, looseness=6] (crash);

\end{tikzpicture}
\caption{Traffic Merge}
\end{subfigure}

\caption{Six benchmark CSGs ($\bar{s}=s_0$) illustrating distinct challenges for robust PAC learning, including coordination, mixed-strategy equilibria, equilibrium non-existence, and sparse-reward exploration (see \Cref{sec:benchmarks}). Green states denote target states of the players, while red states denote absorbing trap states.}% \caption{Benchmark suite. Each game highlights a distinct challenge in robust PAC learning of CSGs. $s_0$ is the initial state in all.}
\label{fig:benchmarks}
\end{figure}

% \angel{add reward structure labels in bold}

\subsection{Exploration Baselines}
\label{app:exploration-baselines}

We give precise definitions of the four exploration strategies compared in RQ2 (\Cref{tab:explorer-cmp}). In each case, the resulting sampling policy replaces $\explore{\sigma}_t$ in \Cref{alg:PAC-learning-CSG}; only the choice of exploration policy differs, with the PAC loop, stopping condition $\Delta_t \le \varepsilon/4$, and all other components held fixed.

\begin{enumerate}[(i)]
    \item \textbf{Robust (ours).} Solves the exploration RMDP $\explore{\csg}_t$ (\Cref{def:exploration-RMDP}) pessimistically, i.e., $\explore{\sigma}_t \in \arg\max_{\sigma\in \Sigma} \inf_{P\in \explore{\Punc}_t}{\ev^{\explore{\sigma},P}[\cdot]}$, maximising worst-case reachability of unknown $(s,a)$-pairs over the current uncertainty set (see \Cref{eq:exploration-profile}). This is the exploration rule used by \Cref{alg:PAC-learning-CSG} and underlies our sample-complexity guarantee (\Cref{thm:PAC-CSG}).
    \item \textbf{Optimistic.} Solves the same exploration RMDP $\explore{\csg}_t$, but with the inner minimisation over $\explore{\Punc}_t$ replaced by a maximisation, i.e.\ $\explore{\sigma}_t \in \arg\max_{\sigma\in \Sigma} \sup_{P\in \explore{\Punc}_t}{\ev^{\explore{\sigma},P}[\cdot]}$, targeting best-case rather than worst-case reachability. This corresponds to the optimism-in-the-face-of-uncertainty principle common in PAC-MDP exploration, but is not covered by our analysis (\Cref{sec:exploration}): an optimistic $\explore{\sigma}_t$ need not lower-bound the true visitation probability under $P^\star$, so the coverage guarantee of \Cref{lem:exploration-prob} need not hold.
    \item \textbf{Round-robin.} Maintains a single target slot: the unknown $(s,a)$-pair with the fewest samples so far, ties broken by index. Rather than solving the exploration RMDP, it computes a policy on the \emph{point-estimate} MDP $\widehat{P}_t$ (i.e., the empirical kernel treated as exact, with zero confidence radius) that maximises the probability of reaching the target slot within the horizon, and re-solves only when the target slot changes (i.e., becomes known).
    \item \textbf{Uniform-random.} No exploration policy is solved; at each state, each player samples independently and uniformly at random from their available actions.
\end{enumerate}

All four share the same per-episode sample count $\numSamples_t$ (\Cref{prop:num-samples-per-episode}) and stopping condition; only the trajectories collected within each episode differ.

\subsection{Learning Dynamics}
\label{app:RQ-learning-dynamics}

We complement the main experimental evaluation (\Cref{sec:experiments}) with a finer-grained analysis of the learning dynamics induced by our exploration procedure. We ask: \emph{how closely does the model-error proxy $\Delta_t$ reflect actual model and decision quality during learning}, and \emph{how does the uncertainty and sampling effort evolve over episodes?}

\begin{figure}[H]
    \centering
    \begin{subfigure}{0.34\textwidth}
        \centering
        \includegraphics[width=\linewidth]{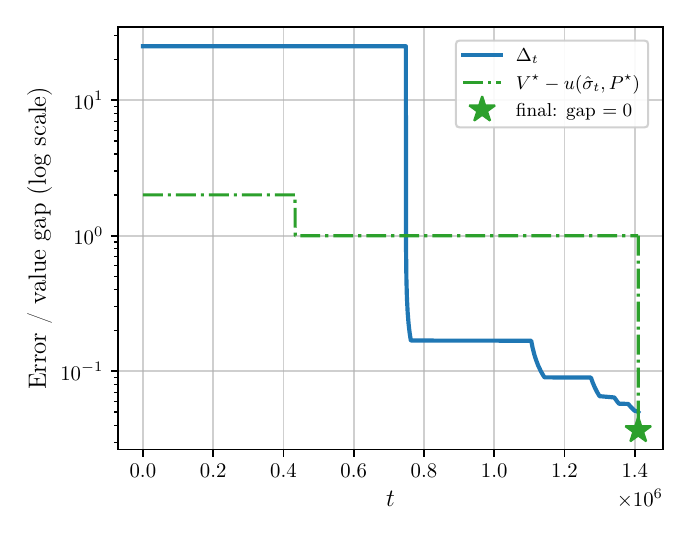}
        \caption{$\Delta_t$ and true value gap}
        \label{fig:delayed-coord1-learning}
    \end{subfigure}
    \hfill
    \begin{subfigure}{0.345\textwidth}
        \centering
        \includegraphics[width=\linewidth]{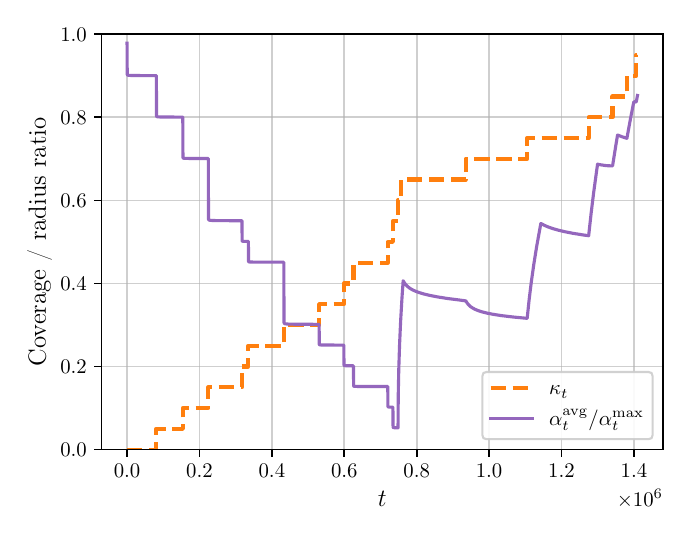}
        \caption{$\knownFrac_t$ and radius ratio}
        \label{fig:radius-ratio-delayed-coord1}
    \end{subfigure}
    \hfill
    % \hspace{2em}
    \begin{subfigure}{0.267\textwidth}
        \centering
        \includegraphics[width=\linewidth]{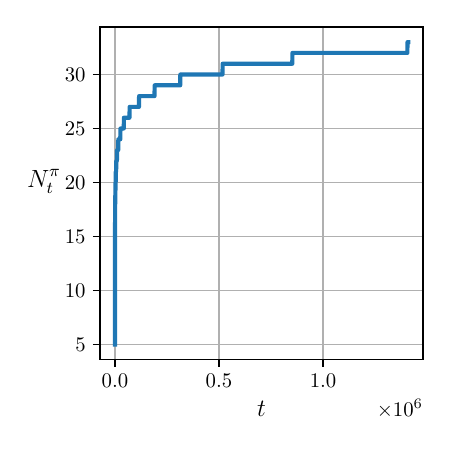}
        \caption{Samples per episode}
        \label{fig:samples-delayed-coord1}
    \end{subfigure}
    \caption{Learning dynamics over episodes in a representative run of the \emph{Delayed Coordination} game for $\llangle p_1:p_2 \rrangle_{\max=?}
    \big(\probop{}[\eventually s_3] + \probop{}[\eventually s_3]\big)$:
    (a) evolution of $\Delta_t$ and the true value gap $V^\star-u(\hat{\sigma}_t, P^\star)$, with $\hat{\sigma}_t$ re-computed every 200 episodes and the final value gap at termination marked by a star;
    (b) evolution of $\knownFrac_t$ and the average-to-maximum confidence-radius ratio; and
    (c) number of sampled trajectories per episode.}
    \label{fig:delayed-coord1-fine}
\end{figure}

We measure:
\begin{enumerate*}[(i)]
    \item the model error proxy $\Delta_t := \frac{1}{2} H^2 \max_{(s,a)\in\ntslots}\|\widehat{P}_{t,sa} - P^\star_{sa}\|_1$ ($\Rmax=1$);
    \item the fraction $\knownFrac_t := \lvert \known_t \rvert / \lvert \unknown_1 \rvert$ of known $(s,a)$ pairs by episode $t$; and
    \item the true value gap $V^\star - u(\hat{\sigma}_t, P^\star)$, where $\hat{\sigma}_t$ is obtained by periodically solving the empirical game every 200 episodes.
\end{enumerate*}
Together, these quantities allow us to assess how conservative the sufficient stopping condition $\Delta_t\le\varepsilon/4$ is in practice.
We illustrate these dynamics on a representative run of the sparse-reward exploration benchmark, \emph{Delayed Coordination}, for the infinite-horizon property $\llangle p_1:p_2 \rrangle_{\max=?}(\probop{}[\eventually s_3] + \probop{}[\eventually s_3])$: the evolution of $\Delta_t$ and the true value gap is shown in \Cref{fig:delayed-coord1-learning}, the evolution of $\knownFrac_t$ alongside the distribution of confidence radii in \Cref{fig:radius-ratio-delayed-coord1}, and the number of sampled trajectories per episode in \Cref{fig:samples-delayed-coord1}.

\Cref{fig:delayed-coord1-learning} shows that convergence is not smooth or gradual, especially for the model error $\Delta_t$ and the true value gap $V^\star - u(\hat{\sigma}_t, P^\star)$. Rather than decreasing smoothly, $\Delta_t$ stays pinned at its worst-case value while a single hard-to-reach $(s,a)$-pair remains almost entirely unsampled. Since $\Delta_t$ is governed by the \emph{largest} per-slot confidence radius rather than the average (\Cref{lem:sensitivity}), it drops sharply once that slot accumulates enough samples, then follows a staircase pattern as the remaining slots are resolved. 
% Each drop reflects the shrinking confidence radius of the newly-known slot.
%
Convergence of the extracted strategy is similarly abrupt: $\hat{\sigma}_t$ achieves value $0$ (versus the true value $V^\star=2$) for roughly the first third of training, then rises to and remains at exactly $1$ --- the value attained by the non-robust point-estimate strategy on most seeds --- until the last periodic checkpoint (200 iterations) before $\Delta_t$ crosses $\varepsilon/4$. Only in the final handful of episodes does $\hat{\sigma}_t$ resolve to the fully cooperative equilibrium, which \emph{Delayed Coordination} requires at every one of the three consecutive states ($s_0,s_1,s_2$) to reach the goal $s_3$ at all (see \Cref{fig:delayed-coord}).
Note that this does not violate \Cref{cor:anytime-value-gap}'s guarantee $V^\star-u(\hat{\sigma}_t,P^\star)\le 4\Delta_t$: the bound requires $\mu^\star \ge 4\Delta_t-\varepsilon$, and here $\mu^\star=0$, since all four joint actions at the indifferent state $s_4$ induce the same transition distribution, so deviating there costs either player nothing and attains the infimum in \Cref{def:Nash-margin} at $0$. Thus, the precondition reduces exactly to $\Delta_t\le\varepsilon/4$, the algorithm's own stopping condition, so the bound is not expected to hold before the run actually stops, i.e., the regime shown here.

\Cref{fig:radius-ratio-delayed-coord1} shows that $\kappa_t$ increases step-wise but much more smoothly than $\Delta_t$, with no comparably large jump around episode 700--800K, suggesting that the other $(s,a)$-pairs become known relatively steadily throughout training. 
$\kappa_t$ stops at $0.95$ ($19$ of the $20$ $(s,a)$-pairs), indicating that termination is governed by the statistical stopping condition $\Delta_t \le \varepsilon/4$, which can occur before full coverage (\Cref{cor:Delta-stopping-condition-correctness}). 
In contrast, the average-to-maximum confidence-radius ratio $\alpha^{\mathrm{avg}}_t/\alpha^{\max}_t$ decreases for roughly the first half of training as most slots are resolved while the single hardest-to-reach slot keeps $\alpha^{\max}_t$ near its initial value, then rises back towards $1$ once that slot is sampled and the remaining radii converge. 
Together, these trends imply that learning is bottlenecked by a single difficult-to-sample slot: its large confidence radius dominates $\Delta_t$, which reflects the worst-case rather than the average uncertainty across $(s,a)$-pairs (\Cref{lem:sensitivity}).
% This behaviour is consistent with the $1/\reach{p}$ dependence in the sample complexity bound (\Cref{thm:PAC-CSG}), and implies that learning is bottlenecked by hard-to-reach regions.
%

On the other hand, \Cref{fig:samples-delayed-coord1} shows that the number of trajectories sampled per episode, $\numSamples_t$, grows logarithmically with the episode index $t$, matching the adaptive sampling rule $N^\omega_t \propto \log(1/\confCount_t)=\bigO{\log t}$ from \Cref{prop:num-samples-per-episode} and \Cref{app:exploration}.
This reflects the algorithm's increasingly conservative per-episode coverage requirement. Specifically, as the confidence budget is distributed across an increasing number of episodes, the allowed per-episode failure probability $\confCount_t$ decreases as $\confCount_t = \confCount/(t(t+1))$ (\Cref{prop:coverage-confidence-allocation}), so later episodes require more trajectories to maintain their high-probability coverage guarantee, while keeping the total failure probability from coverage at most $\confCount$.

Overall, these results show that learning is primarily bottlenecked by a small number of hard-to-explore $(s,a)$ pairs: $\Delta_t$ and strategy quality improve only once these pairs receive sufficient samples, even when most pairs are already well resolved. The stopping condition $\Delta_t\le\varepsilon/4$ is therefore conservative in the sense that it is driven by the single worst-case confidence radius rather than decision-relevant uncertainty. To resolve bottlenecks earlier, future improvements could target high-impact, poorly explored regions more directly, e.g., by designing rewards proportional to confidence radii and/or inversely proportional to visitation probabilities.

\ifICLR
\subsection{Exploration strategy comparison}
\label{app:explorer-cmp-table}

\Cref{tab:explorer-cmp} reports sample consumption and value gap for each exploration rule under an otherwise identical PAC loop, as discussed under RQ2 in \Cref{sec:discussion}.

\fi

\ifICLR
\subsection{Sample-complexity scaling}
\label{app:scaling-plots}

\Cref{fig:scaling-plots} reports the empirical scaling of total sample consumption with each problem parameter, as discussed under \hyperref[RQ3-sample-complexity]{RQ3} in \Cref{sec:discussion}.
\begin{figure}[!t]
\centering

\begin{subfigure}{0.245\textwidth}
    \centering
    \includegraphics[width=\linewidth]{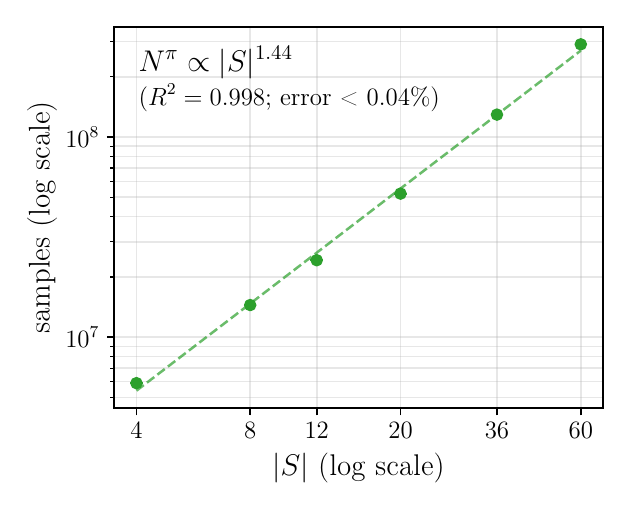}
    \caption{$|S|$-scaling}
    \label{fig:S-scaling}
\end{subfigure}
\hfill
\begin{subfigure}{0.245\textwidth}
    \centering
    \includegraphics[width=\linewidth]{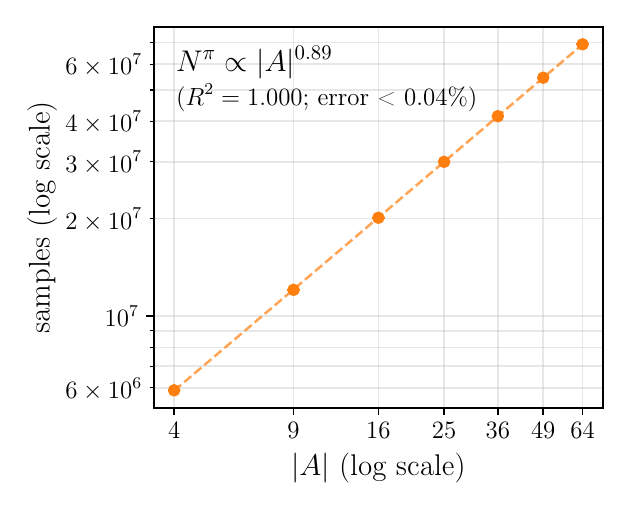}
    \caption{$|A|$-scaling}
    \label{fig:A-scaling}
\end{subfigure}
\hfill
\begin{subfigure}{0.245\textwidth}
    \centering
    \includegraphics[width=\linewidth]{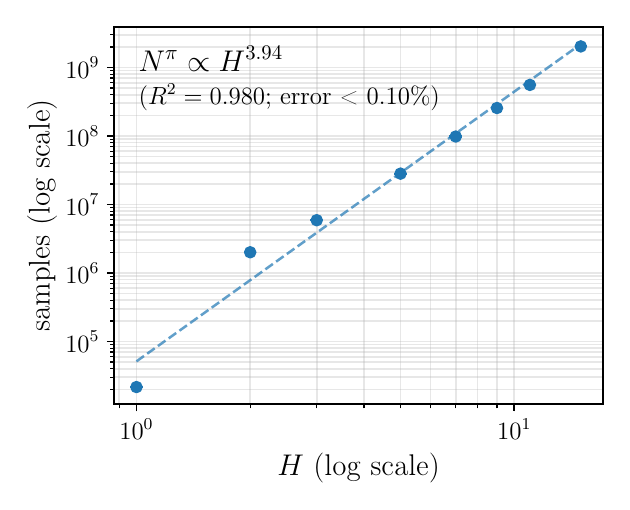}
    \caption{$H$-scaling}
    \label{fig:H-scaling}
\end{subfigure}
\hfill
\begin{subfigure}{0.245\textwidth}
    \centering
    \includegraphics[width=\linewidth]{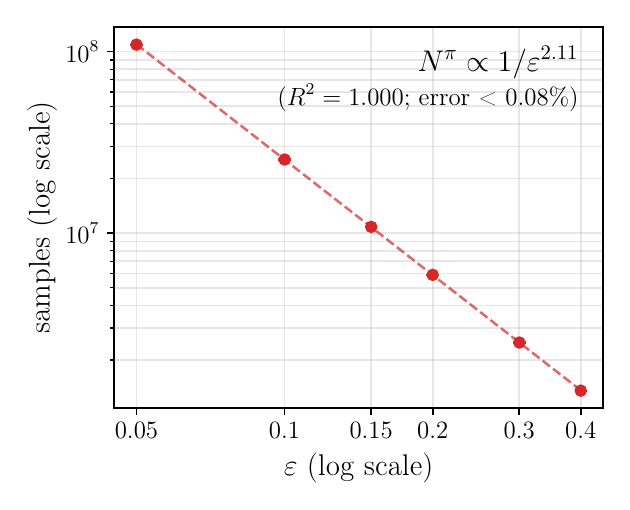}
    \caption{$\varepsilon$-scaling}
    \label{fig:eps-scaling}
\end{subfigure}

\caption{
Scaling of sample complexity (log--log scale) in the \textit{Safe vs.\ Risky} game, following the setup in \hyperref[RQ3-sample-complexity]{RQ3}, with respect to (a) $|S|$, (b) $|A|$, (c) $H$, and (d) $\varepsilon$. Error bars show standard deviation over 10 runs; all are plotted but negligible ($<0.10\%$ of the mean) and thus not visible.
}
\label{fig:scaling-plots}
\vspace*{-1.5em}
\end{figure}

\fi

\subsection{Full experimental results}
\label{app:full-experimental-results}

We present the complete set of experimental results in \Cref{tab:full-results}, in response to RQ1 outlined in \Cref{sec:experiments} of the main paper.

\begin{table}[t]
\centering
\caption{Full experimental results (mean $\pm$ std over 10 runs). % for all benchmarks, each with an infinite-horizon and finite-horizon property.
All cases use $\varepsilon=0.1$, except \textit{Delayed Coordination}'s unbounded-reachability property, which uses $\varepsilon=0.2$ to keep runtime within a reasonable budget.
For every case with a stationary NE returned by the CSG solver, our algorithm achieves the same true-game value, with differences only on action-indifferent states (only $s_4$ in \textit{Delayed Coordination}). % ; thus $V^\star - u(\hat{\sigma},P^\star)=0$ in all cases. 
\emph{Hide-or-Run} is a limit-equilibrium case in which the optimal value is approached in the limit but is not attained by any stationary profile.
For the third \textit{Safe vs.\ Risky} property, our learner outputs $\hat{V}=\infty$ in all runs.}
\label{tab:full-results}
\resizebox{\textwidth}{!}{%
\begin{tabular}{@{}l l r r r r r@{}}
\toprule
\textbf{Case study} & \textbf{Property} & 
% $n_{\min}\ (\times10^{7})$ & 
% \textbf{Runtime (s)} &
\textbf{No. of episodes} $(\times 10^6)$ & \textbf{Total no. of samples} $(\times10^7)$ & $V^\star$ & $V^\star - \hat{V}$ \\
\midrule

\multirow{2}{*}{Cyclic Prefs.}
& $\llangle p_1:p_2 \rrangle_{\max=?}
\left( \rewop{r_1}[\eventually s_3] + \rewop{r_2}[\eventually s_3] \right)$
% & $0.866$
% & $(3.2057 \pm 1.4606)\times10^{3}$
& $2.6237 \pm 0.0003$
& $8.3317 \pm 0.0010$
& -- & -- \\

& $\llangle p_1:p_2 \rrangle_{\max=?}
\left( \probop{}[\eventually^{\le 3} s_3] + \probop{}[\eventually^{\le 4} s_3] \right)$
% & $0.866$
% & $(2.1407 \pm 0.5375)\times10^{3}$
& $2.4997 \pm 0.0002$
& $7.9101 \pm 0.0007$
& $0.615$ & $0.0021 \pm 0.0001$ \\

\midrule

\multirow{2}{*}{Delayed Coord.}
& $\llangle p_1:p_2 \rrangle_{\max=?}
\left( \probop{}[\eventually s_3] + \probop{}[\eventually s_3] \right)$
% & $2.260$
% & $(9.0098 \pm 0.7746)\times10^{2}$
& $1.4098 \pm 0.0000$
& $4.2953 \pm 0.0002$
& $2.000$ & $0.0000 \pm 0.0000$ \\

& $\llangle p_1:p_2 \rrangle_{\max=?}
\left( \rewop{r_1}[\Ct^{\le 5}] + \rewop{r_2}[\Ct^{\le 5}] \right)$
% & $2.260$
% & $(4.1437 \pm 1.0305)\times10^{3}$
& $5.5241 \pm 0.0000$
& $18.3638 \pm 0.0002$
& $3.240$ & $0.0016 \pm 0.0001$ \\

\midrule

\multirow{2}{*}{Hide-or-run}
& $\llangle p_1 \rrangle \probop{\max=?}
\left[ \eventually s_1 \right]$
% & $0.202$
% & $(1.5691 \pm 0.0056)\times10^{2}$
& $0.4671 \pm 0.0000$
& $1.3216 \pm 0.0000$
& $0.999$ & $0.0000 \pm 0.0000$ \\

& $\llangle p_1 \rrangle \probop{\max=?}
\left[ \eventually^{\le 5} s_1 \right]$
% & $0.516$
% & $(4.2046 \pm 0.2274)\times10^{2}$
& $1.0625 \pm 0.0000$
& $3.1834 \pm 0.0000$
& $0.800$ & $0.0000 \pm 0.0000$ \\

\midrule

\multirow{2}{*}{Mixed NE}
& $\llangle p_1:p_2 \rrangle_{\max=?}
\left( \probop{}[\eventually s_1] + \probop{}[\eventually s_2] \right)$
% & $0.866$
% & $(4.8315 \pm 0.3076)\times10^{2}$
& $1.4136 \pm 0.0000$
& $4.3077 \pm 0.0000$
& $1.000$ & $0.0031 \pm 0.0000$ \\

& $\llangle p_1:p_2 \rrangle_{\max=?}
\left( \probop{}[\eventually^{\le 4} s_1] + \probop{}[\eventually^{\le 2} s_2] \right)$
% & $0.866$
% & $(1.1519 \pm 0.1298)\times10^{3}$
& $1.4136 \pm 0.0000$
& $4.3077 \pm 0.0000$
& $1.000$ & $0.0031 \pm 0.0000$ \\

\midrule

\multirow{2}{*}{Safe vs.\ Risky}
& $\llangle p_1:p_2 \rrangle_{\max=?}
\left( \probop{}[\eventually s_1] + \probop{}[\eventually s_2] \right)$
% & $0.866$
% & $(8.3992 \pm 0.8911)\times10^{2}$
& $0.3321 \pm 0.0411$
& $5.8710 \pm 0.0007$
& $1.000$ & $0.0017 \pm 0.0000$ \\

& $\llangle p_1:p_2 \rrangle_{\max=?}
\left( \probop{}[\eventually^{\le 3} s_1] + \probop{}[\eventually^{\le 3} s_2] \right)$
% & $0.258$
% & $(5.8348 \pm 0.5452)\times10^{2}$
& $0.1711 \pm 0.0066$
& $2.5460 \pm 0.0008$
& $1.000$ & $0.0026 \pm 0.0000$ \\

& $\llangle p_1:p_2 \rrangle_{\max=?}
\left( \rewop{r_1}[\eventually s_1] + \rewop{r_2}[\eventually s_2] \right)$
% & $8.664$
% & $(8.4656 \pm 0.9068)\times10^{2}$
& $0.3321 \pm 0.0411$
& $5.8710 \pm 0.0007$
& $\infty$ & -- \\

\midrule

\multirow{2}{*}{Traffic Merge}
& $\llangle p_1:p_2 \rrangle_{\min=?}
\left( \rewop{t_1}[\eventually s_4] + \rewop{t_2}[\eventually s_4] \right)$
% & $4.943$
% & $(5.9532 \pm 1.8531)\times10^{3}$
& $0.6712 \pm 0.0000$
& $24.3784 \pm 0.0003$
& $8.000$ & $0.0000 \pm 0.0000$ \\

& $\llangle p_1 \rrangle \probop{\max=?}
\left[ \eventually^{\le 5} s_4 \right]$
% & $0.541$
% & $(1.7836 \pm 0.7970)\times10^{3}$
& $0.2171 \pm 0.0001$
& $3.1098 \pm 0.0005$
& $1.000$ & $0.0000 \pm 0.0000$ \\

\bottomrule

\end{tabular}%
}
\end{table}

\paragraph{Licenses for external tools.}
\label{app:licenses}
Our implementation builds on the PRISM-games model checker, which is distributed under the GNU General Public License (GPL), version 2. We adhere to its licensing terms. Additional dependencies and their licenses are documented at~\citet{PRISMweb}.

\section{Limitations}
\label{app:limitations}
% \subsection{Assumptions}
We first summarise the assumptions that underpin our analyses, all of which are common in the PAC learning and robust stochastic games literature.

\paragraph{(A1) Centralised exploration.}
We assume that the players perform \emph{centralised} play during exploration. This assumption is standard in PAC analyses of multi-agent systems, including turn-based stochastic games~\citep{brafman2002RMAX,strehl2009RL} and Markov games~\citep{littman1994markov,hu2003nash}, where a centralised learner simplifies the exploration problem. However, it does not capture decentralised or competitive settings where agents act independently or have limited information sharing. 
We remark that the equilibrium-transfer results in \Cref{sec:theory1-estim-error} would remain applicable if comparable confidence sets could be obtained under decentralised learning, but the joint-coverage and termination analysis in \Cref{sec:exploration,sec:main-PAC-CSG} would require a different exploration mechanism.
Developing such a mechanism for decentralised settings remains an important open direction.

\paragraph{(A2) Reachability.}
Our guarantees rely on a \emph{reachability} condition (\Cref{ass:reachability}) ensuring that every relevant (non-target) $(s,a)$-pair is reachable under some profile with positive probability at least $\reach{p}$. This excludes degenerate cases where reachability probabilities vanish and is analogous to standard communicating or reachability assumptions in PAC-MDP analyses~\citep{jaksch2010near,azar2017minimax}, adapted to our robust setting. Without such a condition, PAC termination under online trajectories cannot be guaranteed in general, as illustrated by standard lower-bound constructions (e.g., chain-game instances). Relaxing it typically requires additional capabilities such as access to a generative model or structured exploration mechanisms~\citep{kearns2002nearE3}, which allow coverage to be enforced explicitly.

For infinite-horizon reachability reward (stochastic shortest path) objectives, we assume almost-sure target reachability (\Cref{ass:nz-almost-sure-reachability}), ensuring $\tau_T < \infty$ almost surely and that value functions are well-defined and finite~\citep{bertsekas1991analysis,kwiatkowska2021automatic}. We additionally impose a uniform lower bound $p_T > 0$ on the probability of reaching $S_T$ within $|S|$ steps under a proper policy, which yields a geometric tail bound (\Cref{prop:inductive-tauT}) and induces a finite effective horizon (\Cref{lem:effective-horizon}). This assumption is consistent with standard reachability and properness conditions in SSP and PAC analyses~\citep{jaksch2010near,azar2017minimax,kearns2002nearE3,brafman2002RMAX}.

In practice, we precompute the set of reachable state--action pairs via standard reachability analysis over the known support graph when solving the game, following established model checking techniques~\citep{kwiatkowska2021automatic}. This allows us to restrict attention to $(s,a)$-pairs that are reachable under some profile, eliminating vacuous cases where exploration or reaching is impossible. Both probability lower bounds, $\reach{p}$ and $p_T$, can be computed (or conservatively approximated) by solving associated reachability RMDPs over the known support graph.

\paragraph{(A3) Graph preservation.}
We assume the standard \emph{graph preservation} condition, whereby all transition kernels in the uncertainty set share the same support. This assumption is common in robust MDP and RCSG formulations~\citep{iyengar2005robust, nilim2005robust,HP26,meggendorfer2025solving, suilen2024robust}, and is necessary for tractable robust dynamic programming. However, it excludes settings with support mismatch (e.g., rare or unseen transitions), which may arise in practice when data is sparse.

\paragraph{(A4) Bounded rewards.}
We assume that the magnitude of rewards is bounded by a known constant $\Rmax$, which is standard in PAC analyses of reinforcement learning~\citep{brafman2002RMAX,strehl2009RL}. This assumption ensures well-defined concentration bounds and finite value estimates. In practice, this can typically be satisfied via normalisation.

\subsection{Intrinsic limitations}
\label{app:intrinsic-limitations}
Beyond these assumptions, several intrinsic limitations arise from the problem setting.

\paragraph{Computational complexity.}
Our approach depends on a \emph{black-box RCSG solver}, instantiated in practice via PRISM-games~\citep{HP26}. The complexity of solving general-sum CSGs is well known to be high: for nonzero-sum reachability objectives, computing subgame-perfect Nash equilibria is PSPACE-complete~\citep{brihaye2020complexity}, and practical solvers such as~\citet{HP26} can require exponential time in the model size. These complexity results reflect fundamental barriers in equilibrium computation rather than artefacts of our learning framework. Nevertheless, implementation-level optimisations (e.g., parallelised sampling) and increased computational resources can partially mitigate these costs in practice.

\paragraph{Sample complexity dependence.}
The sample complexity scales polynomially with problem parameters, with a leading $H^4$ dependence arising from compounding transition uncertainty over trajectories. It may be possible to improve this dependence to $H^3$ using empirical-Bernstein or stagewise concentration techniques~\citep{farhat2026sample}, but such approaches typically require stronger assumptions, such as finer control over variance across timesteps (e.g., martingale-style concentration over trajectories or stage-dependent confidence bounds that exploit variance information at each timestep), and lead to a more technically involved analysis.

\paragraph{Non-existence detection.}
While our framework supports both equilibrium computation and non-existence detection, the non-existence certificate is \emph{sound but not complete}: when no exact NE exists, the algorithm may still return an $\varepsilon$-NE rather than certifying non-existence. This limitation is inherent to the problem: while $\varepsilon$-NE exist for any $\varepsilon > 0$, this implies that approximate equilibria persist arbitrarily close to the boundary of equilibrium existence, making it fundamentally difficult to distinguish exact equilibrium existence from arbitrarily small Nash margins~\citep{bouyer2014mixed}.

\paragraph{Other extensions.}
Natural extensions include handling more than two players and non-stationary environments. Both introduce additional challenges: the former increases equilibrium complexity and joint-action scaling, while the latter requires adapting to evolving transition kernels and equilibrium structure over time.

% \paragraph{Limitations and future work.}
% The centralised learner assumption is a natural starting point but limits applicability; extending to decentralised exploration where each player acts independently is an important open direction. 
% The $H^4$ dependence in the sample complexity could potentially be improved to $H^3$ via empirical-Bernstein or stagewise concentration techniques~\citep{farhat2026sample} at the cost of a more involved analysis; we leave this as future work.
% The graph reachability assumption (\Cref{ass:reachability}) is necessary for PAC termination under online trajectories, as shown by a chain-game lower bound; relaxing it would require a generative model or structured exploration protocol. 
% In the infinite-horizon setting, almost-sure stopping (\Cref{ass:nz-almost-sure-reachability}) is necessary for convergence of value iteration and is standard for reachability reward (SSP) objectives \citep{bertsekas1991analysis}.
% Finally, natural extensions include $n>2$ players, as well as partially observable and non-stationary settings.

% \section{Comparison to existing work}
% \label{app:cmp-to-existing-work}
% \input{appendix/cmp-to-existing-work}

\section{Broader Impacts}
\label{app:broader-impacts}
This work contributes to the theoretical foundations of learning in multi-agent stochastic systems under uncertainty. Potential positive impacts include improving the robustness and reliability of autonomous systems, such as multi-robot coordination, distributed control, and decision-making in uncertain environments. 

At the same time, advances in multi-agent learning and equilibrium computation may have unintended applications in competitive or adversarial settings, such as automated negotiation, financial markets, or strategic resource allocation, where robustness could be exploited to gain advantage. While our work is primarily theoretical and does not target specific deployments, we encourage careful consideration of such downstream uses.

%%%%%%%%%%%%%%%%%%%%%%%%%%%%%%%%%%%%%%%%%%%%%%%%%%%%%%%%%%%%

\newpage
\section*{NeurIPS Paper Checklist}

\begin{enumerate}

\item {\bf Claims}
    \item[] Question: Do the main claims made in the abstract and introduction accurately reflect the paper's contributions and scope?
    \item[] Answer: \answerYes{} % Replace by \answerYes{}, \answerNo{}, or \answerNA{}.
    \item[] Justification: 
    The abstract and introduction clearly state the main contributions: a PAC-learning framework for general-sum CSGs with robust equilibrium computation and sound equilibrium non-existence certification, along with theoretical sample complexity guarantees and empirical validation on benchmarks.

\item {\bf Limitations}
    \item[] Question: Does the paper discuss the limitations of the work performed by the authors?
    \item[] Answer: \answerYes{} % Replace by \answerYes{}, \answerNo{}, or \answerNA{}.
    \item[] Justification: 
    The paper provides a consolidated discussion of limitations in \Cref{app:limitations}, including our assumptions made, complexity and scalability, and the incompleteness of non-existence detection. These are explicitly discussed along with their implications and connections to prior work.

\item {\bf Theory assumptions and proofs}
    \item[] Question: For each theoretical result, does the paper provide the full set of assumptions and a complete (and correct) proof?
    \item[] Answer: \answerYes{} % Replace by \answerYes{}, \answerNo{}, or \answerNA{}.
    \item[] Justification: 
    All theoretical results are stated under explicit assumptions and supported by complete proofs in the appendix. For results presented in the main text, we include proof sketches there, with full formal statements, detailed proofs, and cross-references provided in Appendix~\ref{app:prelims}--\ref{app:infh}.

    \item {\bf Experimental result reproducibility}
    \item[] Question: Does the paper fully disclose all the information needed to reproduce the main experimental results of the paper to the extent that it affects the main claims and/or conclusions of the paper (regardless of whether the code and data are provided or not)?
    \item[] Answer: \answerYes{} % Replace by \answerYes{}, \answerNo{}, or \answerNA{}.
    \item[] Justification: 
    The paper specifies the experimental setup (benchmarks, objectives, parameters, randomness sources) in \Cref{sec:experimental-setup} and provides access to implementation, models, and scripts at \githubRepo. Additional details and full results are provided in \Cref{app:additional-experimental-results}, enabling reproducibility.

\item {\bf Open access to data and code}
    \item[] Question: Does the paper provide open access to the data and code, with sufficient instructions to faithfully reproduce the main experimental results, as described in supplemental material?
    \item[] Answer: \answerYes{} % Replace by \answerYes{}, \answerNo{}, or \answerNA{}.
    \item[] Justification: 
    All models, property specifications, and scripts are provided in a public repository: \githubRepo. The implementation builds on the publicly available PRISM-games framework.

\item {\bf Experimental setting/details}
    \item[] Question: Does the paper specify all the training and test details (e.g., data splits, hyperparameters, how they were chosen, type of optimizer) necessary to understand the results?
    \item[] Answer: \answerYes{} % Replace by \answerYes{}, \answerNo{}, or \answerNA{}.
    \item[] Justification: 
    The paper specifies all relevant experimental parameters (e.g., $\varepsilon$, $\delta$, $\Rmax$), benchmark structure, and evaluation protocol, with additional implementation details in \Cref{app:additional-experimental-results}.

\item {\bf Experiment statistical significance}
    \item[] Question: Does the paper report error bars suitably and correctly defined or other appropriate information about the statistical significance of the experiments?
    \item[] Answer: \answerYes{} % Replace by \answerYes{}, \answerNo{}, or \answerNA{}.
    \item[] Justification: 
    Where appropriate, we report mean results together with standard deviations over 10 runs with different random seeds, shown as error bars in \Cref{fig:scaling-plots} and as mean $\pm$ standard deviation values in \Cref{tab:correctness,tab:explorer-cmp,tab:full-results}.

\item {\bf Experiments compute resources}
    \item[] Question: For each experiment, does the paper provide sufficient information on the computer resources (type of compute workers, memory, time of execution) needed to reproduce the experiments?
    \item[] Answer: \answerYes{} % Replace by \answerYes{}, \answerNo{}, or \answerNA{}.
    \item[] Justification: 
    The paper specifies the compute resources in \Cref{sec:experimental-setup}: experiments ran on an Apple M4 Pro with 24~GB RAM, with peak memory usage below 0.5~GB. No specialised hardware was required, and the largest experimental configuration required approximately 5~hours per seed.
    
\item {\bf Code of ethics}
    \item[] Question: Does the research conducted in the paper conform, in every respect, with the NeurIPS Code of Ethics \url{https://neurips.cc/public/EthicsGuidelines}?
    \item[] Answer: \answerYes{} % Replace by \answerYes{}, \answerNo{}, or \answerNA{}.
    \item[] Justification: 
    The research is theoretical and algorithmic, does not involve human subjects or sensitive data, and complies with NeurIPS ethics guidelines.

\item {\bf Broader impacts}
    \item[] Question: Does the paper discuss both potential positive societal impacts and negative societal impacts of the work performed?
    \item[] Answer: \answerYes{} % Replace by \answerYes{}, \answerNo{}, or \answerNA{}.
    \item[] Justification: 
    The paper includes a discussion of broader impacts in \Cref{app:broader-impacts}, outlining both potential positive applications (e.g., robust multi-agent decision-making systems) and possible risks (e.g., strategic or adversarial use in competitive environments).
    
\item {\bf Safeguards}
    \item[] Question: Does the paper describe safeguards that have been put in place for responsible release of data or models that have a high risk for misuse (e.g., pre-trained language models, image generators, or scraped datasets)?
    \item[] Answer: \answerNA{} % Replace by \answerYes{}, \answerNo{}, or \answerNA{}.
    \item[] Justification: 
    The paper does not release high-risk models or datasets; it proposes theoretical algorithms and evaluation on synthetic benchmarks.

\item {\bf Licenses for existing assets}
    \item[] Question: Are the creators or original owners of assets (e.g., code, data, models), used in the paper, properly credited and are the license and terms of use explicitly mentioned and properly respected?
    \item[] Answer: \answerYes{} % Replace by \answerYes{}, \answerNo{}, or \answerNA{}.
    \item[] Justification: 
    The paper's implementation uses the PRISM-games model checker, which is publicly available and open-source. Its licensing terms (GNU GPL v2) are acknowledged and respected, with additional details provided in \Cref{app:licenses}.

\item {\bf New assets}
    \item[] Question: Are new assets introduced in the paper well documented and is the documentation provided alongside the assets?
    \item[] Answer: \answerYes{} % Replace by \answerYes{}, \answerNo{}, or \answerNA{}.
    \item[] Justification: 
    We release an implementation of our PAC-CSG algorithm together with a benchmark suite of concurrent stochastic game models and property specifications at \githubRepo. The repository includes documentation (README) describing the models, experimental setup, parameter choices, and instructions for reproducing all results.

\item {\bf Crowdsourcing and research with human subjects}
    \item[] Question: For crowdsourcing experiments and research with human subjects, does the paper include the full text of instructions given to participants and screenshots, if applicable, as well as details about compensation (if any)? 
    \item[] Answer: \answerNA{} % Replace by \answerYes{}, \answerNo{}, or \answerNA{}.
    \item[] Justification: 
    No human subjects or crowdsourcing involved.

\item {\bf Institutional review board (IRB) approvals or equivalent for research with human subjects}
    \item[] Question: Does the paper describe potential risks incurred by study participants, whether such risks were disclosed to the subjects, and whether Institutional Review Board (IRB) approvals (or an equivalent approval/review based on the requirements of your country or institution) were obtained?
    \item[] Answer: \answerNA{} % Replace by \answerYes{}, \answerNo{}, or \answerNA{}.
    \item[] Justification: 
    No human subjects research conducted.

\item {\bf Declaration of LLM usage}
    \item[] Question: Does the paper describe the usage of LLMs if it is an important, original, or non-standard component of the core methods in this research? Note that if the LLM is used only for writing, editing, or formatting purposes and does \emph{not} impact the core methodology, scientific rigor, or originality of the research, declaration is not required.
    %this research? 
    \item[] Answer: \answerNA{} % Replace by \answerYes{}, \answerNo{}, or \answerNA{}.
    \item[] Justification: 
    LLMs are not part of the methodology or experiments.

\end{enumerate}

\end{document}